\documentclass{article} 
\usepackage{iclr2022_conference,times}

\usepackage[T1]{fontenc}
\usepackage{personal}
\usepackage{natbib}
\usepackage{wrapfig}
\usepackage{tikz}
\usepackage{array}
\usetikzlibrary{automata,positioning}
\setcitestyle{authoryear,open={(},close={)}}

\usepackage{hyperref}
\usepackage{url}

\newcommand{\MDP}{\mathcal{M}}
\newcommand{\statespace}{\mathcal{S}}
\newcommand{\actionspace}{\mathcal{A}}
\newcommand{\transition}{\mathbb{P}}
\newcommand{\reward}{r}
\newcommand{\horizon}{H}

\newcommand{\target}{\mathrm{Tg}}
\newcommand{\meta}{\mathrm{hl}}
\newcommand{\history}{\mathcal{H}}
\newcommand{\cluster}{Z}
\newcommand{\hier}{\mathrm{hier}}
\newcommand{\entry}[1]{\mathrm{Ent}(#1)}
\newcommand{\exit}[1]{\mathrm{Ext}(#1)}

\newcommand{\avexit}[1]{\mathrm{AvExt}(#1)}
\newcommand{\interior}[1]{#1^\circ}
\newcommand{\dataset}{\mathcal{D}}

\newcommand{\thresh}{\mathrm{thresh}}
\newcommand{\rewardfree}{\mathrm{RF}}
\newcommand{\taskspecific}{\mathrm{TS}}
\newcommand{\exitdetection}{\mathrm{ED}}
\newcommand{\exitlearning}{\mathrm{EL}}
\newcommand{\UCBVI}{\textsc{UCBVI}}
\newcommand{\Euler}{\textsc{Euler}}
\newcommand{\imagined}{\mathcal{I}}
\newcommand{\imaginedMDPs}{\mathbb{M}}

\newcommand{\effective}{\mathrm{eff}}
\newcommand{\terminate}{\ominus}

\newcommand{\troot}{\mathrm{root}}
\newcommand{\gate}{\mathrm{gate}}
\newcommand{\trap}{\mathrm{trap}}

\title{Provable Hierarchy-Based \newline Meta-Reinforcement Learning}

\author{Kurtland Chua, Qi Lei \& Jason D. Lee \\
	Princeton University\\
	Princeton, NJ 08544, USA \\
	\texttt{\{kchua,qilei,jasonlee\}@princeton.edu} \\
}

\ifdefined\usebigfont

\usepackage{times}
\usepackage[fontsize=13pt]{scrextend}
\usepackage[left=1.56in,right=1.56in,top=1.71in,bottom=1.77in]{geometry}
\else
\fi

\iclrfinalcopy 
\begin{document}

	\maketitle

	\begin{abstract}
		Hierarchical reinforcement learning (HRL) has seen widespread interest as an approach to tractable learning of complex modular behaviors. However, existing work either assume access to expert-constructed hierarchies, or use hierarchy-learning heuristics with no provable guarantees. To address this gap, we analyze HRL in the meta-RL setting, where a learner learns latent hierarchical structure during meta-training for use in a downstream task. We consider a tabular setting where natural hierarchical structure is embedded in the transition dynamics. Analogous to supervised meta-learning theory, we provide “diversity conditions” which, together with a tractable optimism-based algorithm, guarantee sample-efficient recovery of this natural hierarchy. Furthermore, we provide regret bounds on a learner using the recovered hierarchy to solve a meta-test task. Our bounds incorporate common notions in HRL literature such as temporal and state/action abstractions, suggesting that our setting and analysis capture important features of HRL in practice.
	\end{abstract}

	\section{Introduction}
		Reinforcement learning (RL) has demonstrated tremendous successes in many domains \citep{schulman2015trust,vinyals2019grandmaster,schrittwieser2020mastering}, learning near-optimal policies despite limited supervision.
Nevertheless, RL remains difficult to apply to problems requiring temporally extended planning and/or exploration \citep{ecoffet2021first}.
A promising approach to this problem is hierarchical reinforcement learning (HRL), which has seen continued interest due to its appealing biological basis.
In its most basic form, HRL seeks to solve tasks using a collection of primitive skills, each of which is easier to learn individually than the full task.
By restricting the agent to using learned skills, the search space over policies can be greatly reduced.
Furthermore, learned skills can induce simpler state and/or action spaces, simplifying the learning problem.
Finally, learned skills with useful semantic behavior can be reused across tasks, enabling transfer learning.

Naturally, a hierarchy-based learner is limited by the quality of skills that are made available and/or learned.
Accordingly, many empirical works have proposed algorithms for online skill learning in the context of a single RL task \citep{nachum2019data,nachum2018near}.
These approaches have been experimentally demonstrated to be effective in finding useful and interpretable skills.
Other approaches consider the skill learning problem in the context of meta-RL \citep{frans2018meta}, or in the reward-free setting \citep{eysenbach2018diversity}.
Nevertheless, the heuristics and algorithms proposed in these empirical works do not provide any provable guarantees on the quality of learned skills.

On the other hand, theoretical analyses have mostly focused on how learners benefit from having access to skills.
For example, \citet{fruit2017exploration} provide a regret bound on learning with skills in the infinite-horizon average reward case.
Meanwhile, in the meta-RL setting, \citet{brunskill2014pac} consider the problem of finding and using skills in a continual learning setting and provides a sample complexity analysis.
However, these analyses either sidestep the question of how the skills are obtained, or do not address the problem in a computationally tractable manner.

In this work, we provide settings where there exists provable guarantees for hierarchy learning through tractable algorithms.
We focus on the meta-RL setting, in which a learner extracts skills from a set of provided tasks which are then used in a downstream task.
We work in the tabular case, assuming the transition dynamics of the given tasks share latent hierarchical structure induced by predetermined clustering and bottlenecks.

Our contributions are as follows:

\begin{enumerate}
	\item \textbf{``Diversity conditions'' ensuring hierarchy recovery.}
	We develop natural optimism-based coverage conditions which ensure that bottlenecks embedded in the transition dynamics are detectable by solving provided meta-training tasks.

	\item \textbf{A tractable hierarchy-learning algorithm.}
	We provide an algorithm that provably learns the latent hierarchy from interactions, assuming the coverage conditions above.
	Our method has sample complexity scaling as $O(TKS)$ in the leading term compared to $O(TS^2A)$ for a brute-force method, where $T$ are the number of tasks, $S$ is the number of states, $A$ is the number of actions, and $K \ll SA$ is the number of skills to learn.

	\item \textbf{Regret bounds on downstream tasks.}
	We provide regret bounds for learners that apply the extracted hierarchy from meta-training on downstream tasks.
	Furthermore, we show an exponential regret separation between hierarchy-based and hierarchy-oblivious learners for a family of task distributions, corroborating prevailing intuitions regarding when/why HRL helps.
	In particular, hierarchy-based learners incur regret bounded by $O(\sqrt{H^2N})$ while hierarchy-oblivious learners incur worst-case regret of at least $O(2^{H/2}\sqrt{H^2N})$.
\end{enumerate}

	\section{Related Work}
		Hierarchical reinforcement learning has been studied extensively \citep{sutton1999between, parr1998reinforcement, dietterich1998maxq, vezhnevets2017feudal}.
An early approach to formalizing the use of hierarchies in RL is the options framework \citep{sutton1999between}, which fixes a finite set of skills that are available to the learner.
Since then, a large body of work has focused on designing methods for learning and adapting these options throughout the learning process \citep{mcgovern2001automatic,menache2002q,csimcsek2004using,mann2014time}.
Of particular note is the work of \citet{frans2018meta}, which proposes to learn a finite set of neural network sub-skills in the meta-RL setting.
On the other hand, Laplacian-based option discovery as explored in \citet{machado2017laplacian,machado2018eigenoption} studies connections between options and proto-value functions \citep{mahadevan2005proto}, which naturally capture bottlenecks in the state space.
In more theoretical directions, \citet{fruit2017exploration,brunskill2014pac} provide regret and sample complexity bounds, respectively, for learning with options.
Additionally, \citet{mann2014scaling} demonstrate that learning with options can improve the convergence rate of approximate value iteration.

More recent empirical work has considered the problem of hierarchy learning beyond the options framework in a wide variety of settings.
\citet{nachum2019data,levy2018learning} provide algorithms for learning hierarchies based on goal-conditioned policies, reducing the learning problem to choosing subgoals.
\citet{nachum2018near} considers a more general case when learned representations are used to map states to goals.
Other works such as \citet{co2018self,eysenbach2018diversity,sharma2019dynamics} provide intrinsic objectives for learning hierarchies without rewards.

Closely related to our work is that of \citet{wen2020efficiency}, which also decomposes the state space into clusters with exits.
However, they focus on the reduction in regret when the learner already knows the decomposition, as well as cluster equivalences (in terms of dynamics and reward).
In contrast, a major focus of our work is discovering the decomposition itself from interactions.

	\section{Notation}
		We now introduce notation which we will use throughout the paper.
We write $[K] \defas \set{1, \dots, K}$.
Furthermore, we use the standard notations $O, \Theta, \Omega$ to denote orders of growth, and $\tilde{O}, \tilde\Theta, \tilde\Omega$ to indicate suppressed logarithmic factors.
We use $\delta(x)$ to denote the Dirac delta measure on $x$.

We work with finite-horizon Markov decision processes (MDPs), defined as a tuple $\MDP = (\statespace, \actionspace, \transition, r, H)$, where $\statespace$ is the set of states, $\actionspace$ is the set of actions, $\transition: \statespace \times \actionspace \times \statespace \to [0, 1]$ are the transition dynamics, $r: \statespace \times \actionspace \to [0, 1]$ is the reward function, and $H$ is the horizon.
We assume stationary dynamics unless otherwise noted, in which case $\transition^{(h)}$ is the dynamics at time step $h$.
For constants relating to horizons, we will define $[H] \defas \set{0, \dots, H - 1}$.
Given a policy $\pi: [H] \times \statespace \to \actionspace$, we define the value functions
\[
V^{\pi}_h(s) \defas \expt{\sum_{k = h}^{H - 1}r(s_k, a_k) \suchthat s_h = s} \quad \text{and} \quad Q^{\pi}_h(s, a) \defas \expt{\sum_{k = h}^{H - 1}r(s_k, a_k) \suchthat (s_h, a_h) = (s, a)}
\]
where $s_{k + 1} \sim \transition(\cdot \conditionedon s_k, a_k)$ and $a_k = \pi_k(s_k)$.
Furthermore, we write $V^\ast$ and $Q^\ast$ to denote optimal value functions obtained by maximizing over $\pi$ (and are attained by the optimal policy $\pi^\ast$).
When a learner plays $\pi_1, \dots, \pi_N$ in $\MDP$, we define its regret as
\[
\mathrm{Regret}_N(\MDP) \defas \sum_{t = 1}^{N}V^\ast(s_0) - V^{\pi_t}(s_0).
\]

We use $\terminate$ to denote a terminal state.
We let $\tau_\pi$ denote the (random) trajectory generated by $\pi$.
For a state $s$ and length-$H$ trajectory $\tau$, we write $s \in \tau_\pi$ if $s_h = s$ for some $h \in [H]$.
We define $(s, a) \in \tau_\pi$ similarly.
Finally, given an MDP $\MDP$, $\MDP(2H)$ denotes a copy of $\MDP$ with a doubled horizon.

	\section{Setting}
		\begin{wrapfigure}[24]{r}{0.3\linewidth}
	\centering
	\vspace{-2ex}
	\includegraphics[width=\linewidth]{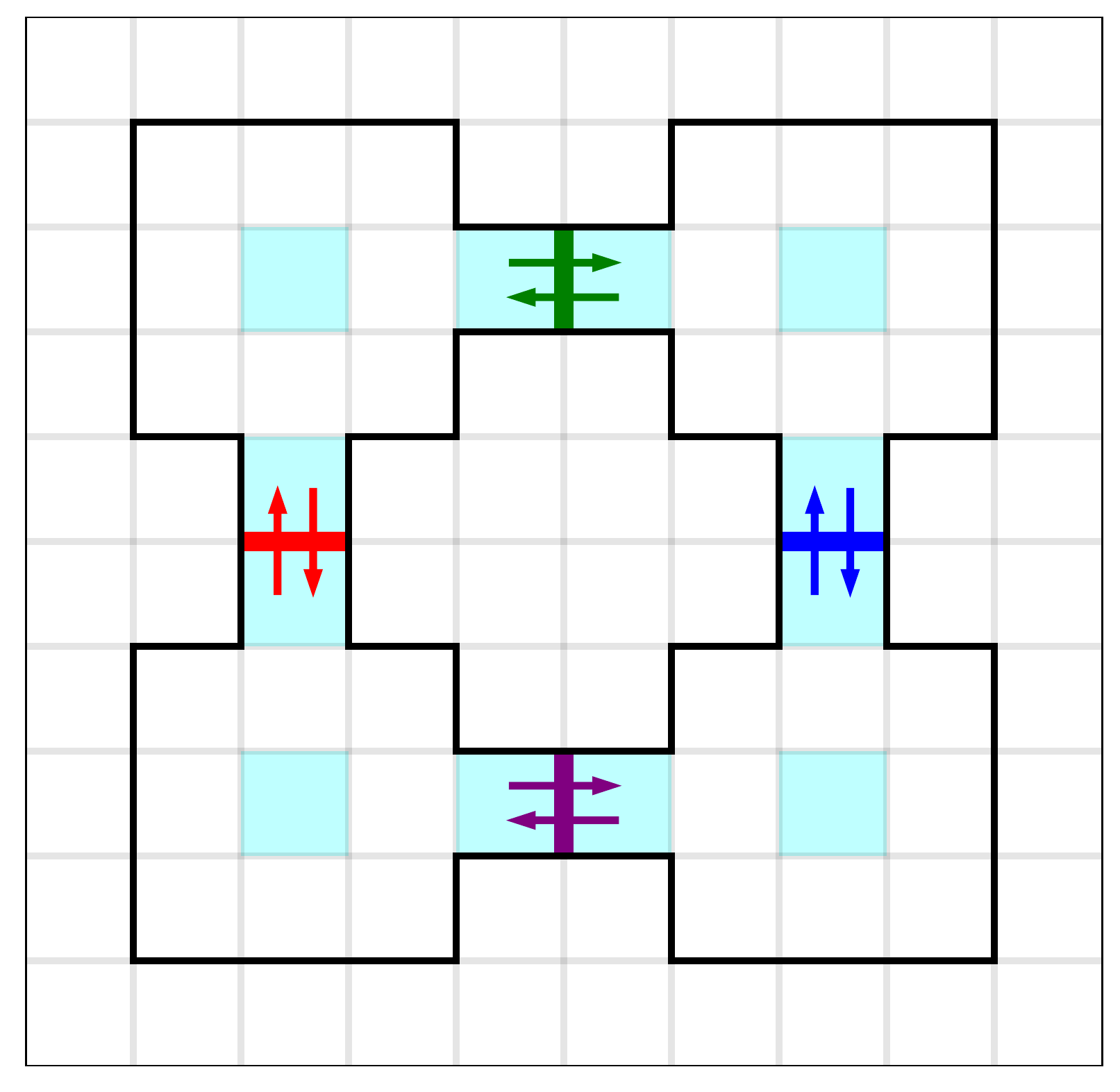}
	\includegraphics[width=\linewidth]{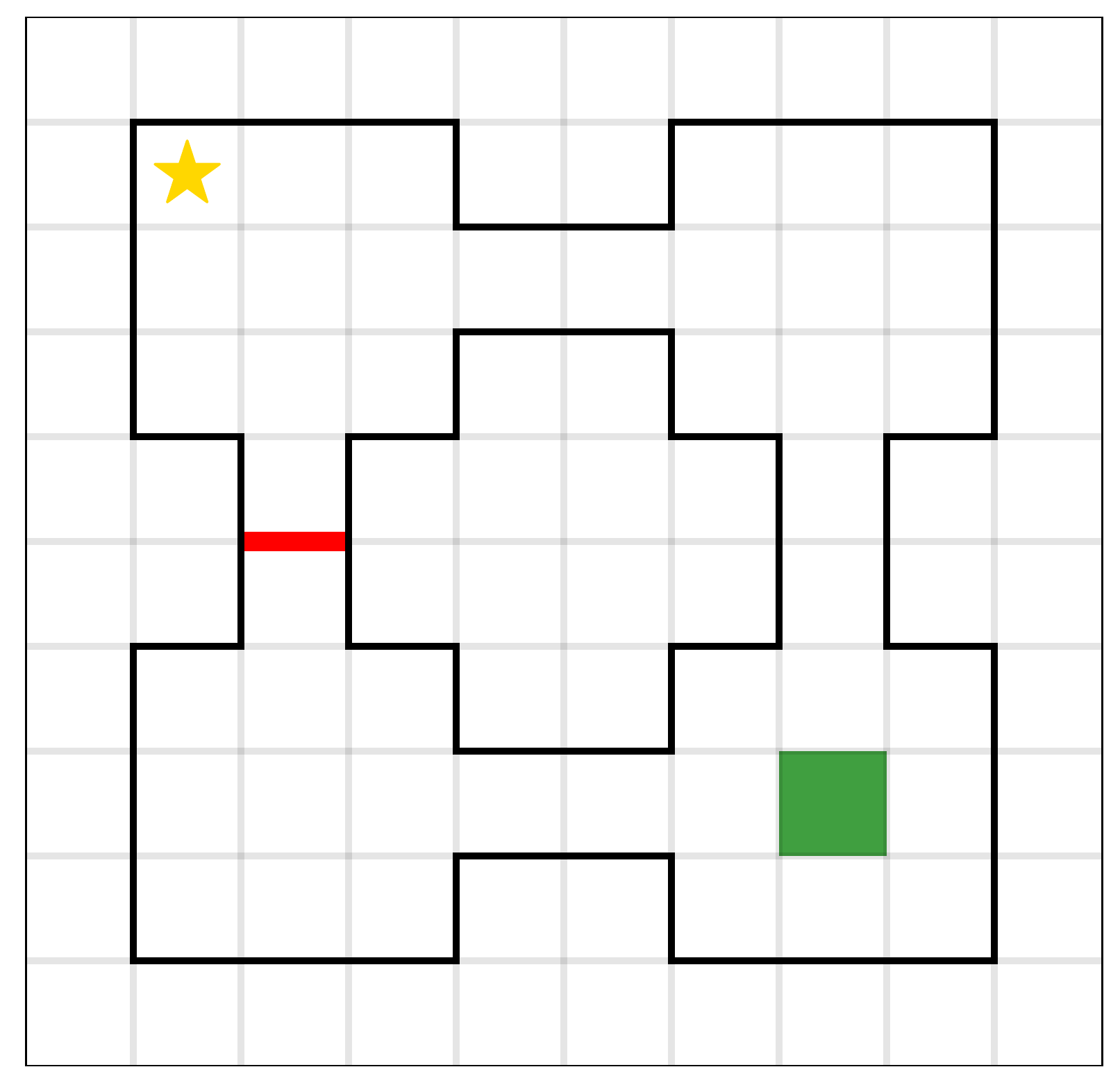}
	\caption{The gated four-room environment, with a sample task. The agent has to navigate from the green square to the star.}
	\label{fig:four-room-env}
\end{wrapfigure}

We work in the tabular meta-RL setting.
The learner has access to $T$ meta-training MDPs $\set{(\statespace, \actionspace, \transition_t, \reward_t, \horizon)}_{t \in [T]}$.
Note that the transition dynamics and reward function both vary across tasks.
We set $S \defas \card{\statespace}$ and $A \defas \card{\actionspace}$.
After interacting with these tasks, the learner is presented with a meta-test MDP $(\statespace, \actionspace, \transition_\target, \reward_\target, \horizon)$, where the learner seeks to minimize its regret.
We assume, without loss of generality, that the MDPs have a shared starting state $s_0$.

For meta-learning to succeed, there needs to be a shared structure among the MDPs above.
We focus on studying shared hierarchical structure, which we now formally define:

\begin{definition}[Latent Hierarchy]
	\label{definition:latent-hierarchy}
	Let $\set{\cluster_c}$ be a partition of the state space $\statespace$ into clusters.
	We associate with each cluster $\cluster$ a set of entrances $\entry{\cluster} \subseteq \cluster$ and a set of exits $\exit{\cluster} \subseteq \cluster \times \actionspace$.
	We say that the tasks have a \term{latent hierarchy} with respect to $(\set{\cluster_c}, \entry{\cdot}, \exit{\cdot})$ if for any $Z_c$:

	\begin{enumerate}[label=(\alph*)]
		\item For any $(s, a) \in (\cluster_c \times \actionspace) \setminus \exit{\cluster_c}$, $\transition_t(\cdot \conditionedon s, a)$ is constant over $t$ and supported on $Z_c$. \qedhere

		\item For any $(s, a) \in \exit{\cluster_c}$, there exists $t, t'$ with $t \neq t'$ such that $\transition_t(\cdot \conditionedon s, a) \neq \transition_{t'}(\cdot \conditionedon s, a)$.
		Furthermore, $\transition_t(\cdot \conditionedon s, a)$ is supported on $\Union_c\entry{\cluster_c}$ for any $t \in [T]$.
	\end{enumerate}
\end{definition}

The latent hierarchy partitions MDPs into clusters such that (1) non-exit $(s, a)$ dynamics do not change between the MDPs and (2) exits are bottlenecks between clusters.
To illustrate \defref{definition:latent-hierarchy}, we begin with a standard example.

\begin{example}[Gated Four-Room]
	\label{example:gated-four-room}
	Consider the gated four-room environment in \figref{fig:four-room-env}, as well as the example task provided.
	The environment has a latent hierarchy with respect to the four rooms outlined by the colored gates (which can be open/closed depending on the task).
	The entrances are colored aqua, while the exits are indicated by arrows.

	Although we assume a single fixed state, we can incorporate task-dependent initial states by appending a dummy state $s_0$.
	A dummy action $a_0$ then takes the agent to the starting state for the task.
	Observe that $(s_0, a_0)$ is an exit, and therefore must transition to an entrance.
\end{example}

To see how \defref{definition:latent-hierarchy} captures intuitive notions of hierarchy in practical settings, we provide an example of a continuous setting roughly fitting into our framework:

\begin{wrapfigure}[12]{r}{0.4\linewidth}
	\centering
  \vspace{-1.5ex}
	\includegraphics[width=\linewidth]{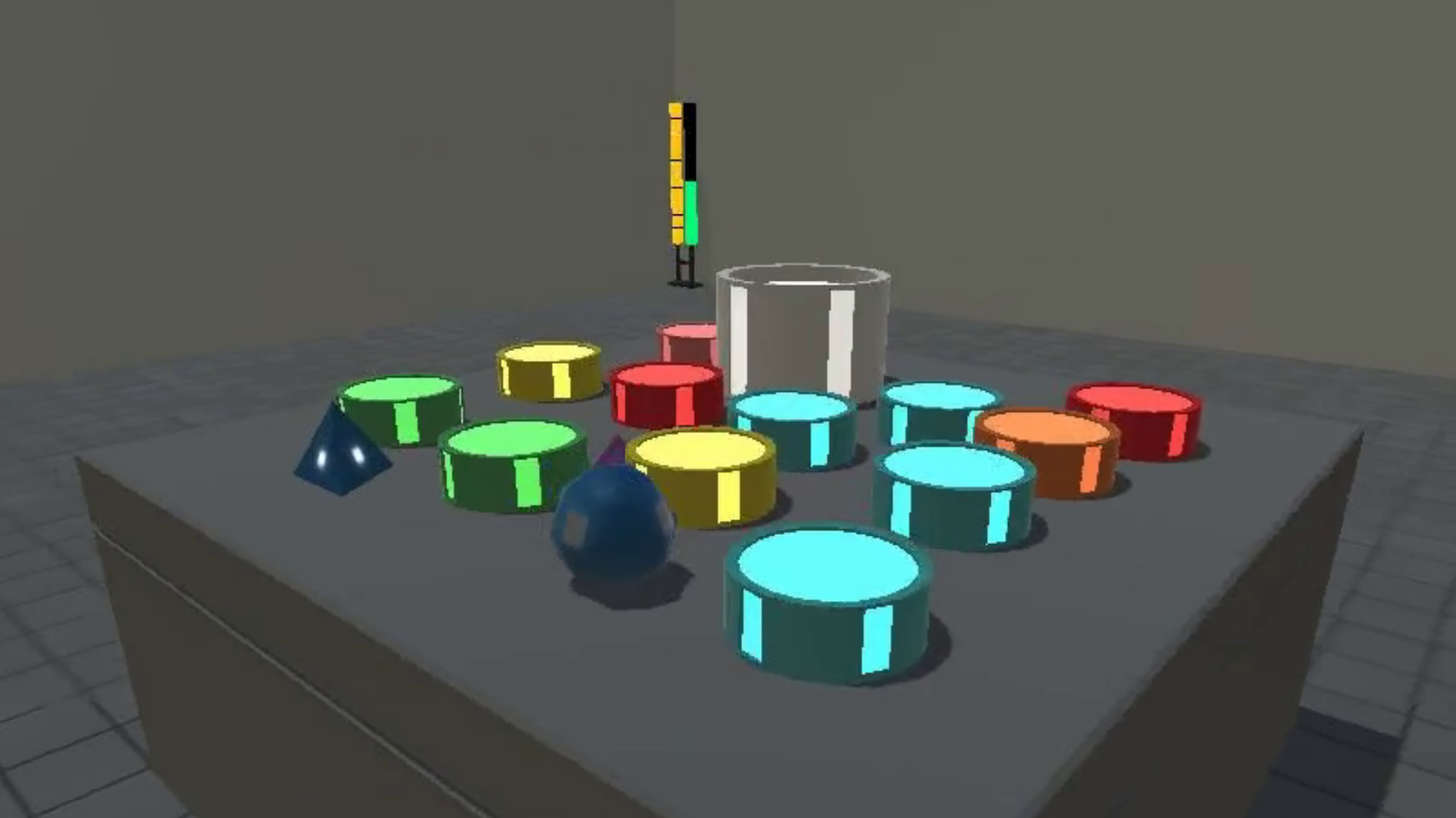}
	\caption{The Alchemy benchmark. Placing stones in potions moves the agent through a latent graph of object properties.\protect\footnotemark}
	\label{fig:alchemy-env}
\end{wrapfigure}

\begin{example}[The Alchemy benchmark]
	Alchemy \citep{wang2021alchemy} is a recently proposed empirical benchmark for meta-RL, where the agent needs to place a stone in a series of potions to obtain some desired appearance, as illustrated in \figref{fig:alchemy-env}.
	Dipping a stone into a potion traverses an edge (determined by the potion) in a graph where nodes are possible stone appearances.
	We focus on task distributions that randomize the edges of this graph (i.e., potion positions and feasible stone appearances are fixed).
	Then, the set of obtainable MDPs has a latent hierarchy where dipping the stone into any of the potions is an exit.
	Indeed, other than dipping the stone into a potion, all other actions (e.g., moving the stone around the room) have the same dynamics in all tasks.
\end{example}

\footnotetext{Image from \citet{wang2021alchemy}, extracted from a larger figure with no other modifications (\href{https://creativecommons.org/licenses/by/4.0/}{License}).}

For convenience, we will define several relevant notions.
First, for any cluster $\cluster$, we define its \emph{interior}, denoted $\interior{\cluster}$, as $\interior{\cluster} \defas (\cluster \times \actionspace) \setminus \exit\cluster$.
Furthermore, we let $\entry\statespace \defas \Union_c\entry{Z_c}$ denote the set of all entrances and $\exit\statespace \defas \Union_c\exit{Z_c}$ the set of all exits.
Finally, we define the quantities
\begin{align*}
  K &\defas \card{\exit\statespace} \\
  L &\defas \card{\entry\statespace} \\
  M &\defas \sup_{c}\card{\exit{Z_c}}
\end{align*}
so that $K$ and $L$ are the total number of exits and entrances, respectively, while $M$ is the maximal number of exits from any cluster.

\paragraph{Connections to the Options Framework.} The options framework \citep{sutton1999between} seeks to formalize hierarchical reasoning in RL.
Central to the framework is the notion of an \emph{option}, a tuple $(\pi, \mathcal{I}, \beta)$ where $\pi$ is a policy, $\mathcal{I} \subseteq \statespace$ is the \emph{initiation set} from where the option can be invoked, and $\beta: \statespace \to [0, 1]$ is the \emph{termination condition} (defined as a termination probability at every state).
An option thus encodes a temporally extended behavior available to an agent.

Observe that the latent hierarchy in \defref{definition:latent-hierarchy} induces a natural set of options.
In particular, for any cluster $\cluster$ and $(g, a) \in \exit{\cluster}$, we can let $\pi$ be the optimal $g$-reaching policy, $\mathcal{I} \defas \entry{\cluster}$, and $\beta(s) \defas \ind{s = g}$.
As we will see, this set of options is useful for tasks requiring navigation through a sequence of exits to reach a goal.

\paragraph{Query Model.} We work in the online setting, where the agent interacts with the tasks by playing policies from the initial state $s_0$.
During meta-training, we allow the agent to interact with the environments using an unbounded number of timesteps for each trajectory before resetting.
We then compute query complexity in terms of the total number of timesteps spent in all tasks in total.

	\section{Meta-Training Analysis}
		In this section, we provide an algorithm for uncovering information about the latent structure that can be used for downstream tasks.
Recall that a defining feature of exits is that they can change dynamics between tasks.
Thus, to ensure exit detection, we make the following assumption:

\begin{assumption}[$\beta$-dynamics separation]
	\label{assumption:dynamics-separation}
	There exists $\beta > 0$ such that for any $t, t' \in [T]$ and $(s, a) \in \exit\statespace$, $\transition_t(\cdot \conditionedon s, a) \neq \transition_{t'}(\cdot \conditionedon s, a) \implies \TV{\transition_t(\cdot \conditionedon s, a) - \transition_{t'}(\cdot \conditionedon s, a)} \geq \beta$.
\end{assumption}

\assumpref{assumption:dynamics-separation} is necessary to ensure that exits can be detected with finite samples.
In particular, estimators for $\transition_t(\cdot \conditionedon s, a)$ and $\transition_{t'}(\cdot \conditionedon s, a)$ with at least $O(S/\beta^2)$ samples will be separated by $\Omega(\beta)$ in total variation distance with high probability.

Note that there is a brute-force approach to learning the underlying structure.
In particular, one can learn $\transition_t(\cdot \conditionedon s, a)$ for all $(s, a)$ and $t \in [T]$, and iterate over $(s, a)$ to check for changing dynamics.
This can be done with query complexity $\tilde{O}(TS^2A/\beta^2)$ time steps.
However, under reasonable ``coverage'' assumptions outlined in the next section, this query cost can be lowered to $\tilde{O}(TKS/\beta^2)$.

\subsection{Defining a Notion of Coverage}

In supervised meta-learning, ``diversity conditions'' ensure that the meta-training tasks reveal the underlying latent structure \citep{tripuraneni2020theory, du2020few}.
We provide analogous conditions ensuring that $\MDP_1, \dots, \MDP_T$ ``cover'' the latent hierarchy.
Since solving $\max_\pi V^\pi(s_0)$ requires fewer samples than learning $\transition$, we expect such conditions to provide sample complexity gains.

\paragraph{Visitation Probabilities and $\alpha$-Importance.} Minimally, exits should be visited by optimal policies of the meta-training tasks for coverage. We thus define the following notion:

\begin{wrapfigure}[17]{r}{0.3\linewidth}
	\centering
	\vspace{-1ex}
	\captionsetup{width=0.95\linewidth}
	\includegraphics[width=\linewidth]{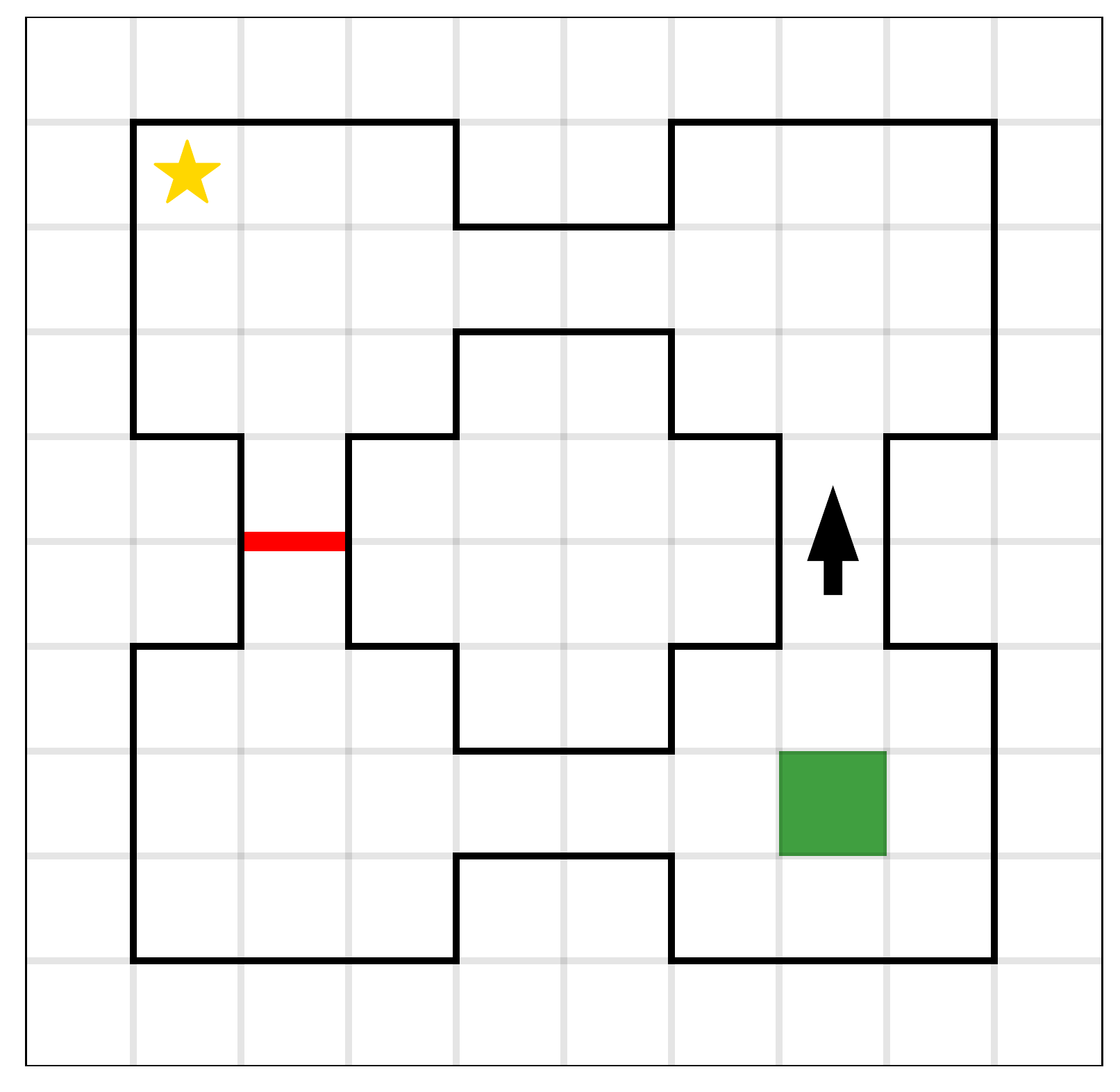}
	\caption{Illustrating $\alpha$-importance. Since the black arrow is the only path to the goal, it is $V_0^\ast(s_0)$-important.}
	\label{fig:alpha-importance-illustration}
\end{wrapfigure}

\begin{definition}
	Fix an MDP $\MDP = (\statespace, \actionspace, \transition, r, H)$, and let $(s, a) \in \statespace \times \actionspace$.
	Construct a modified MDP $\MDP \setminus (s, a)$, where $(s, a)$ brings the agent to a terminal state with no reward.
	Then, we say that $(s, a)$ is \term{$\alpha$-important for $\MDP$} if $V^{\MDP \setminus (s, a), \ast}(s_0) < V^{\MDP, \ast} - \alpha$.
\end{definition}

The $\alpha$-importance condition quantifies the value gap between policies that can use $(s, a)$ and those that cannot.
For example, consider the task in \figref{fig:alpha-importance-illustration}.
Any $\pi$ with $V^\pi(s_0) > 0$ must visit the marked $(s, a)$ pair with some probability.
Therefore, $(s, a)$ is $V^\ast_0(s_0)$-important.
This example suggests that high $\alpha$-importance implies high visitation probability by near-optimal policies.
The following result, proven in \secref{sec:meta-train-proofs}, formalizes this connection:

\begin{lemma}
	\label{lemma:importance-implies-visitation-mt}
	Assume that $(s, a)$ is $\alpha$-important for $\MDP$.
	Then, for any policy $\pi$ with $V^\ast(s_0) - V^\pi(s_0) < \epsilon$ for $\epsilon \in [0, \alpha)$,
	\[
	\prob{(s, a) \in \tau_\pi} > \frac{1}{H}(\alpha - \epsilon).
	\]
\end{lemma}

\paragraph{A Preliminary Coverage Assumption?} \lemmaref{lemma:importance-implies-visitation-mt} suggests a simple coverage condition: for any $(s, a) \in \exit\statespace$, assume that there exists $t, t' \in [T]$ so that $(s, a)$ is $\alpha$-important for $t$ and $t'$, and $\transition_t(\cdot \conditionedon s, a) \neq \transition_{t'}(\cdot \conditionedon s, a)$.
However, as the following example shows, this condition excludes natural settings:

\begin{figure*}[b!]
	\centering
	\begin{minipage}{0.65\linewidth}
		\centering
		\includegraphics[width=0.49\linewidth]{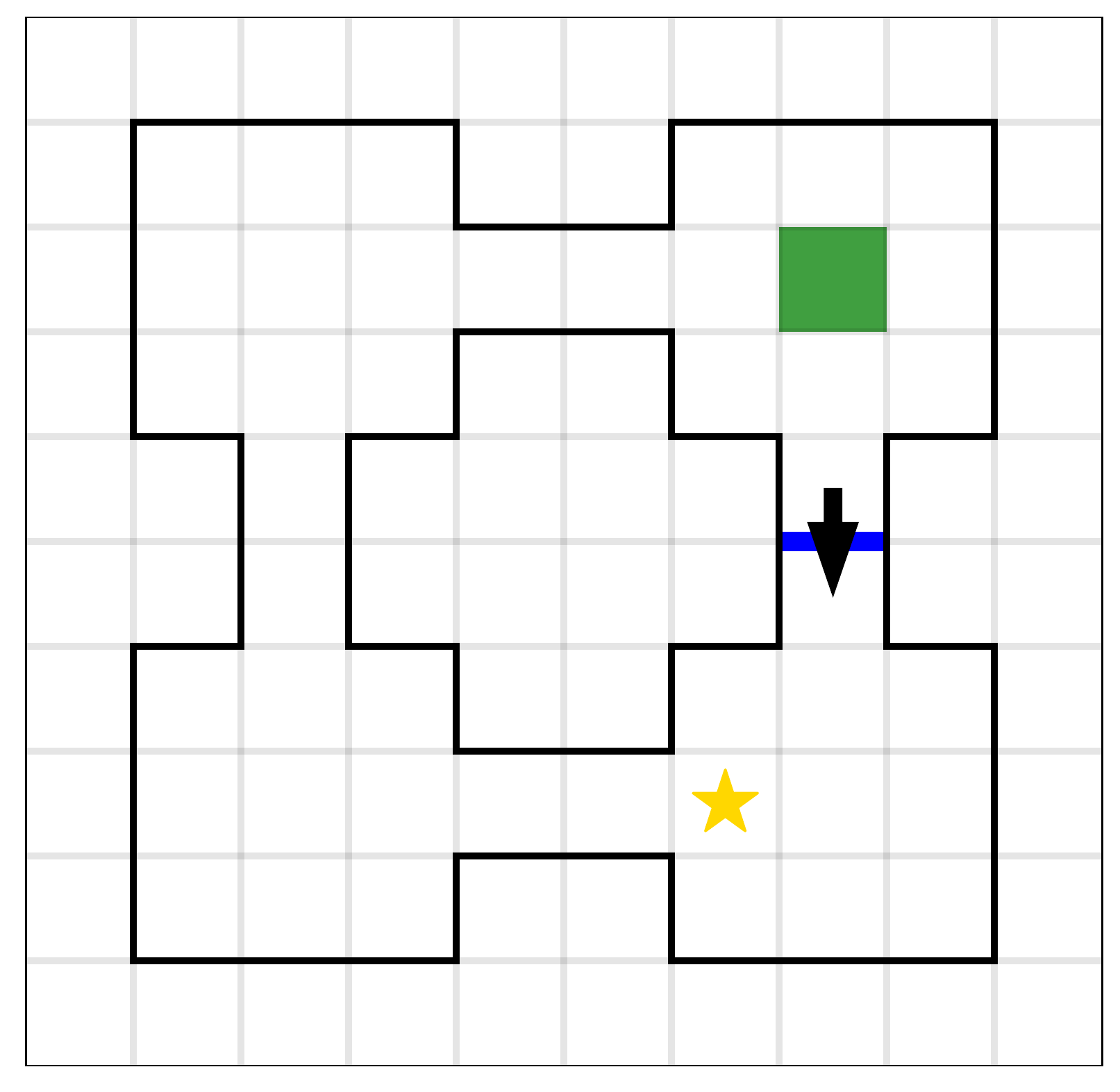}
		\includegraphics[width=0.49\linewidth]{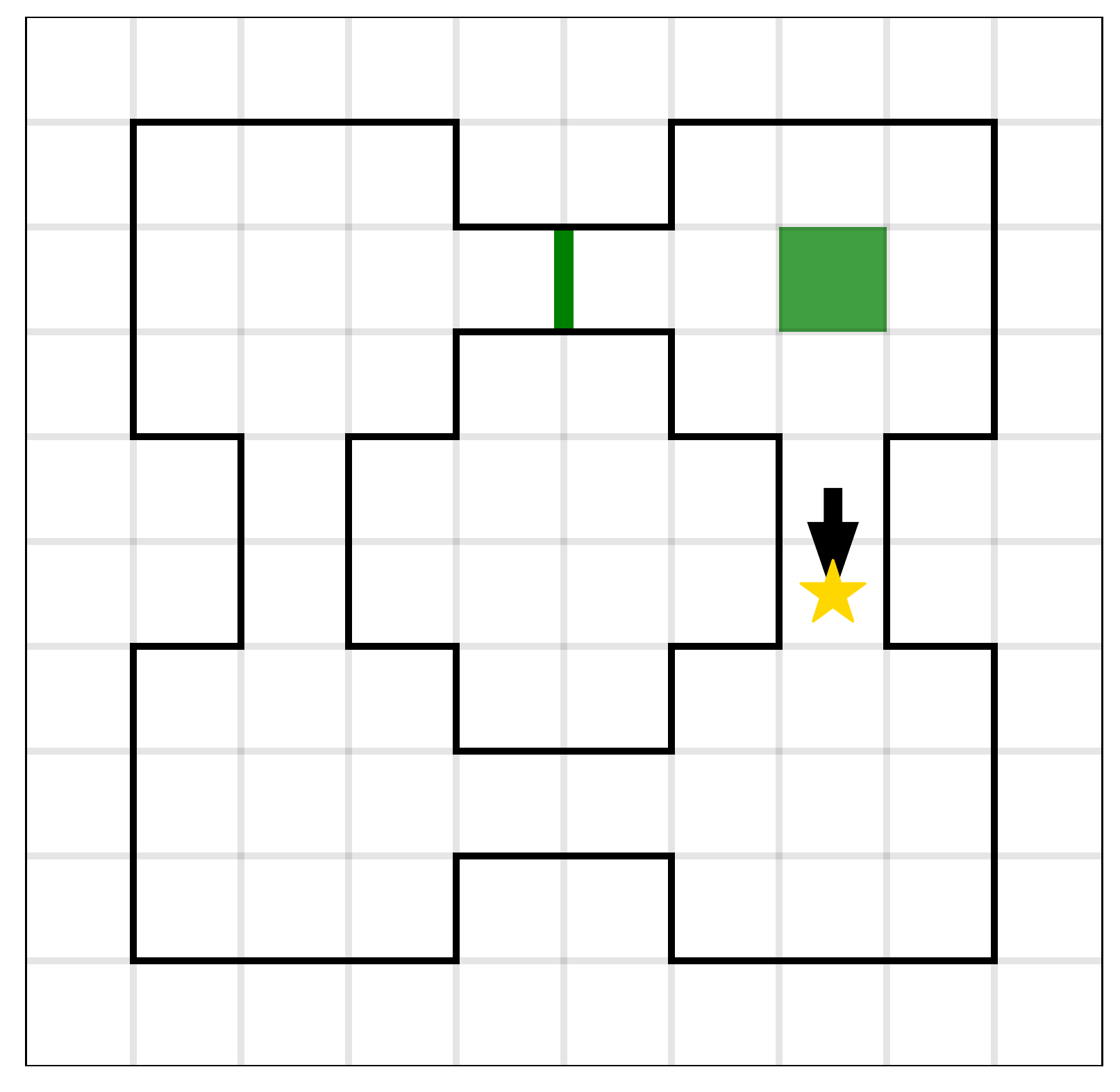}
		\caption{The black arrow is not $\alpha$-significant for one of the tasks, but is nevertheless ``covered'' by optimistic imagination.}
		\label{fig:coverage-base-tasks}
	\end{minipage}
	\hfill
	\begin{minipage}{0.33\linewidth}
		\centering
		\includegraphics[width=0.96\linewidth]{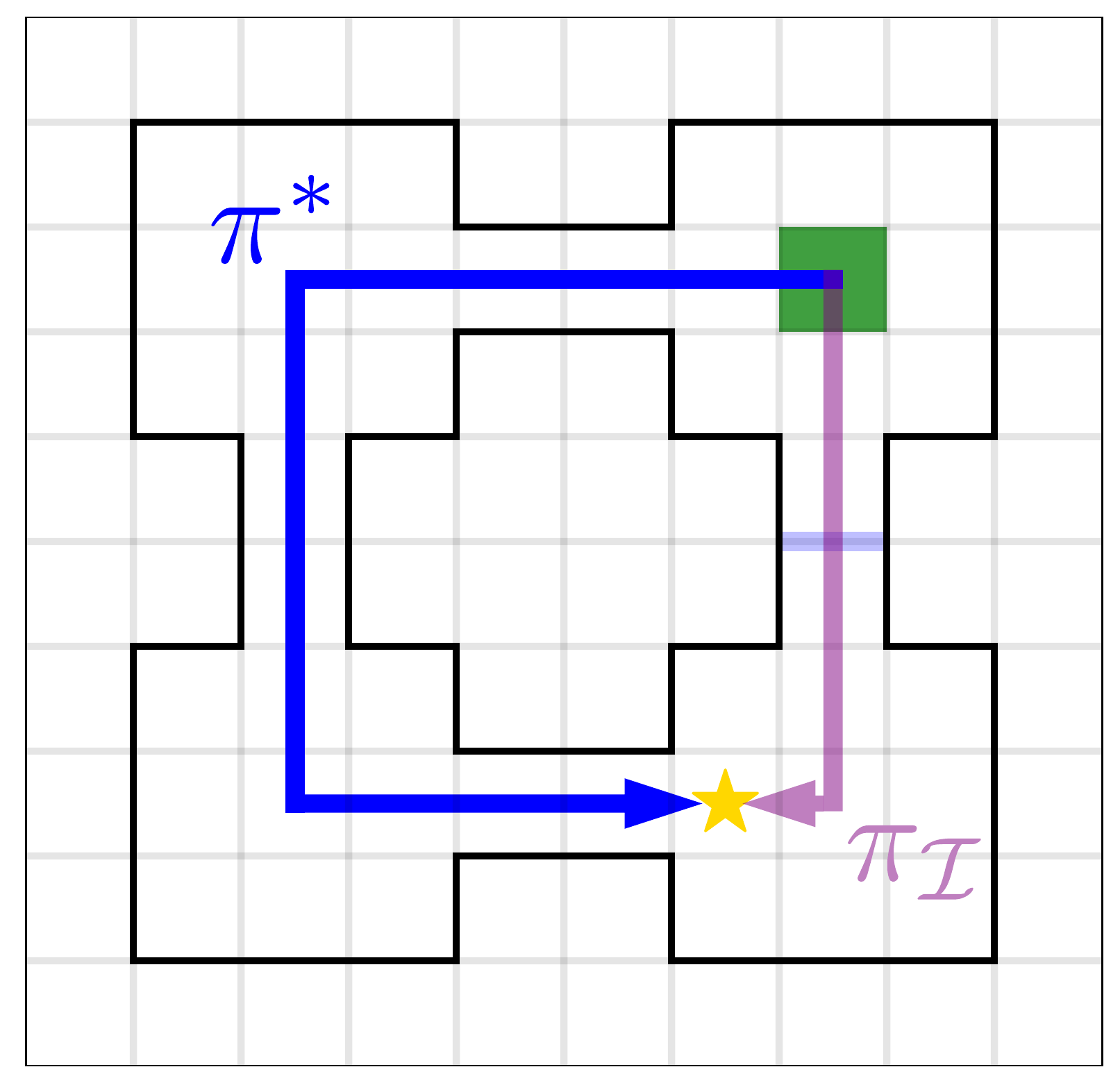}
		\caption{Using optimistic imagination for exit detection.}
		\label{fig:optimistic-imagination-illustration}
	\end{minipage}
\end{figure*}

\begin{example}
	In \exref{example:gated-four-room}, gates are either open or closed.
	For closed gates, the associated $(s, a)$ pairs are unimportant for any goal.
	Thus, if $(s, a)$ is $\alpha$-important for some $t, t' \in [T]$ with $\alpha > 0$, then $\transition_t(\cdot \conditionedon s, a) = \transition_{t'}(\cdot \conditionedon s, a)$, and no change in dynamics can be detected using $\pi_t^\ast$ and $\pi_{t'}^\ast$.

	\figref{fig:coverage-base-tasks} illustrates how the proposed condition fails in covering an exit marked with a black arrow.
	While $(s, a)$ is unused by optimal policies in the first task, it is $V^\ast_0(s_0)$-important for the second.
\end{example}

\paragraph{Optimistic Imagination as a Coverage Mechanism.}
The proposed assumption fails because there are cases where exits are only $\alpha$-important in certain configurations (e.g., only open corridors are $\alpha$-important in \exref{example:gated-four-room}).
In such cases, near-optimal policies for the tasks only ever see one configuration of the dynamics for such exits.

As an alternative, consider the following hypothetical scenario in the context of \figref{fig:coverage-base-tasks}: an agent has solved both tasks, achieving optimal values $V_1^\ast$ and $V_2^\ast$.
Additionally, in the process of learning the second task, the agent has learned $\transition_2(\cdot \conditionedon \Downarrow)$.
If the agent then relearns the first task while setting $\hat\transition_1(\cdot \conditionedon \Downarrow) \gets \transition_2(\cdot \conditionedon \Downarrow)$, it would obtain a new value $\hat{V}_1 \gg V_1^\ast$.
Thus, it can reasonably conclude that the black arrow must have been an exit.
We illustrate this process in \figref{fig:optimistic-imagination-illustration}, where $\pi_{\mathcal{I}}$ is the optimal policy after ``borrowing dynamics.''
The learner could then run $\pi_{\mathcal{I}}$ for exit detection.

We refer to the counterfactual reasoning about the dynamics used above as \emph{optimistic imagination}.
Note that optimistic imagination replaces exit importance in a second task in the preliminary condition with a borrowing-induced value gap condition, allowing for coverage when the preliminary condition fails.
With the above intuition in mind, we now present the main coverage assumption.\footnote{\assumpref{assumption:exit-coverage} does not incorporate the preliminary condition. However, our algorithm can be trivially modified for this more general assumption, and thus we focus on optimistic imagination-based exit detection.}

\begin{assumption}[$(\alpha, \zeta)$-coverage]
	\label{assumption:exit-coverage}
	Assume $(\MDP_t)_{t \in [T]}$ have a latent hierarchy with respect to $(\set{\cluster_k}, \entry{\cdot}, \exit{\cdot})$.
	There exists $\alpha, \zeta > 0$ such that for any $\set{(s_1, a_1), \dots, (s_n, a_n)} \subseteq \exit\statespace$,
	\begin{enumerate}[label=(\alph*)]
		\item For any $i \in [n]$, $(s_i, a_i)$ is $\alpha$-important for some meta-training MDP $\MDP_i$.

		\item For some $\MDP_t$ with $t \in [T]$, if we construct a new MDP $\bar\MDP_t = (\statespace, \actionspace, \bar\transition, r_t, H)$ via
		\[
		\bar\transition(\cdot \conditionedon s, a) =
		\begin{cases}
		\transition^{\MDP_i}(\cdot \conditionedon s, a) & (s, a) = (s_i, a_i) \\
		\transition^{\MDP_t}(\cdot \conditionedon s, a) & \text{otherwise}
		\end{cases},
		\]
		i.e. we replace $(s_i, a_i)$ dynamics with those from $\MDP_i$, then $V^{\bar\MDP_t, \ast}(s_0) > V^{\MDP_t, \ast}(s_0) + \zeta$.
	\end{enumerate}
\end{assumption}

Informally, $(\alpha, \zeta)$-coverage states that if not all exits have been found, then optimistic imagination can borrow dynamics for the remaining exits from other tasks to find a better optimal policy.

\subsection{Algorithm Outline}

\label{sec:src-alg-outline}

\begin{wrapfigure}[24]{r}{0.3\linewidth}
	\centering
  \vspace{-8.7ex}
	\begin{minipage}{\linewidth}
		\captionsetup{width=0.95\linewidth}
		\includegraphics[width=\linewidth]{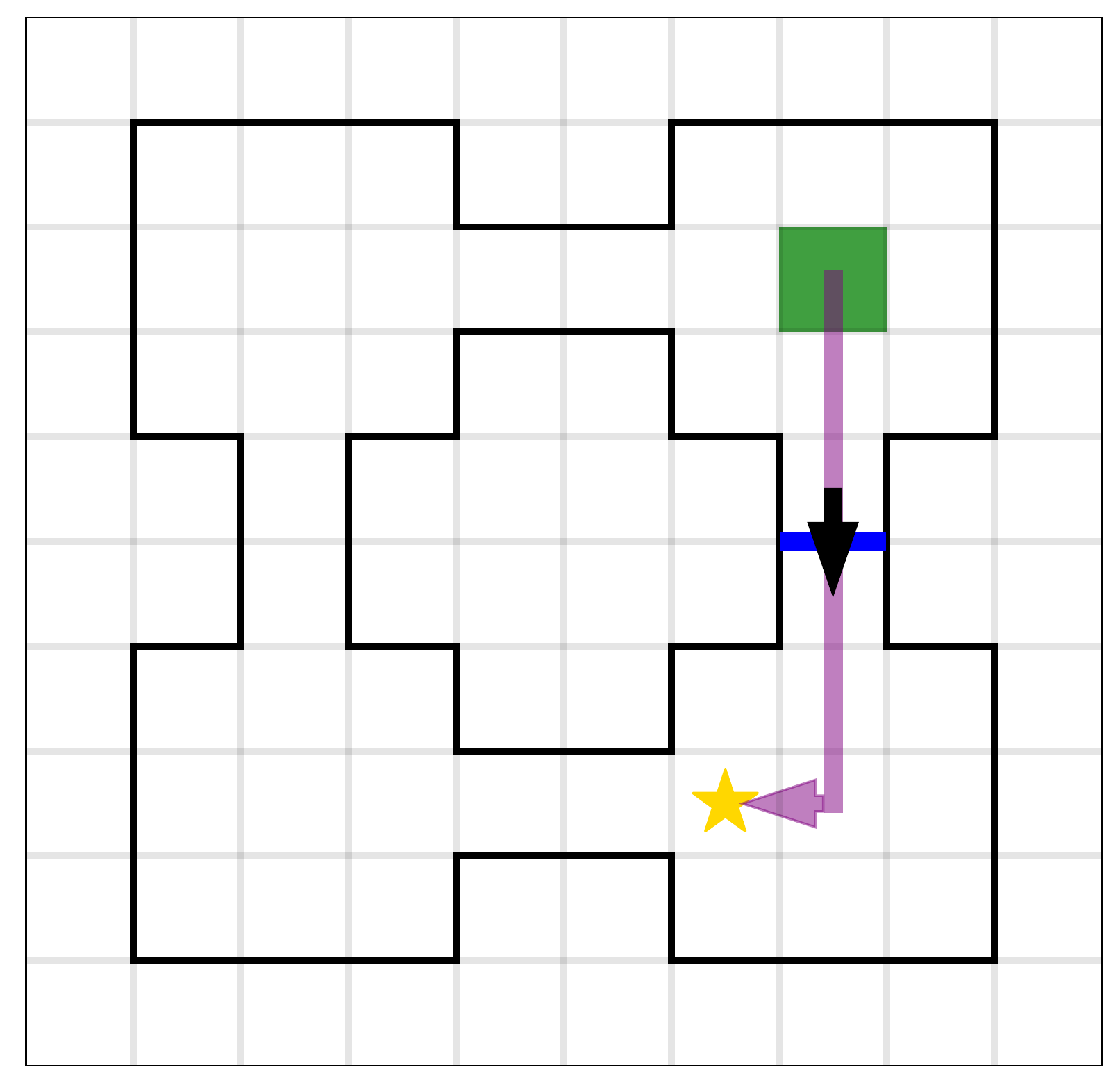}
		\caption{Phase I contribution to learning $\pi_{\mathcal{I}}$ in \figref{fig:optimistic-imagination-illustration}, marked with an arrow.}
		\label{fig:phase-i-contribution}
	\end{minipage}
	\begin{minipage}{\linewidth}
		\captionsetup{width=0.95\linewidth}
		\includegraphics[width=\linewidth]{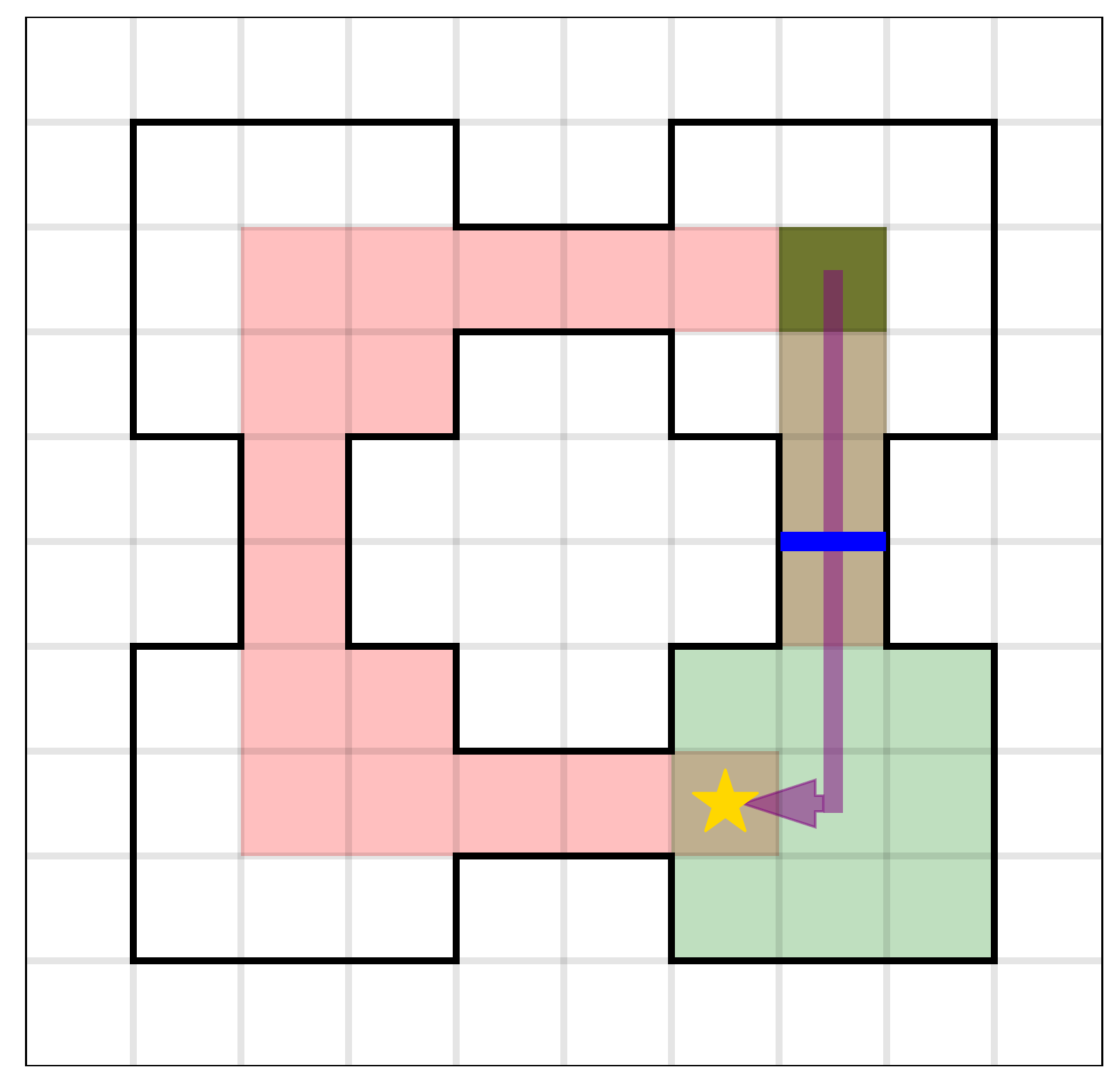}
		\caption{Phase II contribution. Learning $\pi_{\mathcal{I}}$ requires the green region, and thus optimal policy state coverage (red) is insufficient.}
		\label{fig:phase-ii-contribution}
	\end{minipage}
\end{wrapfigure}

In this section, we outline the algorithm that we use to detect exits.
Our procedure can be naturally divided into three phases: a task-solving phase, a reward-free phase, and an exit detection phase.
Throughout, we illustrate our steps in \figref{fig:optimistic-imagination-illustration}, showing how Phases I and II allow the learner to find the imagined policy $\pi_{\mathcal{I}}$.
The full details of the algorithm are provided in \secref{sec:src-alg-formal}.

\subsubsection{Phase I: Task-Specific Dynamics Learning}

First, we solve $\MDP_1, \dots, \MDP_T$ with UCBVI.
Using UCBVI regret bounds from \citet{azar2017minimax} together with \lemmaref{lemma:importance-implies-visitation-mt}, we can guarantee that all $\alpha$-important exits are sufficiently visited.
Thus, during optimistic imagination, the learner would be able to borrow high-quality estimates of exit dynamics from other tasks.
For example, a learner that has solved both tasks in \figref{fig:coverage-base-tasks} can borrow open blue gate dynamics for use in the first task during optimistic imagination, as shown in \figref{fig:phase-i-contribution}.

\subsubsection{Phase II: Reward-Free RL}

In order to perform optimistic imagination, the learner also needs to simulate non-borrowed $(s, a)$ dynamics.
This is done by fully learning the dynamics of one of the tasks, proving a template $\hat\transition_0$.
Learning $\hat\transition_0$ is achieved using reward-free RL \citep{jin2020reward}.

To understand the necessity of Phase II, note that in \figref{fig:phase-ii-contribution}, near-optimal policies (in red) have no coverage over states past the blue gate.
Therefore, dynamics estimates from Phase I are insufficient for optimistic imagination.
On the other hand, fully learning the dynamics in one of the tasks (which includes the green region) allows the learner to simulate the dynamics if the blue gate were open and successfully recover $\pi_{\mathcal{I}}$.

\subsubsection{Phase III: Exit Detection}

Having completed the previous two phases, the learner can use optimistic imagination to detect exits.
In particular, we consider a modified value iteration method where the learner optimistically chooses dynamics estimates from Phases I and II to perform Bellman backups.
This implicitly defines an MDP whose optimal value is at least as large as $\bar\MDP_t$ in \assumpref{assumption:dynamics-separation}, as $\bar\MDP_t$ would have been feasible for this process.
Analogously with $\alpha$-importance, this value gap implies that the corresponding optimal policy $\pi_{\mathcal{I}}$ for this new MDP must visit an $(s, a)$ pair whose dynamics are borrowed.
Therefore, by playing $\pi_{\mathcal{I}}$, the learner can determine a new exit.

\subsection{Meta-Training Guarantee}

We now outline our main result for the algorithm in \secref{sec:src-alg-outline}.
We first define a ``hierarchy oracle'' that will be used in downstream tasks:

\begin{definition}[Hierarchy oracle]
	\label{definition:hierarchy-oracle}
	Let $\terminate_S$ and $\terminate_F$ denote successful and failed termination, respectively.
	Consider any tuple $(x, f, r, \tilde{H})$ such that $x \in \entry\statespace$, $f: \exit\statespace \to \set{\terminate_S, \terminate_F}$, $r$ is a reward function, and $\tilde{H} \leq H$.
	Every such tuple induces an MDP $\MDP(x, f, r, \tilde{H}) = (\statespace, \actionspace, \transition_f, r, \tilde{H})$ whose starting state is $x$ and whose transition dynamics is given by
	\[
	\transition_f(\cdot \conditionedon s, a) =
	\begin{cases}
	\delta(f(s, a)) & (s, a) \in \exit\statespace \\
	\delta(s) & s \in \set{\terminate_S, \terminate_F} \\
	\transition_t(\cdot \conditionedon s, a) & \text{otherwise, for any $t \in [T]$.}
	\end{cases}
	\]
	An \term{$\epsilon$-suboptimal hierarchy oracle}, when queried with any valid $(x, f, r, \tilde{H})$, returns an $\epsilon$-suboptimal policy for $\MDP(x, f, r, \tilde{H})$.
\end{definition}

Informally, a hierarchy oracle can perform any task in an MDP where the clusters are disconnected.
This includes performing an exit as quickly as possible, or optimally collecting rewards within a cluster.
Our meta-training guarantee ensures that such an oracle is implementable:

\begin{theorem}[Meta-training guarantee, informal]
	\label{thm:meta-train-guarantee-informal}
	Under \assumpref{assumption:dynamics-separation}, \assumpref{assumption:exit-coverage} and other assumptions in \secref{sec:other-src-assumptions}, the data obtained from the algorithm in \secref{sec:src-alg-outline} allows for:
	\begin{enumerate}[label=(\alph*)]
		\item implementing an $\epsilon$-suboptimal hierarchy oracle, and

		\item determining, for every $s \in \entry\statespace$, the available exits in the cluster containing $s$,
	\end{enumerate}
	simultaneously with probability at least $1 - p$.
	Furthermore, this is achieved with query complexity:
	\[
	\tilde{O}\left[T\left(\frac{KL}{\alpha\min(\zeta, \beta)^2} + \frac{KS}{\alpha\zeta^2} + \frac{SA}{\min(\alpha, \zeta)^2} + \frac{KS^2A}{\alpha}\right) + \frac{S^4A}{\min(\epsilon, \zeta)} + \frac{S^2A}{\min(\epsilon, \zeta)^2}\right]\mathrm{poly}(H).
	\]
\end{theorem}

As a point of comparison, we have the following guarantee on brute-force hierarchy learning:

\begin{theorem}
	The brute-force approach outlined in \secref{sec:brute-force-hierarchy-learning}, under \assumpref{assumption:dynamics-separation} and \assumpref{assumption:exit-coverage}(a), determines the set of exits with high probability, incurring query complexity
	\[
	\tilde{O}\left[T\left(\frac{S^2A}{\alpha\beta^2} + \frac{SA}{\alpha^2} +  \frac{S^4A}{\alpha}\right)\right]\mathrm{poly}(H).
	\]
\end{theorem}

When $\alpha$, $\beta$, and $\zeta$ are of the same order, we see that the proposed method incurs a smaller query complexity compared to a brute force learner that has only learned the exits.
We provide proofs of both results in \secref{sec:meta-train-proofs}, along with all other necessary assumptions and full algorithm details.

	\section{Meta-Test Analysis}
		In this section, we provide regret bounds on learning an MDP $\MDP_\target$ using the hierarchy oracle.
We first characterize a family of tasks for which one can achieve improved regret bounds.
Furthermore, we provide sufficient conditions ensuring that using the hierarchy incurs low suboptimality.

\subsection{Assumptions}

In this section, we outline the assumptions that we make to prove a regret bound on the meta-test task.
Let $\MDP_\target = (\statespace, \actionspace, \transition_\target, r_\target, H)$ be the meta-test MDP.
We assume that $\transition_\target(\cdot \conditionedon s, a) = \transition_t(\cdot \conditionedon s, a)$ for any source task $t$ and $(s, a) \not\in \exit\statespace$.
Our first assumption restricts the set of tasks to those which are compatible with the hierarchical structure:

\begin{assumption}[Task Compatibility]
	\label{assumption:meta-test-compatibility}
	There exists a cluster $\cluster^\ast$ such that $r_\target$ is supported on $\interior{(Z^\ast)} \union \exit\statespace$.
	Furthermore, there exists an optimal policy $\pi^\ast$ satisfying
	\begin{enumerate}[label=\textup{(\alph*)}]
		\item Conditioned on $s_h \in \cluster^\ast$, we have that $(s_{h'}, a_{h'}) \not\in \exit{\cluster^\ast}$ for $h' \geq h$ almost surely.

		\item The number of exits encountered by $\pi^\ast$ is bounded by $\horizon_\effective$ with probability $\zeta$.
	\end{enumerate}
\end{assumption}

\paragraph{Hierarchical compatibility.}
Intuitively, the assumption on the reward function structure and condition (a) suggests that the task can be decomposed into a $(\cluster^\ast)$-searching phase and a within-$(\cluster^\ast)$ phase.
We expect the hierarchy oracle to reduce the complexity of exploration in both phases.
Thus, these conditions ensure compatibility with the learned hierarchy.

\paragraph{Temporal Abstraction.}
Since a hierarchical learner only needs to make decisions upon entering a new cluster (to decide which exit to use/whether to stay), one can expect a reduction in the planning horizon.
Condition (b) serves to quantify this reduction.
Note that in most practical settings, the failure probability can be expected to be small, even with modest values of $\horizon_\effective$.

\paragraph{Hierarchical Suboptimality.}
By restricting the learner to executing oracle-provided policies, we have reduced the feasible set of policies.
While this reduction leads to improved regret bounds, this also incurs approximation error, as the optimal policy may not lie in this restricted class.
We refer to this error as \emph{hierarchical suboptimality}.
We will show that hierarchical suboptimality is controllable with appropriate conditions on $\transition_T$, which require the following notions of reaching times:

\begin{definition}[Reaching times]
	Fix a cluster $\cluster$, starting and goal states $s, g \in \cluster$, and planning horizon $\tilde{H} \leq H$.
	For any policy $\pi$, let $(s_0, s_1, \dots, s_{\tilde{H}})$ be the states visited by $\pi$ from $s$, where $s_0 = s$.
	Then, we define
	\begin{equation*}
	\begin{aligned}
	T^\pi_{\tilde{H}}(s, g) &\defas \min \set{h \in \set{0, \dots, \tilde{H}} \suchthat \text{$s_h = g$ and $s_{h'} \in Z$ for $h' < h$}} \union \set{L} \\
	T^\ast_{\tilde{H}}(s, g) &\defas \inf_\pi\expt{T^\pi_{\tilde{H}}(s, g)} \\
	T^{\min}(s, g) &\defas \inf_{\pi}\min\set{h \in \naturals \suchthat \prob{T^\pi(s, g) = h} > 0}.
	\end{aligned} \qedhere
	\end{equation*}
\end{definition}

In words, $T_{\tilde{H}}^\pi(s, g)$ is the time $\pi$ takes to reach a state $g$ from $s$ while remaining within the same cluster.
By minimizing this quantity in expectation over all policies, we obtain $T_{\tilde{H}}^\ast(s, g)$.
Finally, $T^{\min}(s, g)$ is the minimum time for which it is possible to reach $g$ from $s$.

\begin{assumption}[Regular and low-variance dynamics]
	\label{assumption:env-low-var-and-regular}
	There exists $\alpha, \beta, \gamma > 0$ such that for any cluster $\cluster$, states $s, g \in \cluster$, and horizon $\tilde{H} < H$,
	\begin{enumerate}[label=(\alph*)]
		\item ($(\alpha, \beta)$-unreliability)
		For any deterministic policy $\pi$ with $\mathbb{E}[T_{\tilde{H}}^\pi(s, g)] - T^\ast_{\tilde{H}}(s, g) < \alpha$, $T_{\tilde{H}}^\pi(s, g)$ has a sub-Gaussian upper tail with variance proxy $\beta^2\mathbb{E}[T_{\tilde{H}}^\pi(s, g)]^2$.

		\item ($\gamma$-goal-reaching suboptimality) $T_{\tilde{H}}^\ast(s, g) \leq (1 + \gamma)T^{\min}(s, g)$.
	\end{enumerate}
\end{assumption}

To understand $(\alpha, \beta)$-unreliability, note that the condition only considers near-optimal deterministic policies.
Therefore, the condition controls the randomness in $T^\pi(s, g)$ derived from the transition dynamics.
On the other hand, (b) is a regularity condition, ensuring that near-optimal goal-reaching policies reach their goals as quickly as possible.
Deterministic environments satisfy these conditions with $\alpha = \infty$ and $\beta = \gamma = 0$.
We provide an extended discussion of these assumptions, including failure cases without them, in \secref{sec:low-var-discussion}.
Note that the guarantees of \assumpref{assumption:env-low-var-and-regular} scales with the cluster width, and thus we have the following final assumption:

\begin{assumption}
	\label{assumption:skill-horizon}
	For any cluster $\cluster$ and $s, g \in Z$ with $s \neq g$, $T^{\min}(s, g) \leq W$.
	Furthermore, $\horizon_\effective W \ll H$.
\end{assumption}

\assumpref{assumption:skill-horizon} limits the length of the subtasks within each cluster.
This is consistent with hierarchy-based methods in practice, with skills only being executed for a limited amount of time.
This width bound, together with \assumpref{assumption:meta-test-compatibility}, suggests that $\pi^\ast$ requires $O(\horizon_\effective W)$ timesteps with high probability.
Therefore, the condition $\horizon_\effective W \ll H$ means that the task horizon is much longer than the minimum time required to complete the task, which often holds in practice.

\subsection{Meta-test Regret Guarantee}

Under the assumptions in the previous section, we have the following meta-test guarantee:

\begin{theorem}
  We work under Assumptions \ref{assumption:meta-test-compatibility}, \ref{assumption:env-low-var-and-regular}, and \ref{assumption:skill-horizon}.
  Furthermore, assume that the learner has access to an $\epsilon$-suboptimal hierarchy oracle as guaranteed by \thmref{thm:meta-train-guarantee-informal}, where $\epsilon < \alpha$.
  Then, a learner that applies the procedure in \secref{sec:meta-test-learning-procedure} to $\MDP_\target$ incurs regret
	\begin{align*}
  	\mathrm{Regret}(N) &\lesssim \sqrt{H^2\horizon_\effective WLMN} + N\epsilon_{\mathrm{subopt}} \\
  	\epsilon_{\mathrm{subopt}} &\defas (1 + \horizon_\effective + \beta\sqrt{\horizon_\effective})\epsilon + \left[\gamma\horizon_\effective + \beta(1 + \gamma)\sqrt{\horizon_\effective}\right]W + \zeta H.
	\end{align*}
  with high probability.
\end{theorem}

Observe that the irreducible hierarchical suboptimality $\epsilon_{\mathrm{subopt}}$ (i.e. when $\epsilon = 0$) tends to zero as $\gamma, \beta, \zeta \to 0$.
In particular, environments with deterministic in-cluster dynamics do not incur hierarchical suboptimality.
We prove this regret bound in \secref{sec:meta-test-proofs}.

\begin{wrapfigure}[18]{r}{0.35\linewidth}
	\centering
	\vspace{-3ex}
	\includegraphics[width=\linewidth]{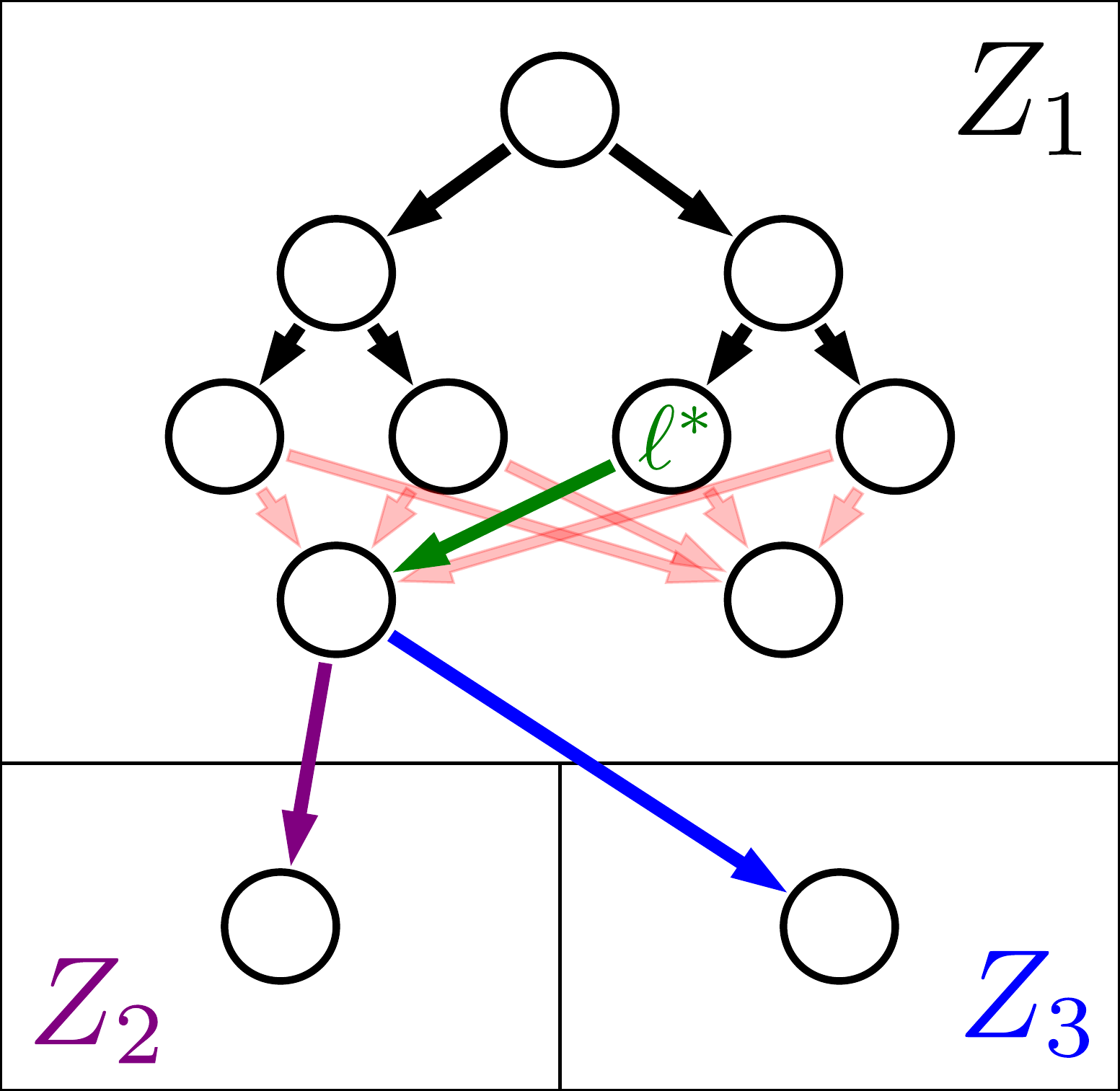}
	\caption{The binary tree environment, which demonstrates an exponential regret separation between hierarchy-based and hierarchy-oblivious learners.}
	\label{fig:binary-tree-env}
\end{wrapfigure}

\paragraph{When does knowing the hierarchy help?}
Consider the binary tree environment in \figref{fig:binary-tree-env}.
All of the leaves take the learner to a state with exits with probability $1/2$, with the exception of a special leaf $\ell^\ast$ that does so with probability $(1/2) + \epsilon$.
Rewards can only be collected upon performing one of the exit actions (blue/purple).
To achieve low regret, a learner has to quickly identity $\ell^\ast$ and the correct exit.

We consider the collection of task distributions indexed by $\ell^\ast$ that randomize the reward-granting exit action.
Knowing the hierarchy amounts to knowing $\ell^\ast$, reducing the exploration problem to determining the correct exit action.
However, a hierarchy-oblivious learner needs to explore the tree, leading to regret that is exponential in the tree depth.
Formally, we have the following result:

\begin{theorem}
	There exists a family of task distributions such that any hierarchy-oblivious learner incurs expected regret lower bounded by $\Omega(2^{W/2}\sqrt{H^2N})$ on at least one task distribution.
	In contrast, a learner with access to a $0$-suboptimal\footnote{We use a $0$-suboptimal hierarchy oracle for the separation result for ease of presentation.} hierarchy oracle incurs regret bounded by $O(\sqrt{H^2N})$ with high probability, over any sampled task from any of the task distributions.
\end{theorem}

We prove this result in \secref{sec:hardness-result-formal-proof}, using recent results by \citet{domingues2021episodic} which demonstrate that the set of binary tree subproblems above form a set of minimax instances for any RL algorithm.
This separation result suggests that hierarchy-based learners gain in situations where temporally extended exploratory behaviors are needed.
This corroborates the experimental findings of \citet{nachum2019does}, which attributes the benefits of hierarchical RL to improved exploration.

	\section{Conclusion}
		We have demonstrated that certain natural coverage conditions allow for learning useful hierarchies from tasks.
Interesting future directions include analyzing hierarchy-based multi-task RL and extending the ideas in this work to continuous state and/or action spaces.
Another interesting direction would be to provide sample-efficient algorithms for learning additional structures that can be imposed on the learned hierarchy, such as cluster equivalences as in \citet{wen2020efficiency}.

	\section{Acknowledgements}
		KC is supported by a National Science Foundation Graduate Research Fellowship, Grant DGE-2039656.
		QL is supported by NSF \#2030859 and the Computing Research Association for the CIFellows Project.
		JDL acknowledges support of the ARO under MURI Award W911NF-11-1-0304,  the Sloan Research Fellowship, NSF CCF 2002272, NSF IIS 2107304, and an ONR Young Investigator Award.
		Additionally, we thank Aurick Zhou for discussions and feedback.

	\bibliography{references}
	\bibliographystyle{iclr2022_conference}

	\clearpage

	\appendix

	\section{Meta-Training Proofs}
		\label{sec:meta-train-proofs}
		\subsection{Algorithm}
  \label{sec:src-alg-formal}

  In this section, we provide the complete algorithm for exit detection with optimistic imagination.
  For readability, we separate the three phases.

  \subsubsection{Phase I: Task-Specific Learning}

    \begin{algorithm}[h]
    	\caption{Exit Detection, Phase I: Task-Specific Learning}
    	\label{alg:phase-i}
    	\begin{algorithmic}[1]
    		\Require{Tasks $\MDP_1, \dots, \MDP_T$, $N_\UCBVI$ \UCBVI{} iterations, $N_\taskspecific$ policy samples, threshold $N_\thresh^\taskspecific$}
    		\ForAll{$t \in [T]$}
    		\State $\dataset^t \gets \emptyset$.
    		\State Obtain policy set $\Phi^{t} \gets \textsc{UCBVI}$($\MDP_t$, $N_\UCBVI$).
    		\ForAll{$n = 1, \dots, N_{\taskspecific}$}
    		\State Sample $\pi \sim \uniform{\Phi^t}$.
    		\State Play $\pi$ in $\MDP_t$, add all $(s, a, s')$ pairs to $\dataset^t$, get sum of rewards $\hat{V}^{(n)}$.
    		\EndFor
    		\State Form estimated dynamics model $\hat\transition_t$ from $\dataset_t$.
    		\State Form optimal value estimate $\displaystyle\hat{V}_t \gets \frac{1}{N_\taskspecific}\sum_{n = 1}^{N_\taskspecific}\hat{V}^{(n)}$
    		\State $N_t(s, a) \gets \card{\set{(x, u, x') \in \dataset_t \suchthat x = s, u = a}}$.
    		\ForAll{$(s, a) \in \statespace \times \actionspace$}
    		\If{$N_t(s, a) < N_{\thresh}^{\taskspecific}$} $\hat\transition_t(\cdot \conditionedon s, a) \gets 0$. \EndIf
    		\EndFor
    		\EndFor
    		\State \Return dynamics estimates $\hat\transition_t$ and value estimates $\hat{V}_t$ for $t \in [T]$.
    	\end{algorithmic}
    \end{algorithm}

  \subsubsection{Phase II: Learning Reference Dynamics}

    \begin{algorithm}[h]
    	\caption{Exit Detection, Phase II: Learning Reference Dynamics}
    	\label{alg:phase-ii}
    	\begin{algorithmic}[1]
    		\Require{MDP $\MDP$, $N_\Euler^\rewardfree$ Euler iterations, $N_\rewardfree$ policy samples}
    		\State Set policy class $\Psi \gets \emptyset$ and dataset $\dataset_\rewardfree \gets \emptyset$
    		\ForAll{$g \in \statespace$}
    		\State Create MDP $\MDP_g$ from $\MDP$ with horizon $2H$ and $\prob{\terminate \conditionedon g, a} = 1$ for any $a$.
    		\State $r_g(s, a) \gets \ind{s = g}$ for any $(s, a) \in (\statespace \union \set{\terminate}) \times \mathcal{A}$.
    		\State $\Phi^g \gets$ \textsc{Euler}($\MDP_g, r_g, N_\rewardfree^\Euler$)
    		\State $\pi_h(\cdot \conditionedon g) \gets \uniform{\actionspace}$ for $\pi \in \Phi^g$, $h \in [H]$.
    		\State Add policies in $\Phi^g$ to $\Psi$.
    		\EndFor
    		\ForAll{$n = 1, \dots, N_\rewardfree$}
    		\State Sample $\pi \sim \uniform{\Psi}$.
    		\State Play $\pi$ in $\MDP$ and obtain trajectory $(s_0, a_0, \dots, s_{2H})$.
    		\State Sample $h \distas \uniform{[2H]}$ and add $(s_h, a_h, s_{h + 1})$ to $\dataset_\rewardfree$.
    		\EndFor
    		\State \Return reference dynamics $\hat\transition_0$ formed from $\dataset_\rewardfree$.
    	\end{algorithmic}
    \end{algorithm}

  \clearpage

  \subsubsection{Phase III: Exit Detection}

    \begin{algorithm}[h!]
    	\caption{Exit Detection, Phase III: Exit Detection}
    	\label{alg:phase-iii}
    	\begin{algorithmic}[1]
    		\Require{$N_\exitdetection, N_\thresh^\exitdetection, N_\Euler^\exitlearning, N_{\mathrm{EL}}$ policy samples}

    		\State Initialize $\textsc{IsExit}[s, a] \gets \textsc{False}$ for $(s, a) \in \statespace \times \actionspace$.
    		\While{True}
    		\ForAll{$t \in [T]$} \label{algostep:phase-iii-iteration}
    		\State $\hat\transition_0(\cdot \conditionedon s, a) \gets \hat\transition_t(\cdot \conditionedon s, a)$ for $(s, a) \in \statespace \times \actionspace$ with $\textsc{IsExit}[s, a]$.
    		\State $\tilde{V}^t, \tilde{Q}^t \gets$ \textsc{OptImgVI}($\hat\transition_0$, $(\hat\transition_1, \dots, \hat\transition_T)$, $r_t$, \textsc{IsExit})
    		\If{$\tilde{V}^t_0(s_0) - \hat{V}_t > (2/3)\zeta$}
    		\State Run greedy policy with respect to $\tilde{Q}$ $N_{\exitdetection}$ times and form estimate $\hat\transition$ for $(s, a)$ pairs visited at least $N_\thresh'$ times. \label{algostep:exit-detect-and-learn-start}
    		\ForAll{$(s, a) \in \statespace \times \actionspace$ with $\hat\transition(\cdot \conditionedon s, a) \not\equiv 0$}
    		\If {$(\exists t \in [T]) \ \hat\transition_t(\cdot \conditionedon s, a) \not\equiv 0$ and $\TV{\hat\transition(\cdot \conditionedon s, a) - \hat\transition_t(\cdot \conditionedon s, a)} > \beta/2$}
    		\State $\textsc{IsExit}[s, a] \gets$ True
    		\State $\hat\transition_t(\cdot \conditionedon s, a) \gets \textsc{Learn-Exit}(\MDP_t, (s, a), N_\Euler^\exitlearning, N_{\exitlearning})$ \textbf{for all} $t \in [T]$ \label{algostep:exit-detect-and-learn-end}
    		\EndIf
    		\EndFor
    		\EndIf
    		\If{no new exits found after passing through $T$ tasks since last found exit}
    		\State \Return \textsc{IsExit}
    		\EndIf
    		\EndFor
    		\EndWhile
    	\end{algorithmic}
    \end{algorithm}

    \begin{algorithm}[h]
    	\caption{\textbf{B}orrowing \textbf{O}ptimistically \textbf{A}cross \textbf{T}asks during \textbf{V}alue \textbf{I}teration (BOAT-VI)}
    	\label{alg:boat-vi}
    	\begin{algorithmic}[1]
    		\Procedure{BOAT-VI}{Reference dynamics $\hat\transition_0$, Estimated dynamics $(\hat\transition_1, \dots, \hat\transition_T)$, \par\hspace{6.9em} Reward function $r$, Table \textsc{IsExit}[$\statespace \times \actionspace$]}
    		\State $\hat V_{H}(s) \gets 0$ for $s \in \statespace$.
    		\ForAll{$h = H-1, \dots, 0$}
    		\ForAll{$(s, a) \in \statespace \times \actionspace$}
    		\If{$\textsc{IsExit}[s, a]$}
    		\State $\hat Q_h(s, a) \gets r(s, a) + \hat\transition_0\hat V_{h + 1}(s, a)$
    		\Else
    		\State $\hat Q_h(s, a) \gets r(s, a) + \max_{t = 0, \dots, T}\hat\transition_t\hat V_{h + 1}(s, a)$
    		\EndIf
    		\EndFor
    		\State $\hat V_h(s) \gets \max_a\hat Q_h(s, a)$ for $s \in \statespace$.
    		\EndFor
    		\State \Return $\hat{V}$, $\hat{Q}$
    		\EndProcedure
    	\end{algorithmic}
    \end{algorithm}

    \begin{algorithm}[h]
    	\caption{Exit-learning subroutine}
    	\label{alg:exit-learning}
    	\begin{algorithmic}[1]
    		\Procedure{Learn-Exit}{MDP $\MDP$, exit $(s, a)$, $N_\Euler^\exitlearning$ \textsc{Euler} iterations, $N_\mathrm{EL}$ policy samples}
    		\State Create MDP $\tilde\MDP$ from $\MDP$ so $\prob{\terminate \conditionedon s, a} = 1$.
    		\State $\tilde r(s', a') \gets \ind{(s', a') = (s, a)}$ for any $(s', a') \in (\statespace \union \set{\terminate}) \times \mathcal{A}$.
    		\State $\Psi \gets$ \textsc{Euler}($\tilde\MDP, \tilde r, N_\Euler^\exitlearning$)
    		\ForAll{$n = 1, \dots, N_{\mathrm{EL}}$}
    		\State Sample $\pi \sim \uniform{\Psi}$.
    		\State Play $\pi$ in $\MDP$ and obtain trajectory $(s_0, a_0, \dots, s_H)$.
    		\EndFor
    		\State \Return reference dynamics $\hat\transition(\cdot \conditionedon s, a)$ formed from all trajectory data.
    		\EndProcedure
    	\end{algorithmic}
    \end{algorithm}

\clearpage

\subsection{Other Assumptions}
  \label{sec:other-src-assumptions}

  The remaining assumptions quantify the reachability of certain states.
  First, we have the following assumption, which in effect ensures that one can reach most states regardless of exit configuration in the meta-training tasks:

  \begin{assumption}[Non-limiting exit configurations]
  	\label{assumption:non-limiting-exits}
  	Let $\MDP = (\statespace, \actionspace, \transition, H)$ be any reward-free environment with time-varying dynamics
  	\[
  	\transition^{(h)}(\cdot \conditionedon s, a) = \transition_{t(h, s, a)}(\cdot \conditionedon s, a) \quad \text{for some $t: [H] \times \statespace \times \actionspace \to [T]$.}
  	\]
  	Then, there exists $C > 1$ such that for any $s$ and $t \in [T]$,
  	\[
  	\max_{\pi}P_\MDP(s \in \tau_\pi) \leq C\max_{\pi}P_{\MDP_t}(s \in \tau_\pi)
  	\]
  \end{assumption}

  Intuitively, the assumption states that the reachability of a state in $\MDP_t$ would not be significantly improved even under an optimal configuration of the exits.
  Therefore, running reward-free RL on one of the meta-training tasks is sufficient for learning all non-exit $(s, a)$ pairs.

  \begin{remark}
  	We note that \assumpref{assumption:non-limiting-exits} is restrictive in that it requires that every state be roughly reachable in any of the meta-training MDPs.
  	This may not hold in practice, e.g., consider a four-room environment where one of the rooms is blocked off for one of the tasks.
  	However, this can be weakened to requiring that $s$ be reachable in at least one of $N$ arbitrarily chosen meta-training tasks.
  	This would require that the algorithm run Phase II over $N$ meta-training tasks, which results in a benign increase in the query complexity of the algorithm, so long as $N$ is a constant much smaller than $T$.
  	We focus on the $N = 1$ case for ease of presentation.
  \end{remark}

  To simplify the presentation of the rest of the assumptions, we recall the following definition of $\delta$-significance in \citet{jin2020reward}:

  \begin{definition}
  	A state $s$ is \term{$\delta$-significant} if $\max_\pi P(s \in \tau_\pi) \geq \delta$.
  	Additionally, we say that $(s, a)$ is $\delta$-significant if $s$ is $\delta$-significant.
  \end{definition}

  Note that we have modified the definition to remove the dependence on the timestep $h \in [H]$.
  This is because the dynamics are stationary, and thus it does not matter when $s$ is visited in a trajectory.

  Having defined $\delta$-significance, we now have the following assumption, which simply quantifies the reachability of every entrance:

  \begin{assumption}[$\rho$-significant entrances]
  	\label{assumption:entrance-reachability}
  	For any $s \in \entry\statespace$, $s$ is $\rho$-significant for all of the tasks.
  \end{assumption}

  This assumption merely quantifies the reachability of all entrances and is nonrestrictive given \assumpref{assumption:non-limiting-exits}, which already ensures that reachability in task implies reachability in the other tasks\footnote{One can weaken this assumption in a way that is compatible with the weakened form of \assumpref{assumption:non-limiting-exits}.}.

  Finally, we also need to assume that every exit in a cluster is reachable from every entrance in that same cluster.
  Otherwise, if an exit is hard to reach, then it can be ignored for learning purposes since the probability of using the exit is very low.

  \begin{assumption}[In-cluster exit reachability]
  	\label{assumption:in-cluster-exit-reachability}
  	Fix any cluster $\cluster$, entry $s \in \entry\cluster$, and exit $(g, a) \in \exit\cluster$.
  	Consider the reward-free environment $\MDP_t|_\cluster = (\cluster, A, \transition_t|_\cluster, H)$, where $\transition_t|_\cluster$ is the restriction of $\transition_t$ to $\cluster \times \actionspace$ and the starting state is $s$.
  	Then, $g$ is $\delta$-significant in $\MDP_t|_\cluster$ for any $t \in [T]$.
  \end{assumption}

  The requirement that the assumption hold for any $t \in [T]$ is without loss of generality since non-exit dynamics do not change.

\subsection{Verifying Exit Detection}

  In this section, we demonstrate that the algorithm in \secref{sec:src-alg-formal} can successfully discover $\exit\statespace$ with high probability.
  Formally, we have the following result:

  \begin{theorem}[Provable exit detection]
  	\label{thm:provable-exit-detection}
  	Assume we run the algorithm in \secref{sec:src-alg-formal} with the parameter choices given in \tableref{table:src-param-choices}.
  	Then, with probability at least $1 - p$, the algorithm returns an array \textsc{IsExit} satisfying:
  	\[
  	\set{(s, a) \suchthat \textsc{IsExit}[s, a]} = \exit\statespace.
  	\]
  \end{theorem}

  \begin{table}[h!]
  	\centering
  	\setlength\extrarowheight{3ex}
  	\begin{tabular}{|c|c|}
  		\hline
  		Parameter & Value \\[3ex]\hline
  		$N_\UCBVI$ & $\displaystyle\frac{H^2SA}{\min(\alpha, \zeta)^2}\log^2\frac{HSAT}{p}$ \\[3ex]\hline
  		$N_\thresh^\taskspecific$ & $\displaystyle S\max\left(\frac{H^4}{\zeta^2}, \frac{1}{\beta^2}\right)\log\frac{SAHT}{p\alpha\min(\beta, \zeta)}$ \\[3ex]\hline
  		$N_\taskspecific$ & $\displaystyle S\max\left(\frac{H^5}{\alpha\zeta^2}, \frac{H}{\alpha\beta^2}\right)\log\frac{SAHT}{p\alpha\min(\beta, \zeta)} + \frac{H^2}{\min(\alpha, \zeta)^2}\log\frac{SAT}{p}$ \\[3ex]\hline
  		$N_\Euler^{\rewardfree}$ & $\displaystyle\frac{H^2S^4A}{\min(\rho\min(\epsilon, \epsilon_0), \zeta/C)}\log^3\frac{HSA}{p}$ \\[3ex]\hline
  		$N_\rewardfree$ & $\displaystyle\frac{H^5S^2A}{\min(\rho\min(\epsilon, \epsilon_0)^2, \zeta^2/C)}\log\frac{A}{p}$ \\[3ex]\hline
  		$N^\exitdetection_\thresh$ & $\displaystyle\frac{S}{\beta^2}\log\frac{SAH}{p\zeta\beta}$ \\[3ex]\hline
  		$N_\exitdetection$ & $\displaystyle\frac{HKS}{\zeta\beta^2}\log\frac{SAH}{p\zeta\beta} + \frac{H^2K^2}{\zeta^2}\log\frac{K}{p}$ \\[3ex]\hline
  		$N_\thresh^\exitlearning$ & $\displaystyle L\max\left(\frac{H^4}{\zeta^2}, \frac{1}{\beta^2}\right)\log\frac{CSAHT}{p\alpha\min(\beta, \zeta)}$ \\[3ex]\hline
  		$N_\Euler^\exitlearning$ & $\displaystyle\frac{CH^3S^2A}{\alpha}\log^3\left(\frac{HSAT}{p}\right)$ \\[3ex]\hline
  		$N_\exitlearning$ & $\displaystyle L\max\left(\frac{CH^5}{\alpha\zeta^2}, \frac{CH}{\alpha\beta^2}\right)\log\frac{CSAHT}{p\alpha\min(\beta, \zeta)} + \frac{C^2H^2}{\alpha^2}\log\frac{SAT}{p}$ \\[3ex]\hline
  	\end{tabular}
  	\caption{Table of parameters for the results in \thmref{thm:provable-exit-detection}. Since $K \leq SA$ and $L \leq S$, the agent does not need to know $K$ or $L$ in advance, at the expense of a worse sample complexity bound.}
  	\label{table:src-param-choices}
  \end{table}

  To prove this result, we proceed with a phase-by-phase analysis of the algorithm in \secref{sec:src-alg-formal}, which we then compile into proof of the desired result.

  \subsubsection{Phase I Analysis}

    First, we prove that during Phase I, \algoref{alg:phase-i} sufficiently visits all relevant exits and that all value estimates are sufficiently close.
    Formally, we have the following result:

    \begin{proposition}
    	\label{prop:phase-i-guarantee}
    	Set
    	\[
    	N_{\thresh}^\taskspecific = \Omega\left[S\max\left(\frac{H^4}{\zeta^2}, \frac{1}{\beta^2}\right)\log\frac{SAHTN_\taskspecific}{p}\right],
    	\]
    	and consider the following procedure applied to one of the meta-training tasks $\MDP_t$:
    	\begin{enumerate}[label=\textup{\arabic*.}]
    		\item \UCBVI{} is run for
    		\[
    		N_\UCBVI{} = \Omega\left(\frac{H^2SA}{\min(\alpha, \zeta)^2}\log^2\frac{HSAT}{p}\right)
    		\]
    		iterations, generating policies $\pi_1^{(t)}, \dots, \pi_{N_{\UCBVI{}}}^{(t)}$.

    		\item The learner uniformly samples
    		\[
    		N_\taskspecific = \Omega\left(\frac{H}{\alpha}N_\thresh^\taskspecific + \frac{H^2}{\min(\alpha, \zeta)^2}\log\frac{TK}{p}\right)
    		\]
    		policies from the previous step, runs each policy in $\MDP_t$, and obtains a dataset of transitions $\dataset_t$ and returns $\hat{V}^{(1)}, \dots, \hat{V}^{(N_\taskspecific)}$.
    	\end{enumerate}
    	Then, with probability at least $1 - p/3T$,
    	\begin{enumerate}[label=\textup{(\alph*)}]
    		\item We have the regret bound
    		\[
    		V^\ast_0(s_0) - \frac{1}{N_\UCBVI}\sum_{k = 1}^{N_\UCBVI}V_0^{\pi_k^{(t)}}(s_0) < \frac{\zeta}{6}.
    		\]

    		\item The set of obtained returns satisfy
    		\[
    		\abs{\frac{1}{N_\taskspecific}\sum_{i=1}^{N_\taskspecific}\hat{V}^{(i)} - \frac{1}{N}\sum_{k = 1}^{N_\UCBVI}V_0^{\pi_k^{(t)}}(s_0)} < \frac{\zeta}{6}.
    		\]

    		\item If $(s, a)$ is $\alpha$-important for $\MDP_t$, then $N_t(s, a) \geq N_\thresh$.

    		\item For every $(s, a)$ pair such that $N_t(s, a) \geq N_\thresh$,
    		\[
    		\sup_{f: \statespace \to [0, H]}\abs{\left[(\hat\transition_t - \transition_t)f\right](s, a)} < \min\left(\frac{\zeta}{24H}, \frac{\beta H}{2}\right).
    		\]
    	\end{enumerate}
    \end{proposition}

    To prove the above result, we first recall the following regret bound on \UCBVI{}, as proven by \citet{azar2017minimax}:

    \begin{lemma}[\textsc{UCBVI} regret bound]
    	\label{lemma:ucbvi-regret-bound}
    	For sufficiently large $N$, with probability at least $1 - p/6$,
    	\[
    	V_0^\ast(s_0) - \frac{1}{N}\sum_{k = 1}^{N}V_0^{\pi_k}(s_0) \lesssim \sqrt{\frac{H^2SA}{N}}\log\left(\frac{HSAN}{p}\right).
    	\]
    \end{lemma}

    As we will see later on, with our choice of $N_\UCBVI$, we obtain the desired regret bound in (a).
    Additionally, by Hoeffding's inequality, the average of $N_\taskspecific$ returns concentrates around the desired quantity with high probability, proving (b).
    Thus, all that remains is ensuring that every $\alpha$-important exit is sufficiently visited, and thus their dynamics are sufficiently well-estimated.

    Recall from the main text that the key step is demonstrating that a near-optimal policy for a task must visit its $\alpha$-important states with non-negligible probability:

    \begin{lemma}
    	\label{lemma:importance-implies-visitation}
    	Let $(s, a)$ be $\alpha$-important for $\MDP$, and let $\pi$ be an $\epsilon$-suboptimal policy for $\epsilon < \alpha$.
    	Then,
    	\[
    	\prob{(s, a) \in \tau_\pi} > \frac{1}{H}(\alpha - \epsilon).
    	\]
    \end{lemma}
    \begin{proof}
    	By $\alpha$-importance,
    	\begin{align*}
    	\alpha &\leq V^{\MDP, \ast}_0(s_0) - V^{\MDP^{\setminus (s, a)}, \ast}_0(s_0) \leq \left[V^{\MDP, \ast}_0(s_0) - V^{\MDP, \pi}_0(s_0)\right] + \left[V^{\MDP, \pi}_0(s_0) - V^{\MDP^{\setminus (s, a)}, \ast}_0(s_0)\right] \\
    	&\leq \left[V^{\MDP, \pi}_0(s_0) - V^{\MDP^{\setminus (s, a)}, \ast}_0(s_0)\right] + \epsilon.
    	\end{align*}
    	Therefore, by applying \lemmaref{lemma:value-gap-implies-visitation} and noting that $\set{\Delta \intersect \tau_\pi \neq \emptyset} = \set{(s, a) \in \tau_\pi}$, we obtain the desired result.
    \end{proof}

    Through the prior result, we can relate the \UCBVI{} regret bound to the probability that a randomly chosen \UCBVI{}-generated policy visits an $\alpha$-important state:

    \begin{lemma}
    	\label{lemma:ucbvi-visit-prob}
    	Let $(s, a)$ be $\alpha$-important for $\MDP$.
    	Assume that \textsc{UCBVI}, when run for
    	\[
    	N_{\UCBVI} = \Omega\left(\frac{H^2SA}{\alpha^2}\log^2\frac{HSA}{p}\right)
    	\]
    	iterations, generates policies $\pi_1, \dots, \pi_N$.
    	If we sample $\pi$ uniformly from these policies and let $\tau$ be the (random) trajectory generated by this randomly selected policy, then
    	\[
    	\prob{(s, a) \in \tau} > \frac{\alpha}{2H},
    	\]
    	conditioned on the high probability event in \lemmaref{lemma:ucbvi-regret-bound}.
    \end{lemma}
    \begin{proof}
    	By \lemmaref{lemma:importance-implies-visitation}, for any fixed $\pi$, we can write
    	\[
    	\prob{(s, a) \in \tau_\pi} \geq \frac{1}{H}(\alpha - [V_0^\ast(s_0) - V_0^\pi(s_0)])_+,
    	\]
    	where $x_+ = x\ind{x > 0}$.
    	Then, since $\pi$ is chosen randomly from the policies generated by \textsc{UCBVI},
    	\begin{align*}
    	\prob{(s, a) \in \tau} &= \frac{1}{N}\sum_{k = 1}^{N}\prob{(s, a) \in \tau_{\pi_k}} \geq \frac{1}{HN}\sum_{k = 1}^{N}(\alpha - [V_0^\ast(s_0) - V_0^{\pi_k}(s_0)])_+ \\
    	&\geq \frac{1}{H}\left[\alpha - \frac{1}{N}\sum_{k = 1}^{N}V_0^\ast(s_0) - V_0^{\pi_k}(s_0)\right].
    	\end{align*}
    	Therefore, by applying the regret bound in \lemmaref{lemma:ucbvi-regret-bound} and the choice of $N$, we find that
    	\[
    	\prob{(s, a) \in \tau} > \frac{\alpha}{2H}. \qedhere
    	\]
    \end{proof}

    With all of the above intermediate results, we can now prove \propref{prop:phase-i-guarantee}.

    \begin{proof}[Proof of \propref{prop:phase-i-guarantee}]
    	Throughout this proof, we condition on the high-probability event in \lemmaref{lemma:ucbvi-regret-bound}, instantiated to occur with probability at least $p/12T$.

    	\begin{enumerate}[label=(\alph*)]
    		\item By the choice of $N_\UCBVI$,
    		\[
    		N_{\UCBVI} \gtrsim \frac{H^2SA}{\zeta^2}\log^2\frac{HSAT}{p},
    		\]
    		and thus, we obtain the desired bound by plugging this value into the regret bound provided by \lemmaref{lemma:ucbvi-regret-bound}.

    		\item Note that $(\hat{V}^{(i)})$ are i.i.d., bounded in $[0, H]$, and for any $i \in [N_\taskspecific]$,
    		\[
    		\expt{\hat{V^{(i)}}} = \frac{1}{N}\sum_{k = 1}^{N}V_0^{\pi_k^{(t)}}(s_0).
    		\]
    		Therefore, by applying Hoeffding's inequality, with probability at least $1 - p/12T$,
    		\[
    		\abs{\frac{1}{N_\taskspecific}\sum_{i = 1}^{N_\taskspecific}\hat{V}^{(i)} - \frac{1}{N_\UCBVI}\sum_{k = 1}^{N_\UCBVI}V^{\pi_k^{(t)}}_0(s_0)} \lesssim \sqrt{\frac{H^2}{N_\taskspecific}\log\frac{T}{p}}
    		\]
    		The result immediately follows from the fact that $N_\taskspecific \gtrsim (H^2/\zeta^2)\log(T/p)$.

    		\item The result simply follows from \lemmaref{lemma:dynamics-estimation-error-bound} instantiated with failure probability $1 - p/12T$, together with the choice of $N_\thresh^\taskspecific$.

    		\item With the choice of $N_\UCBVI$, the conclusion of \lemmaref{lemma:ucbvi-regret-bound} can be made to hold with probability at least $1 - p/24T$.
    		Fix an $\alpha$-important exit $(s, a)$ for $\MDP_t$, so that the probability that $(s, a)$ is visited by the procedure is at least $\alpha/2H$.
    		By \lemmaref{lemma:sampling-to-meet-threshold}, sampling $N_\taskspecific$ trajectories is sufficient to ensure that $N_t(s, a) \geq N_\thresh^\taskspecific$ with probability at least $1 - p/24TK$.
    		Therefore, by performing a union bound over the set of $\alpha$-important exits (which contains at most $K$ elements), $N_t(s, a) \geq N_\thresh^\taskspecific$ for any $\alpha$-important exit with probability at least $1 - p/24T$.
    		Thus, overall, this event occurs with probability at least $1 - p/12T$.
    	\end{enumerate}

    	Since each part fails with probability at most $p/12T$, the overall failure probability is at most $p/3T$, the desired result.
    \end{proof}

  \subsubsection{Phase II Analysis}

    In this section, we provide guarantees on the dataset $\dataset_\rewardfree$ obtained by performing reward-free RL in \algoref{alg:phase-ii}.
    Formally, we have the following high-probability result:

    \begin{proposition}
    	\label{prop:phase-ii-guarantee}
    	For any $\delta > 0$ and failure probability $p$, if \algoref{alg:phase-ii} is run with parameters
    	\begin{equation*}
    	\begin{aligned}
    	N_\Euler^\rewardfree &= O\left(\frac{H^2S^2A}{\delta}\log^3\frac{HSA}{p}\right) \\
    	N_{\rewardfree} &= O\left[\max\left(\frac{C}{\zeta^2}, \frac{1}{\rho\min(\epsilon, \epsilon_0)^2}\right)H^5S^2A\log\frac{A}{p}\right]
    	\end{aligned}.
    	\end{equation*}
    	Then, with probability at least $1 - p/3$:
    	\begin{enumerate}[label=(\alph*)]
    		\item The distribution $\mu$ generating each sample in $\dataset_\rewardfree$ satisfies
    		\[
    		\text{$s \in \statespace$ is $\delta$-significant in $\MDP_1(2H)$} \implies \max_{a, \pi}\frac{\prob{(s, a) \in \tau_\pi}}{\mu(s, a)} \leq 4SAH.
    		\]

    		\item The estimated dynamics model $\hat\transition_0$ satisfies
    		\begin{align*}
    		&\max_{f: \statespace \to [0, H]}\max_{\nu: \statespace \to \actionspace}\expt[(s, a) \sim \mu]{\abs{\left[(\hat\transition - \transition)f\right](s, a)}^2\ind{a = \nu(s)}} \\
    		&\qquad \lesssim \min\left(\frac{\zeta^2}{4\cdot24^2C}, \frac{\rho\min(\epsilon, \epsilon_0)^2}{16}\right)\frac{1}{H^3SA}.
    		\end{align*}
    	\end{enumerate}
    \end{proposition}

    The details of the proof of \propref{prop:phase-ii-guarantee} follow that of \citet{jin2020reward}, which we provide here for completeness.
    First, we adapt the regret bound from \citet{zanette2019tighter} for any MDP and reward function used in \algoref{alg:phase-ii}.

    \begin{lemma}
    	\label{lemma:euler-regret-bound}
    	For any $g \in \statespace$, running $\Euler$ in $\MDP_g$ for $N$ iterations returns $N$ policies $\pi_1, \dots, \pi_N$ satisfying the regret bound
    	\[
    	V_0^\ast(s_0) - \frac{1}{N}\sum_{k = 1}^{N}V_0^{\pi_k}(s_0) \lesssim \sqrt{4V_0^\ast(s_0)\frac{SA}{N}\log\frac{SAHN}{p}} + \frac{S^2AH^2}{N}\log^3\frac{SAHN}{p}
    	\]
    	with probability at least $1 - p$.
    \end{lemma}
    \begin{proof}
    	Observe that
    	\begin{align*}
    	&\frac{1}{NH}\sum_{k = 1}^{N}\expt[\pi_k]{\left(\sum_{h=1}^{H - 1}r(s_h, a_h) - V_0^{\pi_k}(s_0)\right)^2 \suchthat s_0} \\
    	&\qquad \leq \frac{2}{NH}\sum_{k = 1}^{N}\expt[\pi_k]{\left(\sum_{h=1}^{H - 1}r(s_h, a_h)\right)^2 + \left(V_0^{\pi_k}(s_0)\right)^2 \suchthat s_0} \\
    	&\qquad \leq \frac{2}{NH}\sum_{k = 1}^{N}\expt[\pi_k]{\sum_{h=1}^{H - 1}r(s_h, a_h) + V_0^{\pi_k}(s_0) \suchthat s_0} \\
    	&\qquad \leq \frac{4}{H}V_0^\ast(s_0).
    	\end{align*}
    	Therefore, by applying the regret bounds from \citet{zanette2019tighter}, we obtain the regret bound
    	\[
    	V_0^\ast(s_0) - \frac{1}{N}\sum_{k = 1}^{N}V_0^{\pi_k}(s_0) \lesssim \sqrt{4V_0^\ast(s_0)\frac{SA}{N}\log\frac{SAHN}{p}} + \frac{S^2A^2H^2}{N}\log^3\frac{SAHN}{p}
    	\]
    	with probability at least $1 - p$.
    \end{proof}

    With the regret bound above, we now proceed to prove \propref{prop:phase-ii-guarantee}.

    \begin{proof}[Proof of \propref{prop:phase-ii-guarantee}]
    	(a) Fix a $\delta$-significant $g \in \statespace$.
    	Note that for $r_g$, $V_0^\pi(s_0) = P(g \in \tau_\pi)$ for any policy $\pi$.
    	Therefore, via the regret bound from \lemmaref{lemma:euler-regret-bound} and the choice of $N_\Euler^\rewardfree$, we obtain
    	\begin{align*}
    	&\max_\pi P(g \in \tau_\pi) - \frac{1}{N_\Euler^\rewardfree}\sum_{k = 1}^{N_\Euler^\rewardfree}P(g \in \tau_\pi) \leq \frac{1}{2}\max_\pi P(g \in \tau_\pi) \\
    	&\qquad \implies \max_{\pi}\prob{g \in \tau_\pi} \leq \frac{2}{N_\rewardfree^\Euler}\sum_{\pi \in \Phi_g}P(g \in \tau_\pi)
    	\end{align*}
    	with probability at least $1 - p/2S$.
    	Now, since $\pi(\cdot \conditionedon g) \distas \uniform{\actionspace}$, we have that for any $a$,
    	\[
    	\max_{\pi}\prob{(g, a) \in \tau_\pi} \leq \frac{2A}{N_\rewardfree^\Euler}\sum_{\pi \in \Phi_g}P((g, a) \in \tau_\pi).
    	\]
    	Finally, by applying the same argument above across all $\delta$-significant $g \in \statespace$, we have that for any $(g, a)$,
    	\[
    	\max_{\pi}\prob{(g, a) \in \tau_\pi} \leq \sum_{g \in S}\max_{a, \pi}\prob{(g, a) \in \tau_\pi} \leq 2SA\left[\frac{1}{SN_\rewardfree^\Euler}\sum_{\pi \in \Psi}\prob{(g, a) \in \tau_\pi}\right]
    	\]
    	with probability at least $1 - p/2$.
    	To complete the proof of (a), observe that
    	\[
    	\frac{1}{SN_\rewardfree^\Euler}\sum_{\pi \in \Psi}\frac{1}{2H}\prob{(g, a) \in \tau_\pi} \leq \mu(s, a) \implies \max_{s, a, \pi}\frac{\prob{(s, a) \in \tau_\pi}}{\mu(s, a)} \leq 4SAH,
    	\]
    	since conditioned on $(g, a) \in \tau_\pi$, the probability that $(g, a)$ is sampled is at least $1/2H$.

    	(b) The result follows by following the same proof of Lemma C.2 in \citet{jin2020reward}, with failure probability $p/2$.
    	Note that the dynamics are stationary, and thus we do not need to perform a union bound over the time step $h \in [H]$.
    \end{proof}

  \subsubsection{Phase III Analysis}

    Having analyzed the previous two phases, we now show that \algoref{alg:phase-iii} successfully finds all exits during Phase III.
    As part of this, we prove the following guarantee:

    \begin{proposition}
    	\label{prop:phase-iii-guarantee}
    	Assume that \algoref{alg:phase-iii} is at \algostepref{algostep:phase-iii-iteration}, having just arrived at this step for the first time, or after finding a new exit.
    	Let $E = \set{(s, a) \suchthat \textsc{IsExit}[s, a]}$.
    	We assume:
    	\begin{enumerate}[label=(\alph*)]
    		\item $E \subseteq \exit{\statespace}$.

    		\item The high-probability events in \propref{prop:phase-i-guarantee} (for any $t \in [T]$) and \propref{prop:phase-ii-guarantee} (for $\delta \leq \zeta/24CH^2S$) both hold, providing estimators $\hat\transition_0, \hat\transition_1, \dots, \hat\transition_T$.

    		\item For every $(s, a) \in E$ and $t \in [T]$, we have access to an estimator $\hat\transition_t(\cdot \conditionedon s, a)$ for $\transition_t(\cdot \conditionedon s, a)$ satisfying
    		\[
    		\max_{f: \statespace \to [0, H]}\abs{\left[(\hat\transition_t - \transition_t)f\right](s, a)} \leq \min\left(\frac{\zeta}{24H}, \frac{\beta H}{2}\right).
    		\]
    	\end{enumerate}
    	Then, if $E = \exit{\statespace}$, the algorithm terminates after passing through $T$ tasks.
    	Otherwise, if $E \neq \exit{\statespace}$, the following events hold simultaneously with probability at least $1 - p/3K$:
    	\begin{enumerate}[label=(\alph*)]
    		\item For one of the next $T$ tasks that the algorithm inspects, there exists at least one $t \in [T]$ such that
    		\[
    		\abs{\tilde{V}^t - \hat{V}_t} > \frac{2}{3}\zeta.
    		\]

    		\item For the task in (a), running Lines \ref{algostep:exit-detect-and-learn-start}--\ref{algostep:exit-detect-and-learn-end} finds at least one $(s, a) \in \exit{\statespace} \setminus E$ (and only $(s, a)$ pairs in this set), and learns an estimator $\hat\transition_t(\cdot \conditionedon s, a)$ for $\transition_t(\cdot \conditionedon s, a)$ satisfying
    		\[
    		\max_{f: \statespace \to [0, H]}\abs{\left[(\hat\transition_t - \transition_t)f\right](s, a)} \leq \min\left(\frac{\zeta}{24H}, \frac{\beta H}{2}\right).
    		\]
    	\end{enumerate}
    \end{proposition}

    To prove the above result, we will consider the following special set of MDPs:

    \begin{definition}[Imaginable MDPs]
    	Fix a task $\MDP_t$.
    	Furthermore, let $E \subseteq \exit{\statespace}$.
    	For any function $\mathcal{I}: \statespace \times \actionspace \times [H] \to \set{0, \dots, T}$, we can construct an associated MDP $\MDP_\mathcal{I} = (\statespace, \actionspace, \transition_\mathcal{I}, r_t, H)$ via
    	\[
    	\transition^{(h)}_{\mathcal{I}}(\cdot \conditionedon s, a) = \transition_{\mathcal{I}(s, a, h)}(\cdot \conditionedon s, a).
    	\]
    	We define the set of imaginable MDPs to be the set $\imaginedMDPs_t(E)$ to be the set of MDPs generated by any $\mathcal{I}$ satisfying
    	\[
    	\imagined(s, a, h) \in
    	\begin{cases}
    	\set{t} & (s, a) \in E \\
    	\set{0} \union \set{k \suchthat N^k(s, a) \geq N_\thresh} & \text{otherwise}
    	\end{cases}. \qedhere
    	\]
    \end{definition}

    Informally, $\imaginedMDPs_t(E)$ is the set of obtainable MDPs by borrowing dynamics for $(s, a)$ pairs that are not known to be exits.
    This set is of particular interest in our analysis, since BOAT-VI performs a maximization over the MDPs in this set:

    \begin{lemma}[Optimism]
    	\label{lemma:opt-img-vi-optimism}
    	Assume the preconditions of \propref{prop:phase-iii-guarantee}.
    	Over the course of running \algoref{alg:boat-vi}, the algorithm implicitly defines an index function $\mathcal{I}: \statespace \times \actionspace \times [H] \to \set{0, \dots, T}$.
    	This function $\mathcal{I}$ satisfies $\MDP_\mathcal{I} \in \imaginedMDPs_t(E)$, and $\MDP_\mathcal{I}$ is a maximizer of
    	\[
    	\max_{\MDP \in \imaginedMDPs_t(E)}\max_{\pi}\hat{V}_0^{\MDP, \pi}(s_0).
    	\]
    \end{lemma}
    \begin{proof}
    	To see that $\MDP_{\mathcal{I}} \in \imaginedMDPs_t(E)$, note that if $(s, a) \in E$, then $\mathcal{I}_h(s, a) = t$ for any $h \in [H]$.
    	Otherwise, note that although the maximum is over all indices, $\hat\transition_k(s' \conditionedon s, a) = 0$ for any $s'$ if $N_t(s, a) < N_\thresh$.
    	Therefore, since the estimated value function is always positive, the maximum is effectively only over any $k$ with $N_k(s, a) \geq 0$.
    	Thus, $\MDP_{\mathcal{I}} \in \imaginedMDPs_t(E)$.

    	Now, we prove that $\MDP_\mathcal{I} = \MDP$ is a maximizer of the estimated value function, which we prove by induction.
    	Let $\mathcal{I}'$ be another index function satisfying $\MDP' = \MDP_{\mathcal{I}'} \in \imaginedMDPs_t(E)$.
    	Clearly, $\hat{V}_H^{\MDP, \ast}(s) = 0 \leq \hat{V}_H^{\MDP', \ast}(s)$.
    	Then, for any $h \in [H]$ and $(s, a)$,
    	\begin{align*}
    	\hat{Q}_h^{\MDP, \ast}(s, a) &= r(s, a) + \hat\transition_{\mathcal{I}(s, a, h)}\hat{V}_{h + 1}^{\MDP, \ast}(s, a) \\
    	&\geq r(s, a) + \hat\transition_{\mathcal{I}'(s, a, h)}\hat{V}_{h + 1}^{\MDP, \ast}(s, a) \geq r(s, a) + \hat\transition_{\mathcal{I}'(s, a, h)}\hat{V}_{h + 1}^{\MDP', \ast}(s, a) \\
    	&= \hat{Q}_h^{\MDP', \ast}(s, a),
    	\end{align*}
    	where the first inequality follows from the definition of $\mathcal{I}$, and the second follows from the inductive hypothesis.
    	Therefore, for any $s$,
    	\[
    	\hat V^{\MDP, \ast}_{h}(s) = \max_{a}\hat Q_h^{\MDP, \ast}(s, a) \geq \max_{a}\hat Q_h^{\MDP', \ast}(s, a) = \hat{V}_h^{\MDP', \ast}(s)
    	\]
    	Thus, by induction, $\hat{V}_0^{\MDP, \ast}(s_0) \geq \hat{V}_0^{\MDP', \ast}(s_0)$.
    	Since the argument applies for any $\mathcal{I}'$, we have shown the desired optimality result.
    \end{proof}

    Note that $\bar\MDP$ is contained in $\imaginedMDPs_t(E)$ via our assumed preconditions, suggesting that the BOAT-VI should find an MDP with a sufficiently over-optimistic value.
    However, the maximization above makes use of estimated dynamics, and thus we need to prove that every MDP in $\imaginedMDPs_t(E)$ is sufficiently well-estimated.
    We now show that the preconditions of \propref{prop:phase-iii-guarantee} are sufficient for estimation.
    To this end, we recall the performance difference lemma:

    \begin{lemma}[Performance Difference]
    	\label{lemma:performance-difference}
    	Fix two MDPs $\MDP = (\statespace, \actionspace, r, \transition, H)$ and $\MDP' = (\statespace, \actionspace, r, \transition', H)$.
    	Then, for any policy $\pi$,
    	\[
    	V^{\MDP', \pi}_0(s_0) - V^{\MDP, \pi}_0(s_0) = \expt[\MDP, \pi]{\sum_{h = 0}^{H - 1}[(\transition'_h - \transition_h)\hat V_{h + 1}](s_h, a_h) \suchthat s_0}.
    	\]
    \end{lemma}

    We now present the estimation result:

    \begin{lemma}
    	\label{lemma:uniform-imagined-MDP-error-bound}
    	For any $\pi$ and $t \in [T]$, let $V^{\MDP, \pi}_0(s_0)$ be the value of a policy $\pi$ in $\MDP \in \imaginedMDPs_t(E)$, and $\hat V^{\MDP, \pi}_0(s_0)$ an estimate using available quantities from the preconditions of \propref{prop:phase-iii-guarantee}.
    	Then,
    	\[
    	\sup_{t \in [T]}\sup_{\substack{\pi \\ \MDP \in \mathbb{M}_t(E)}}\abs{\hat{V}_0^{\MDP, \pi}(s_0) - V_0^{\MDP, \pi}(s_0)} < \frac{\zeta}{6}.
    	\]
    \end{lemma}
    \begin{proof}
    	Fix a $t \in [T]$, $\MDP \in \imaginedMDPs_t(E)$ and policy $\pi$, with associated index function $\mathcal{I}$.
    	\lemmaref{lemma:performance-difference} implies that
    	\begin{align*}
    	\abs{\hat{V}_0^{\MDP, \pi}(s_0) - V_0^{\MDP, \pi}(s_0)} &\leq \sum_{h = 0}^{H - 1}\expt[\MDP, \pi]{\abs{\left[(\hat\transition^{(h)} - \transition^{(h)})\hat V_{h + 1}^\pi\right](s_h, a_h)}} \\
    	&\leq \sum_{h = 0}^{H - 1}\sum_{(s, a)}\abs{\left[(\hat\transition^{(h)} - \transition^{(h)})\hat V_{h + 1}^\pi\right](s, a)}P_h^{\MDP, \pi}(s, a).
    	\end{align*}
    	We now define the following sets:
    	\begin{align*}
    	U_\delta &= \set{\text{$(s, a)$ is $\delta$-insignificant for $\MDP_1$}} \\
    	B_h &= \set{(s, a) \suchthat \mathcal{I}_h(s, a) \neq 0} \setminus (E \union U_\delta) \\
    	R_h &= \set{(s, a) \suchthat \mathcal{I}_h(s, a) = 0} \setminus U_\delta.
    	\end{align*}
    	Note that $R_h$ is estimated via reference dynamics from Phase II, while $B_h$ is estimated using task-specific dynamics from Phase I.
    	Then, for a fixed $h$, we can decompose the inner sum above as
    	\begin{align*}
    	&\sum_{(s, a)}\abs{\left[(\hat\transition^{(h)} - \transition^{(h)})\hat V_{h + 1}^\pi\right](s, a)}P_h^{\MDP, \pi}(s, a) \\
    	&\qquad \leq \underbrace{\sum_{\substack{(s, a) \in E}}\abs{\left[(\hat\transition^{(h)} - \transition^{(h)})\hat V_{h + 1}^\pi\right](s, a)}P_h^{\MDP, \pi}(s, a)}_{\asdef \text{(I)}} \\
    	&\qquad\qquad + \underbrace{\sum_{(s, a) \in R_h}\abs{\left[(\hat\transition^{(h)} - \transition^{(h)})\hat V_{h + 1}^\pi\right](s, a)}P_h^{\MDP, \pi}(s, a)}_{\asdef \text{(II)}} \\
    	&\qquad\qquad + \underbrace{\sum_{(s, a) \in B_h}\abs{\left[(\hat\transition^{(h)} - \transition^{(h)})\hat V_{h + 1}^\pi\right](s, a)}P_h^{\MDP, \pi}(s, a)}_{\asdef \text{(III)}} \\
    	&\qquad\qquad + \underbrace{\sum_{(s, a) \in U_\delta}\abs{\left[(\hat\transition^{(h)} - \transition^{(h)})\hat V_{h + 1}^\pi\right](s, a)}P_h^{\MDP, \pi}(s, a)}_{\asdef \text{(IV)}}
    	\end{align*}
    	Note the inequality since $E \intersect U_\delta$ is not necessarily disjoint.
    	We now bound the four terms above separately.

    	\textbf{Bounding (I): Dynamics Error from Known Exits.}
    	We first bound (I), which we note derives from errors in estimating the dynamics of known exits.
    	Recall that by precondition (c) in \propref{prop:phase-iii-guarantee},
    	\[
    	\sup_{f: \statespace \to [0, H]}\abs{\left[(\hat\transition_t - \transition_t)f\right](s, a)} \leq \frac{\zeta}{24H}.
    	\]
    	Therefore,
    	\begin{align*}
    	\text{(I)} &= \sum_{\substack{(s, a) \in E}}\abs{\left[(\hat\transition^{(h)} - \transition^{(h)})\hat V_{h + 1}^\pi\right](s, a)}P_h^{\MDP, \pi}(s, a) \\
    	&= \sum_{\substack{(s, a) \in E}}\abs{\left[(\hat\transition_t - \transition_t)\hat V_{h + 1}^\pi\right](s, a)}P_h^{\MDP, \pi}(s, a) \leq \frac{\zeta}{24H}\sum_{(s, a) \in E}P_h^{\MDP, \pi}(s, a) \\
    	&\leq \frac{\zeta}{24H}.
    	\end{align*}

    	\textbf{Bounding (II): Reference Dynamics Error.}
    	Note that within $R_h$, $\transition_h = \transition_0$, which we estimate via $\hat\transition_0$.
    	Therefore, we bound the error resulting from using $\dataset_{\rewardfree}$ to estimate $\transition_0$.
    	This part of the proof follows that of \citet{jin2020reward}.
    	First, by Cauchy-Schwarz,
    	\begin{align*}
      	\text{(II)} &= \sum_{(s, a) \in R_h}\abs{\left[(\hat\transition^{(h)} - \transition^{(h)})\hat V_{h + 1}^\pi\right](s, a)}P_h^{\MDP, \pi}(s, a) \\
          &= \sum_{(s, a) \in R_h}\abs{\left[(\hat\transition_0 - \transition_0)\hat V_{h + 1}^\pi\right](s, a)}P_h^{\MDP, \pi}(s, a) \\
        	&\leq \left[\sum_{(s, a) \in R_h}\abs{\left[(\hat\transition_0 - \transition_0)\hat V_{h + 1}^\pi\right](s, a)}^2P_h^{\MDP, \pi}(s, a)\right]^{1/2}.
    	\end{align*}
    	Observe that $\hat V_{h + 1}^\pi$ only depends on $\pi$ through timesteps $h + 1, \dots, H - 1$.
    	Therefore,
      \begin{align*}
      	&\sum_{(s, a) \in R_h}\abs{\left[(\hat\transition_0 - \transition_0)\hat V_{h + 1}^\pi\right](s, a)}^2P_h^{\MDP, \pi}(s, a) \\
          &\qquad \leq \max_{\nu: \statespace \to \actionspace}\sum_{(s, a) \in R_h}\abs{\left[(\hat\transition_0 - \transition_0)\hat V_{h + 1}^\pi\right](s, a)}^2P_h^{\MDP, \pi}(s)\ind{\nu(s) = a}.
      \end{align*}
    	By applying \assumpref{assumption:non-limiting-exits},
    	\begin{align*}
    	P_h^{\MDP, \pi}(s) &\leq P^{\MDP}(s \in \tau_{\pi}) \leq \max_{\pi}P^\MDP(s \in \tau_\pi) \leq C\max_{\pi}P^{\MDP_1}(s \in \tau_\pi) \\
    	&\leq C\max_{\pi}P^{\MDP_1(2H)}(s \in \tau_\pi) \leq 4CHSA\mu(s, a),
    	\end{align*}
    	where we have applied \assumpref{assumption:exit-coverage} to move from $\MDP$ to $\MDP_1$.
    	Substituting into the earlier expression,
    	\begin{align*}
    	&\max_{\nu: \statespace \to \actionspace}\sum_{(s, a) \in R_h}\abs{\left[(\hat\transition_0 - \transition_0)\hat V_{h + 1}^\pi\right](s, a)}^2P_h^{\MDP, \pi}(s)\ind{\nu(s) = a} \\
    	&\qquad \leq 4CHSA\max_{\nu: \statespace \to \actionspace}\sum_{(s, a) \in R_h}\abs{\left[(\hat\transition_0 - \transition_0)\hat V_{h + 1}^\pi\right](s, a)}^2\ind{a = \nu(s)}\mu(s, a) \\
    	&\qquad \leq 4CHSA\max_{\nu: \statespace \to \actionspace}\sum_{s, a}\abs{\left[(\hat\transition_0 - \transition_0)\hat V_{h + 1}^\pi\right](s, a)}^2\ind{a = \nu(s)}\mu(s, a) \\
    	&\qquad = 4CHSA\max_{\nu: \statespace \to \actionspace}\expt[(s, a) \sim \mu]{\abs{\left[(\hat\transition_0 - \transition_0)\hat V_{h + 1}^\pi\right](s, a)}^2\ind{a = \nu(s)}}.
    	\end{align*}
    	Thus, by applying the bound on the right-hand side provided by \propref{prop:phase-ii-guarantee},
    	\[
    	\text{(II)} = \sum_{(s, a) \in R_h}\abs{\left[(\hat\transition^{(h)} - \transition^{(h)})\hat V_{h + 1}^\pi\right](s, a)}P_h^{\MDP, \pi}(s, a) \leq \frac{\zeta}{24H}.
    	\]

    	\textbf{Bounding (III): Error from Task-Specific Dynamics.}
    	Recall that on $B_h$, $\transition_h = \transition_k$ for some $k \neq 0$.
    	Thus, (III) is the error resulting from dynamics estimation in \algoref{alg:phase-i}.
    	By following the same argument as that used to bound (I) and applying \propref{prop:phase-i-guarantee}, we find that
    	\[
    	\sum_{(s, a) \in B_h}\abs{\left[(\hat\transition^{(h)} - \transition^{(h)})\hat V_{h + 1}^\pi\right](s, a)}P_h^{\MDP, \pi}(s, a) \leq \frac{\zeta}{24H}.
    	\]

    	\textbf{Bounding (IV): Error from $\delta$-Insignificance.}
    	The remaining set of $(s, a)$ pairs are those such that $s$ is $\delta$-insignificant in $\MDP_1$.
    	Note that
    	\begin{align*}
    	\text{(IV)} &= \sum_{(s, a) \in U_\delta}\abs{\left[(\hat\transition^{(h)} - \transition^{(h)})\hat V_{h + 1}^\pi\right](s, a)}P_h^{\MDP, \pi}(s, a) \leq H\sum_{(s, a) \in U_\delta}P_h^{\MDP, \pi}(s, a) \\
    	&= H\sum_{s \in U_\delta}P_h^{\MDP, \pi}(s).
    	\end{align*}
    	As a result,
    	\begin{align*}
    	P_h^{\MDP, \pi}(s) &\leq P^{\MDP}(s \in \tau_\pi) \leq \max_\pi P^\MDP(s \in \tau_\pi) \leq C\max_{\pi}P^{\MDP_1}(s \in \tau_\pi) \\
    	&\leq C\delta.
    	\end{align*}
    	By setting $\delta = \zeta/24CH^2S$ when performing reward-free RL in Phase II, we thus find that
    	\[
    	\sum_{(s, a) \in U_\delta}\abs{\left[(\hat\transition^{(h)} - \transition^{(h)})\hat V_{h + 1}^\pi\right](s, a)}P_h^{\MDP, \pi}(s, a) \leq \frac{\zeta}{24H}.
    	\]

    	\textbf{Concluding}. By combining the bounds on (I) through (IV) and summing across $h = 0, \dots, H - 1$, we find that
    	\[
    	\abs{\hat{V_0}^{\MDP, \pi}(s_0) - V_0^{\MDP, \pi}(s_0)} \leq \frac{\zeta}{6}.
    	\]
    	Note that this argument simultaneously applies to any such $\MDP$; therefore, the desired conclusion follows.
    \end{proof}

    The prior estimation result, together with \assumpref{assumption:exit-coverage}, suggests that \textsc{BOAT-VI} should find an MDP that sufficiently overestimates the value of the task so long as not all exits have been found.
    This ensures that the exit-finding routine is triggered.
    Formally,

    \begin{lemma}
    	Assume the preconditions of \propref{prop:phase-iii-guarantee}, and that $E \neq \exit\statespace$.
    	Additionally, let $t \in [T]$ be the task with a $\zeta$-overoptimistic value when borrowing exits $\exit\statespace \setminus E$.
    	Finally, let $\tilde{V}^t$ be the value function returned by \algoref{alg:boat-vi} on $\MDP_t$.
    	Then,
    	\[
    	\tilde{V}^{t}_0(s_0) - \hat{V}_t > \frac{2}{3}\zeta.
    	\]
    \end{lemma}
    \begin{proof}
    	Throughout this proof, we omit the timestep $0$ and the initial state $s_0$ for brevity.
    	Define $\MDP^\ast$ and $\pi^\ast$ to be the maximizers of
    	\[
    	\max_{\MDP \in \mathbb{M}_t(E)}\max_{\pi}V^{\MDP, \pi}.
    	\]
    	Furthermore, let $\bar\MDP$ be the imagined MDP guaranteed by \assumpref{assumption:exit-coverage} on top of $\MDP_t$, such that $V^{\bar\MDP, \ast} > V^{\MDP_t, \ast} + \zeta$.
    	Then, we have that
    	\begin{align*}
      	\hat{V}^{\MDP, \pi} - V^{\MDP_t, \ast} &= \underbrace{(\hat V^{\MDP, \pi} - \hat V^{\MDP^\ast, \pi^\ast})}_{\geq 0} + \underbrace{(V^{\hat\MDP^\ast, \pi^\ast} - V^{\MDP^\ast, \pi^\ast})}_{> -\zeta/6} \\
          &\qquad + \underbrace{(V^{\MDP^\ast, \pi^\ast} - V^{\bar\MDP, \bar\pi})}_{\geq 0} + \underbrace{(V^{\bar\MDP, \bar\pi} - V^{\MDP_t, \ast})}_{> \zeta} \\
          &> \frac{5}{6}\zeta.
    	\end{align*}
    	Furthermore,
    	\begin{align*}
      	V^{\MDP_t, \ast} - \hat{V}_t = \underbrace{\left(V^{\MDP_t, \ast} - \frac{1}{N_\UCBVI}\sum_{k = 1}^{N_\UCBVI}V^{\MDP_t, \pi^{(t)}_k}\right)}_{\geq 0} + \underbrace{\left(\frac{1}{N_\UCBVI}\sum_{k = 1}^{N_\UCBVI}V^{\MDP_t, \pi^{(t)}_k} - \hat{V}_t\right)}_{\geq -\zeta/6},
    	\end{align*}
    	where the first follows from optimality, while the second follows from \propref{prop:phase-ii-guarantee}.
    	Thus, putting the two inequalities together,
    	\[
    	\tilde{V}^{t} - \hat{V}_t > \frac{2}{3}\zeta. \qedhere
    	\]
    \end{proof}

    While the prior algorithm ensures that at least one of the tasks will trigger the exit condition, the actual task that triggers the condition may not be the same one invoked in the proof above.
    Nevertheless, we can prove that the trigger condition ensures that the algorithm will find a new exit.

    \begin{lemma}
    	\label{lemma:exit-finding-condition-implies-visitation}
    	Assume the preconditions of \propref{prop:phase-iii-guarantee}, and that $E \neq \exit\statespace$.
    	Let $t \in [T]$ be a task such that the value estimate $\tilde{V}^t$ returned by \algoref{alg:boat-vi} satisfies $\tilde{V}^t_0(s_0) - \hat{V}_t > (2/3)\zeta$, and let $\pi$ be the optimal policy for $\tilde{V}$.
    	Then, there exists $(s, a) \in \exit\statespace \setminus E$ such that for some $t' \neq t$,
    	\begin{enumerate}[label=\textup{(\alph*)}]
    		\item $N_{t'}(s, a) \geq N_\thresh^\taskspecific$ and $\transition_t(\cdot \conditionedon s, a) \neq \transition_{t'}(\cdot \conditionedon s, a)$.

    		\item $P^{M_t}((s, a) \in \tau_\pi) > \zeta/6KH$.
    	\end{enumerate}
    \end{lemma}
    \begin{proof}
    	Let $\MDP$ be the implicit MDP defined by \algoref{alg:boat-vi} in the process of computing $\tilde{V}^t$.
    	We will prove a value gap between $\MDP$ and $\MDP_t$, which implies that $\pi$ must visit state-action pairs with imagined dynamics.

    	Note that
    	\begin{align*}
      	V_0^{\MDP, \pi}(s_0) - V_0^{\MDP_t, \ast}(s_0) &= \underbrace{\left[V_0^{\MDP, \pi}(s_0) - \tilde{V}^t_0(s_0)\right]}_{\geq -\zeta/6} + \underbrace{\left[\tilde{V}_0^t(s_0) - \hat{V}_t\right]}_{\geq (2/3)\zeta} \\
        &\qquad + \underbrace{\left[\hat{V}_t - \frac{1}{N_\UCBVI}\sum_{k = 1}^{N_\UCBVI}V^{\MDP_t, \pi_k^{(t)}}_0(s_0)\right]}_{\geq -\zeta/6} \\
      	&\qquad + \underbrace{\left[\frac{1}{N_\UCBVI}\sum_{k = 1}^{N_\UCBVI}V^{\MDP_t, \pi_k^{(t)}}_0(s_0) - V^{\MDP_t, \ast}_0(s_0)\right]}_{\geq -\zeta/6},
    	\end{align*}
    	where the first inequality comes from \lemmaref{lemma:uniform-imagined-MDP-error-bound} and the last two inequalities come from \propref{prop:phase-i-guarantee}.
    	Thus, $V_0^{\MDP, \pi}(s_0) - V_0^{\MDP_t, \ast}(s_0) \geq \zeta/6$.

    	We now leverage this value gap to show that $\pi$ must use some exit $(s, a)$ whose dynamics in $\MDP$ have been modified from $\MDP_t$ with some probability.
    	Formally, define the set $\Delta = \set{(s, a, h) \suchthat \transition_t(\cdot \conditionedon s, a) \neq \transition_h^{\MDP}(\cdot \conditionedon s, a)}$.
    	By construction, $\Delta \subseteq \exit\statespace \times [H]$, and for any $(s, a, h) \in \Delta$, there exists $t'$ such that $N_{t'}(s, a) \geq N_\thresh^\taskspecific$ and $\transition_{t'}(\cdot \conditionedon s, a) = \transition_h^\MDP(\cdot \conditionedon s, a)$.
    	Furthermore, $\Delta$ must be non-empty, as otherwise, $\transition_t = \transition^M_h$ for all $h$, and thus $V_0^{\MDP, \pi}(s_0) \leq V_0^{\MDP_t, \ast}(s_0)$, a contradiction.
    	Therefore, by applying \lemmaref{lemma:value-gap-implies-visitation}, we find that
    	\[
    	\frac{\zeta}{6H} < P^{\MDP_t}(\tau_\pi \intersect \Delta \neq \emptyset) \leq \sum_{\set{(s, a) \suchthat (s, a, h) \in \Delta}}P^{\MDP_t}((s, a) \in \tau_\pi),
    	\]
    	which implies the desired result, as $\set{(s, a) \suchthat (s, a, h) \in \Delta}$ has at most $K$ elements.
    \end{proof}

    Because of \lemmaref{lemma:exit-finding-condition-implies-visitation}, we simply need to run $\pi$ enough times and threshold at the number of samples needed to reliably determine which $(s, a)$ pairs have an $O(\beta)$ change in TV distance between tasks.

    \begin{lemma}
    	\label{lemma:beta-level-exit-detection}
    	We work in the setting of \lemmaref{lemma:exit-finding-condition-implies-visitation}.
    	Set
    	\[
    	N_\thresh^{\exitdetection} = \Omega\left(\frac{S}{\beta^2}\log\frac{SAHN_\exitdetection}{p}\right) \quad \text{and} \quad N_\exitdetection = \Omega\left(\frac{HK}{\zeta}N_\thresh^\exitdetection + \frac{H^2K^2}{\zeta^2}\log\frac{K}{p}\right).
    	\]
    	Then, if we execute $\pi$ within $\MDP_t$ $N_\exitdetection$ and let $N(s, a)$ be the number of times that $(s, a)$ is played in this process, then with probability at least $1 - p/6K$, the following hold:
    	\begin{enumerate}[label=(\alph*)]
    		\item For the $(s, a)$ pair and task $t'$ in \lemmaref{lemma:exit-finding-condition-implies-visitation}, $N(s, a) \geq N_\thresh^{\exitdetection}$ and
    		\[
    		\TV{\hat\transition(\cdot \conditionedon s, a) - \hat\transition_{t'}(\cdot \conditionedon s, a)} > \frac{\beta}{2}.
    		\]

    		\item For any $(s, a) \not\in \exit\statespace$ with $N(s, a) \geq N_\thresh^{\exitdetection}$ and $t'$ with $N_{t'}(s, a) \geq 0$,
    		\[
    		\TV{\hat\transition_t(\cdot \conditionedon s, a) - \hat\transition_{t'}(\cdot \conditionedon s, a)} \leq \frac{\beta}{2}.
    		\]
    	\end{enumerate}
    \end{lemma}
    \begin{proof}
    	By the choice of $N_\exitdetection$ and the lower bound $P^{\MDP_t}((s, a) \in \tau_\pi) > \zeta/6KH$ from \lemmaref{lemma:exit-finding-condition-implies-visitation}, we guarantee that $N(s, a) \geq N_\thresh^\exitdetection$ with probability at least $1 - p/12K$.
    	Furthermore, due to the choice of $N_\thresh^\exitdetection$, with probability at least $1 - p/12K$, we have that for any $(s, a)$ with $N(s, a) \geq N_\thresh^\exitdetection$,
    	\[
    	\TV{\hat\transition(\cdot \conditionedon s, a) - \transition_t(\cdot \conditionedon s, a)} < \frac{\beta}{4},
    	\]
    	by applying \lemmaref{lemma:dynamics-estimation-error-bound}.
    	We condition on these two events simultaneously for the rest of the proof, which occurs with probability at least $1 - p/6K$.

    	We now prove each part separately.
    	For brevity, we omit $(s, a)$ wherever it is understood.

    	\begin{enumerate}[label=(\alph*)]
    		\item By applying the triangle inequality,
    		\begin{align*}
    		\TV{\transition_t - \transition_{t'}} &\leq \TV{\transition_t - \hat\transition} + \TV{\hat\transition - \hat\transition_{t'}} + \TV{\hat\transition_{t'} - \transition_{t'}} \leq \TV{\hat\transition - \hat\transition_{t'}} + \frac{\beta}{2}.
    		\end{align*}
    		Therefore, by lower bounding the left-hand side using $\beta$-dynamics separation in \assumpref{assumption:dynamics-separation},
    		we find that
    		\[
    		\TV{\hat\transition(\cdot \conditionedon s, a) - \hat\transition_{t'}(\cdot \conditionedon s, a)} > \frac{\beta}{2}.
    		\]

    		\item The triangle inequality implies that
    		\begin{align*}
    		\TV{\hat\transition - \hat\transition_{t'}} \leq \underbrace{\TV{\hat\transition - \transition_t}}_{\leq \beta/4} + \underbrace{\TV{\transition_t - \transition_{t'}}}_{= 0} + \underbrace{\TV{\transition_{t'} - \hat\transition_{t'}}}_{\leq \beta/4} \leq \frac{\beta}{2},
    		\end{align*}
    		where the bound on the first term is provided by \propref{prop:phase-i-guarantee}. \qedhere
    	\end{enumerate}
    \end{proof}

    The prior result demonstrates that if the exit-finding condition is detected at any point by the algorithm, then the algorithm finds a previously undiscovered exit in $\exit{\statespace}$.
    At this point, all that remains is to ensure that the algorithm sufficiently learns the dynamics of the newly-found exit in all of the meta-training tasks.

    \begin{lemma}
    	Fix an $(s, a) \in \exit\statespace$, which was found via exit detection, and let
    	\[
    	N_\thresh^\exitlearning = \Omega\left[L\max\left(\frac{H^4}{\zeta^2}, \frac{1}{\beta^2}\right)\log\frac{SAHN_\exitlearning T}{p}\right]
    	\]
    	Assume we run the exit-learning subroutine with
    	\[
    	N_{\Euler}^{\exitlearning} = \Omega\left[\frac{CS^2AH^3}{\alpha}\log^3\left(\frac{SAHT}{p}\right)\right] \quad \text{and} \quad N_{\mathrm{EL}} = \Omega\left(\frac{CH}{\alpha}N_\thresh^\exitlearning + \frac{C^2H^2}{\alpha^2}\log\frac{SAT}{p}\right).
    	\]
    	in each of the tasks.
    	Then, with probability at least $1 - p/6K$,
    	\[
    	\max_{f: \statespace \to [0, H]}\abs{\left[(\hat\transition_t - \transition_t)f\right](s, a)} \leq \min\left(\frac{\zeta}{24H}, \frac{\beta H}{2}\right)
    	\]
    	for every $t \in [T]$.
    \end{lemma}
    \begin{proof}
    	Fix a task $t \in [T]$.
    	Note that \assumpref{assumption:exit-coverage} implies that $(s, a)$ is $(\alpha/H)$-significant for some task.
    	Then, $(s, a)$ must be $(\alpha/CH)$-significant for all of the other tasks by \assumpref{assumption:non-limiting-exits}.

    	By applying \lemmaref{lemma:euler-regret-bound}, the set of policies found by the exit-learning subroutine for every task $t \in [T]$ satisfies
    	\begin{align*}
    	&\max_{\pi}\prob{(s, a) \in \tau_\pi} - \frac{1}{N_\Euler^\exitlearning}\sum_{k = 1}^{N_\Euler^\exitlearning}P((s, a) \in \tau_{\pi_k}) \\
    	&\qquad \lesssim \sqrt{\max_{\pi}\prob{(s, a) \in \tau_\pi}\frac{SA}{N_\Euler^\exitlearning}\log\frac{SAHTN_\Euler^\exitlearning}{p}} + \frac{S^2AH^2}{N_\Euler^\exitlearning}\log^3\frac{SAHTN_\Euler^\exitlearning}{p}
    	\end{align*}
    	with probability at least $1 - p/18TK$.
    	By setting
    	\[
    	N_\Euler^\exitlearning \gtrsim \frac{CS^2AH^3}{\alpha}\log^3\left(\frac{HSAT}{p}\right),
    	\]
    	we thus find that for any $t \in [T]$,
    	\[
    	\frac{\alpha}{2CH} \leq \frac{1}{2}\max_{\pi}P((s, a) \in \tau_\pi) \leq \frac{1}{N_\Euler^\exitlearning}\sum_{\pi \in \Phi_t(s, a)}\prob{(s, a) \in \tau_\pi}.
    	\]
    	Note that the right-hand side is exactly the probability that the trajectory of a randomly chosen policy in $\Phi_t(s, a)$ contains $(s, a)$ in the trajectory.
    	Therefore, by applying \lemmaref{lemma:sampling-to-meet-threshold}, playing
    	\[
    	N_{\mathrm{EL}} = \Omega\left(\frac{CH^2N_\thresh^\exitlearning}{\alpha} + \frac{C^2H^4}{\alpha^2}\log\frac{SAT}{p}\right)
    	\]
    	is sufficient to guarantee that we obtain at least $N_\thresh^\exitlearning$ samples from $\transition_t(\cdot \conditionedon s, a)$ with probability at least $1 - p/18TK$.
    	Since $(s, a) \in \exit\statespace$, $\transition_t(\cdot \conditionedon s, a)$ (and by extension, $\hat\transition_t(\cdot \conditionedon s, a)$) is supported on $\entry\statespace$.
    	Therefore, we can modify the proof in \lemmaref{lemma:dynamics-estimation-error-bound} so that with probability at least $1 - p/18TK$ we get the bound
    	\begin{align*}
    	\max_{f: \statespace \to [0, H]}\abs{\left[(\hat\transition_t - \transition_t)f\right](s, a)} &\leq \max_{f: \entry\statespace \to [0, H]}\abs{\left[(\hat\transition_t - \transition_t)f\right](s, a)} \\
    	&\leq \min\left(\frac{\zeta}{24H}, \frac{\beta H}{2}\right)
    	\end{align*}
    	with $N_\thresh^\exitlearning$ depending linearly on $L$ instead of $S$.

    	Note that by performing a union-bound, all events occur with probability at least $1 - p/6TK$.
    	Performing a second union-bound over all of the available tasks results in the desired failure probability.
    \end{proof}

    At this point, we have effectively proven the second half of our Phase III guarantee.
    All that remains is to prove that if $\set{(s, a) \suchthat \textsc{IsExit}[s, a]} = \exit\statespace$, then the algorithm terminates without triggering the exit-finding condition.

    \begin{lemma}
    	Assume that $E = \set{(s, a) \suchthat \textsc{IsExit}[s, a]} = \exit\statespace$.
    	Then, under the preconditions of \propref{prop:phase-iii-guarantee}, every task satisfies
    	\[
    	\abs{\tilde{V}^t(s_0) - \hat{V}_t} \leq \frac{2}{3}\zeta.
    	\]
    \end{lemma}
    \begin{proof}
    	Fix a task $t \in [T]$.
    	Once $E = \exit\statespace$, then $\imaginedMDPs_t(E) = \set{\MDP_t}$, since the only $(s, a)$-dynamics that can be substituted from other tasks are those of non-exits, which do not change between tasks.
    	Therefore, by \lemmaref{lemma:opt-img-vi-optimism}, $\tilde{V}^t(s_0) = \hat{V}_0^{\MDP, \ast}(s_0)$.
    	Finally, by applying \lemmaref{lemma:uniform-imagined-MDP-error-bound}, we thus find that
    	\[
    	\tilde{V}^t(s_0) - \hat{V}_t = \left[\hat{V}_0^{\MDP, \ast}(s_0) - V_0^{\MDP, \ast}(s_0)\right] + \left[V_0^{\MDP, \ast}(s_0) - \hat{V}_t\right] < \frac{\zeta}{3}.
    	\]
    	The desired result follows since the argument holds for any task $t$.
    \end{proof}

  \subsubsection{Proof of \thmref{thm:provable-exit-detection}}

    In this section, we compile the guarantees provided by each of the three phases into a proof of \thmref{thm:provable-exit-detection}.

    \begin{proof}[Proof of \thmref{thm:provable-exit-detection}]
    	We condition on the following high-probability events:
    	\begin{enumerate}[label=(\alph*)]
    		\item \propref{prop:phase-i-guarantee} guarantees for all $\MDP_t$ with $t \in [T]$.

    		\item \propref{prop:phase-ii-guarantee}.
    	\end{enumerate}
    	Via a union-bound, this holds with probability at least $1 - (2/3)p$.

    	To prove the theorem, we provide an induction-based analysis of Phase III.
    	In particular, we will show that while $E = \set{(s, a) \suchthat \textsc{IsExit}[s, a]} \subsetneq \exit\statespace$, Phase III will add at least one state-action pair to $E$ that belongs to $\exit\statespace \setminus E$.

    	Formally, let $F_k$ denote the internal state of the algorithm after it has added $k$ state-action pairs.
    	Note that with $k = 0$, $\set{(s, a) \suchthat \textsc{IsExit}[s, a]}$ in $F_k$ is empty.
    	Thus, $F_k$ satisfies the preconditions of \propref{prop:phase-iii-guarantee}, which in turn implies that the algorithm adds a new state-action pair in $\exit\statespace$ and sufficiently learns its dynamics for all tasks with probability at least $1 - p/3K$.
    	In short, the internal state of the algorithm at time $F_1$ also satisfies the preconditions of \propref{prop:phase-iii-guarantee} with probability at least $1 - p/3K$.
    	More generally, \propref{prop:phase-iii-guarantee} ensures that if $F_k$ satisfies the preconditions of \propref{prop:phase-iii-guarantee}, then so does $F_{k + 1}$.
    	Therefore, with probability at least $1 - p/3$, $F_K$ satisfies the preconditions of \propref{prop:phase-iii-guarantee}, which necessarily implies that $\set{(s, a) \suchthat \textsc{IsExit}[s, a]} = \exit\statespace$ in $F_K$, and thus the algorithm exits as desired.
    	By performing a union bound, all this occurs with probability at least $1 - p$.
    \end{proof}

\subsection{Proving the Meta-Training Guarantee}

  Having demonstrated that $\exit\statespace$ can be successfully recovered by interacting with the environment, we now show that the data can also be used to determine exit reachability and implement the hierarchy oracle.

  We formally state our main result here:

  \begin{theorem}
  	\label{thm:meta-train-guarantee-formal}
  	Assume that $\MDP_1, \dots, \MDP_T$ have a latent hierarchy with respect to $(\set{Z_c}, \entry{\cdot}, \exit{\cdot})$, and assume that these tasks satisfy the $(\alpha, \zeta)$-coverage condition in \assumpref{assumption:exit-coverage}.
  	Furthermore, we assume the additional assumptions in \secref{sec:other-src-assumptions}.
  	Then, by running the algorithm in \secref{sec:src-alg-formal} with the parameters in \tableref{table:src-param-choices}, with probability at least $1 - p$, the collected data can be used to implement the following:
  	\begin{enumerate}[label=(\alph*)]
  		\item An $\epsilon$-suboptimal hierarchy oracle.

  		\item A function $\avexit{s}: \entry\statespace \to \mathcal{P}(\exit\statespace)$ such that, given $s \in \entry{\cluster_s}$, returns $\exit{\cluster_s}$.
  	\end{enumerate}
  	The algorithm achieves both of these with query complexity
  	\begin{align*}
  	\tilde{O}&\left[\frac{S^4A}{\min(\rho\min(\epsilon, \epsilon_0), \zeta/C)} + \frac{S^2A}{\min(\rho\min(\epsilon, \epsilon_0)^2, \zeta^2/C)}\right. \\
  	&\qquad + \left.T\left(\frac{SA}{\min(\alpha, \zeta)^2} + \frac{KS}{\zeta\beta^2} + \frac{K^2}{\zeta^2} + \frac{CKS^2A}{\alpha} + \frac{CKL}{\alpha\max(\zeta, \beta)^2}\right)\right]\mathrm{poly}(H).
  	\end{align*}
  \end{theorem}

  \subsubsection{Implementing the Hierarchy Oracle}

    We first show that we can implement the hierarchy oracle in this section.
    In particular, we have the following result:

    \begin{proposition}
    	\label{prop:solving-disconnected-exits}
    	Let $\MDP$ be the MDP corresponding to the index $(x, f, r, \tilde{H})$ as described in \defref{definition:hierarchy-oracle}.
    	Then, given $(x, f, r, \tilde{H})$, we can form the following estimator for $\transition_f$:
    	\[
    	\hat\transition_f(\cdot \conditionedon s, a) =
    	\begin{cases}
    	\delta(f(s, a)) & (s, a) \in \exit\statespace \\
    	\delta(s) & \text{$s = \terminate_S$ or $s = \terminate_F$} \\
    	\hat\transition_0(\cdot \conditionedon s, a) & \text{otherwise}
    	\end{cases},
    	\]
    	where $\hat\transition_0$ is the estimator obtained from Phase II in \secref{sec:src-alg-formal}.
    	Assuming that the high-probability event in \propref{prop:phase-ii-guarantee} holds for $\delta \leq \rho\epsilon/2SH^2$, value iteration using $\hat\transition_f$ returns a policy $\pi$ such that $V^{\MDP, \ast}_0(x) - V^{\MDP, \pi}_0(x_0) \leq \epsilon$.
    \end{proposition}

    Throughout the rest of this section, we fix the tuple $(x, f, r, \tilde{H})$ and the corresponding MDP $\MDP$.
    Furthermore, we write $\cluster$ for the cluster containing $x$.

    To prove \propref{prop:solving-disconnected-exits}, we will show that $\MDP$ can be sufficiently simulated so that the value of any policy can be reasonably estimated.
    Given this simulation result, we can then show that value iteration finds the desired policy.
    This simulation result depends on the following intermediate result, which provides insight as to why Phase II data is sufficient:

    \begin{lemma}
    	\label{lemma:Z-significance-to-2H-significance}
    	For any $s^\ast \in \cluster$,
    	\[
    	\left[\max_{\pi}P^{\MDP_1}(x \in \tau_\pi)\right]\left[\max_\pi P^{\MDP}(s^\ast \in \tau_\pi)\right] \leq \max_\pi P^{\MDP_1(2H)}(s^\ast \in \tau_\pi)
    	\]
    \end{lemma}
    \begin{proof}
    	First, we note that there exists an MDP such that $P^{\MDP_1(2H)}(s^\ast \in \tau_\pi)$ is the corresponding value function.
    	In particular, modifying $\MDP_1(2H)$ so that any action from $s^\ast$ leads to a terminal state and defining $r(s, a) = \ind{s = s^\ast}$ results in such an MDP.

    	Now, let $\pi_x$ and $\pi_{s^\ast}$ be the policies achieving
    	\[
    	\max_\pi P^{\MDP_1}(x \in \tau_\pi) \quad \text{and} \quad \max_{\pi}P^{\MDP}(s^\ast \in \tau_\pi),
    	\]
    	respectively.
    	Consider the concatenation of $\pi_x$ and $\pi_{s^\ast}$ into a history-dependent policy that runs $\pi_x$ until the agent reaches $s$, and switches to $\pi_{s^\ast}$ thereafter.
    	This policy reaches $s$ with probability at least
    	\[
    	\left[\max_\pi P^{\MDP_1}(x \in \tau_\pi)\right]\left[\max_{\pi}P^{\MDP}(s^\ast \in \tau_\pi)\right].
    	\]
    	within the modified MDP described above.
    	Since the optimal value among all policies is achieved by a history-independent policy, we obtain the desired inequality.
    \end{proof}

    Informally, the prior result states that if $x$ is reachable within horizon $H$, then any state reachable from $x$ within $\cluster$ is also reachable in $\MDP_1$ within a $2H$ horizon.
    Therefore, performing reward-free RL with horizon $2H$ during Phase II provides coverage over all clusters.
    Now, we prove the simulation result.

    \begin{lemma}
    	\label{lemma:cluster-mdp-value-estimation-error}
    	Assume that the Phase II guarantee in \propref{prop:phase-ii-guarantee} is instantiated for $\delta \leq \rho\epsilon/4SH^2$.
    	Then, if $V^\pi$ is the value of $\pi$ under $\MDP$, and $\hat{V}^\pi$ is its corresponding estimate under $\hat\transition_f$, then
    	\[
    	\abs{\hat{V}^\pi_0(x) - V^\pi_0(x)} \leq \frac{\epsilon}{2}.
    	\]
    \end{lemma}
    \begin{proof}
    	The proof follows similarly to that of \lemmaref{lemma:uniform-imagined-MDP-error-bound}.
    	By the performance difference lemma,
    	\begin{align*}
    	\abs{\hat{V}_0^{\pi}(s) - V_0^{\pi}(s)} &\leq \sum_{h = 0}^{\tilde{H} - 1}\expt[\MDP, \pi]{\abs{\left[\left(\hat\transition_f - \transition_f\right)\hat V_{h + 1}^\pi\right](s_h, a_h)}} \\
    	&\leq \sum_{h = 0}^{\tilde{H} - 1}\sum_{(s, a)}\abs{\left[\left(\hat\transition_f - \transition_f\right)\hat V_{h + 1}^\pi\right](s, a)}P_h^{\pi}(s, a).
    	\end{align*}
    	Observe that if $s \in \statespace \setminus \cluster$, $P_h^\pi(s, a) = 0$ for any $\pi$.
    	Furthermore, since the dynamics within $\set{\terminate_S, \terminate_F}$ are known, $(\hat\transition_f - \transition_f)\hat V_{h + 1}^\pi(s, a) = 0$ for $s \in \set{\terminate_S, \terminate_F}$.
    	Therefore, we can restrict the sum to be over $\cluster \times \actionspace$.

    	Now, let $Z_\delta$ denote the set of $\delta$-significant $(s, a)$ pairs in $\cluster \times \actionspace$ from $x$, for some $\delta$ to be determined.
    	For a fixed $h \in [\tilde{H}]$, we can decompose the inner sum as
    	\begin{align*}
    	&\sum_{(s, a)}\abs{\left[\left(\hat\transition_f - \transition_f\right)\hat V_{h + 1}^\pi\right](s, a)}P_h^{\pi}(s, a) \\
    	&\qquad \leq \underbrace{\sum_{(s, a) \in Z_\delta}\abs{\left[\left(\hat\transition_f - \transition_f\right)\hat V_{h + 1}^\pi\right](s, a)}P_h^{\pi}(s, a)}_{\mathrm{(I)}} + \underbrace{\sum_{(s, a) \not\in Z_\delta}\abs{\left[\left(\hat\transition_f - \transition_f\right)\hat V_{h + 1}^\pi\right](s, a)}P_h^{\pi}(s, a)}_{\mathrm{(II)}}.
    	\end{align*}

    	\paragraph{Bounding (II): Error from $\delta$-Insignificance}
    	By the definition of $\delta$-significance,
    	\[
    	\mathrm{(II)} = \sum_{(s, a) \not\in Z_\delta}\abs{\left[\left(\hat\transition_f - \transition_f\right)\hat V_{h + 1}^\pi\right](s, a)}P_h^{\pi}(s, a) \leq H\sum_{s \not\in Z_\delta}P_h^\pi(s) \leq HS\delta \leq \frac{\epsilon}{4H},
    	\]
    	where the last inequality follows from setting $\delta = \epsilon/4SH^2$.

    	\paragraph{Bounding (I): Reference Dynamics Error.}
    	By the Cauchy-Schwarz inequality,
    	\begin{align*}
    	\mathrm{(I)} &= \sum_{(s, a) \in Z_\delta}\abs{\left[\left(\hat\transition_f - \transition_f\right)\hat V_{h + 1}^\pi\right](s, a)}P_h^{\pi}(s, a) \\
    	&\leq \left[\sum_{(s, a) \in Z_\delta}\abs{\left[\left(\hat\transition_f - \transition_f\right)\hat V_{h + 1}^\pi\right](s, a)}^2P_h^{\pi}(s, a)\right]^{1/2}.
    	\end{align*}
    	Then,
    	\begin{align*}
    	&\sum_{(s, a) \in Z_\delta}\abs{\left[\left(\hat\transition_f - \transition_f\right)\hat V_{h + 1}^\pi\right](s, a)}^2P_h^{\pi}(s, a) \\
    	&\qquad \leq \max_{\nu: \statespace \to \actionspace}\sum_{(s, a) \in Z_\delta}\abs{\left[\left(\hat\transition_f - \transition_f\right)\hat V_{h + 1}^\pi\right](s, a)}^2P_h^\pi(s)\ind{\nu(s) = a}.
    	\end{align*}
    	Since $x$ is $\rho$-significant in $\MDP_1(2H)$ by \assumpref{assumption:entrance-reachability}, \lemmaref{lemma:Z-significance-to-2H-significance} together with $\delta$-significance in $\MDP$ implies $\rho\delta$-significance in $\MDP_1(2H)$.
    	Therefore,
    	\[
    	P_h^\pi(s) \leq \max_{\pi}P^\MDP(s \in \tau_\pi) \leq \frac{1}{\rho}\max_{\pi}P^{\MDP_1(2H)}(s \in \tau_\pi) \leq \frac{4HSA}{\rho}\mu(s, a),
    	\]
    	where the last inequality follows by part (a) of the Phase II guarantee in \propref{prop:phase-ii-guarantee}.
    	Substituting into the prior expression,
    	\begin{align*}
    	&\max_{\nu: \statespace \to \actionspace}\sum_{(s, a) \in Z_\delta}\abs{\left[\left(\hat\transition_f - \transition_f\right)\hat V_{h + 1}^\pi\right](s, a)}^2P_h^\pi(s)\ind{\nu(s) = a} \\
    	&\qquad \leq \frac{4HSA}{\rho}\max_{\nu: \statespace \to \actionspace}\sum_{(s, a) \in Z_\delta}\abs{\left[\left(\hat\transition_f - \transition_f\right)\hat V_{h + 1}^\pi\right](s, a)}^2\ind{\nu(s) = a}\mu(s, a) \\
    	&\qquad \leq \frac{4HSA}{\rho}\max_{\nu: \statespace \to \actionspace}\expt[(s, a) \sim \mu]{\abs{\left[\left(\hat\transition_f - \transition_f\right)\hat V_{h + 1}^\pi\right](s, a)}^2\ind{\nu(s) = a}}.
    	\end{align*}
    	Thus by applying part (b) of the Phase II guarantee in \propref{prop:phase-ii-guarantee}, we have that
    	\[
    	\mathrm{(I)} \leq \sqrt{\frac{4HSA}{\rho}\max_{\nu: \statespace \to \actionspace}\expt[(s, a) \sim \mu]{\abs{\left[\left(\hat\transition_f - \transition_f\right)\hat V_{h + 1}^\pi\right](s, a)}^2\ind{\nu(s) = a}}} \leq \frac{\epsilon}{4H}.
    	\]

    	\paragraph{Concluding.} By combining the bounds on (I) and (II), we obtain the desired result.
    \end{proof}

    With this estimation result, we can now prove \propref{prop:solving-disconnected-exits}.

    \begin{proof}[Proof of \propref{prop:solving-disconnected-exits}]
    	Let $\pi$ be the policy found by value iteration using $\hat\transition_f$, which achieves the maximal value in the corresponding MDP.
    	Then, by \lemmaref{lemma:cluster-mdp-value-estimation-error}
    	\[
    	V^{\ast}_0(x) - V^{\pi}_0(x) \leq \underbrace{\left[V^{\ast}_0(s_0) - \hat{V}^{\pi^\ast}_0(s_0)\right]}_{\leq \epsilon/2} + \underbrace{\left[\hat{V}^{\pi^\ast}_0(s_0) - \hat{V}^{\hat\pi}_0(s_0)\right]}_{\leq 0} + \underbrace{\left[\hat{V}^{\hat\pi}_0(s_0) - V^{\hat\pi}_0(s_0)\right]}_{\leq \epsilon/2} \leq \epsilon. \qedhere
    	\]
    \end{proof}

  \subsubsection{Determining Available Exits}

    In this section, we prove that we can determine the set of available exits.
    We have the following formal result:

    \begin{proposition}
    	\label{prop:exit-availability-function}
    	Assume access to the $\epsilon$-suboptimal hierarchy oracle from the previous section and that the guarantee in \thmref{thm:provable-exit-detection} holds.
    	Then, we can implement the function $\avexit{s}: \entry\statespace \to \mathcal{P}(\exit\statespace)$ which, given $s \in \entry{\cluster_s}$, returns $\exit{\cluster_s}$.
    \end{proposition}
    \begin{proof}
    	Fix an input $x \in \entry\statespace$, which we assume belongs to some cluster $\cluster_x$.
    	It suffices to demonstrate that we can implement $\ind{e \in Z_x}$ for any fixed $e \in \entry\statespace$.
    	Define
    	\[
    	f_{e}(s, a) =
    	\begin{cases}
    	\terminate_S & (s, a) = e \\
    	\terminate_F & \text{otherwise}
    	\end{cases}
    	\]
    	and $r_e(s, a) = \ind{(s, a) = e}$.
    	By \assumpref{assumption:in-cluster-exit-reachability}, the MDP $\MDP$ corresponding to the tuple $(x, f_e, r_e, H)$ has optimal value $V^\ast = \epsilon_0\ind{e \in Z_s}$.
    	Additionally, by \lemmaref{lemma:cluster-mdp-value-estimation-error}, $|V^\pi_0(s) - \hat{V}^\pi_0(x)| \leq \epsilon/2$.
    	We now proceed by cases.
    	If $e \not\in Z_x$, then $V_0^\pi(x) = 0$ for any policy $\pi$, and thus value iteration can only find a policy $\pi$ with $\hat{V}^\pi_0(x) \leq \epsilon_0/3$.
    	Otherwise, for $e \in Z_x$, $V_0^\ast(x) = \epsilon_0$, and thus value iteration necessarily must find a $\pi$ with $\hat{V}^\pi_0(x) \geq 2\epsilon_0/3$.
    	Putting these together, if $\hat{V}$ is the optimal estimated value in $\MDP$, then
    	\[
    	\ind{e \in Z_x} = \ind{\hat{V} \geq \frac{2}{3}\epsilon_0}. \qedhere
    	\]
    	Note that this is only implementable for all $e \in \exit\statespace$ since the set of exits are already known.
    \end{proof}

  \subsubsection{Finalizing the Guarantee: Query Complexity}

    In this section, we finalize the proof of the meta-training guarantee by computing the query complexity.

    \begin{proof}[Proof of \thmref{thm:meta-train-guarantee-formal}]
    	As demonstrated by \propref{prop:solving-disconnected-exits} and \propref{prop:exit-availability-function}, running the algorithm in \secref{sec:src-alg-formal} with the parameters in \tableref{table:src-param-choices} provides the desired guarantees with probability at least $1 - p$.

    	To compute the query complexity, observe that we perform the following number of trajectories while executing the algorithm in \secref{sec:src-alg-formal}.
    	\[
    	O\left[T(N_\UCBVI + N_\taskspecific) + N_\Euler^\rewardfree + N_\rewardfree + KN_\exitdetection + TK(N_\Euler^\exitlearning + N_\exitlearning)\right].
    	\]
    	Ignoring terms that do not depend on $T$ or $\epsilon$, we obtain the claim.
    \end{proof}

\subsection{Brute-Force Learning of the Hiearchy}
  \label{sec:brute-force-hierarchy-learning}

  \begin{algorithm}[h!]
  	\caption{Brute-force learning of the latent hierarchy.}
  	\label{alg:brute-force-learning}
  	\begin{algorithmic}[1]
  		\Procedure{LearnHierarchy}{$(\MDP_1, \dots, \MDP_T)$, $N_\Euler$ iterations, $N$ policy samples, threshold $N_\thresh$}
  		\ForAll {$t \in [T], s \in \statespace$}
  		\State Create MDP $\MDP_{t}^s$ so $\prob{\terminate \conditionedon s, a} = 1$ for any $a$.
  		\State $\tilde r_s(s', a') \gets \ind{s' = s}$.
  		\State $\Psi_t^s \gets$ \textsc{Euler}($\MDP_t^s, r, N_\Euler$)
  		\EndFor
  		\ForAll {$t \in [T], s \in \statespace, a \in \actionspace$}
  		\State Modify policies in $\Psi_t^s$ to play $a$ on $s$.
  		\ForAll{$n \in [N]$}
  		\State Sample $\pi \sim \uniform{\Psi_t^s}$.
  		\State Play $\pi$ in $\MDP_t$, collect sample $(s, a, s_n')$ if $(s, a)$ is encountered
  		\EndFor
  		\State $N_t(s, a) \gets$ number of times $(s, a)$ is encountered above.
  		\State $\hat\transition_t(\cdot \conditionedon s, a) \gets$ estimate of $(s, a)$ dynamics in $t$.
  		\EndFor
  		\State \Return $\set{(s, a) \suchthat (\exists t \neq t') \TV{\hat\transition_t - \hat\transition_{t'}} > \beta/2, \min(N_t(s, a), N_{t'}(s, a)) \geq N_\thresh}$.
  		\EndProcedure
  	\end{algorithmic}
  \end{algorithm}

  \begin{theorem}
  	Assume that \algoref{alg:brute-force-learning} is run with parameters satisfying
  	\[
  	N_\Euler = \Omega\left(\frac{CH^3S^2A}{\alpha}\log^3\frac{SAHT}{p}\right)
  	\]
  	and
  	\[
  	N_\thresh = \Omega\left(\frac{S}{\beta^2}\log\frac{SAHNT}{p}\right) \quad \text{and} \quad N = \Omega\left(\frac{CH}{\alpha}N_\thresh + \frac{C^2H^2}{\alpha^2}\log\frac{SAT}{p}\right)
  	\]
  	Then, the set returned by the algorithm is exactly $\exit\statespace$ with probability at least $1 - p$.
  	Furthermore, the algorithm achieves this result with query complexity
  	\[
  	\tilde{O}\left[T\left(\frac{CS^4A}{\alpha} + \frac{CS^2A}{\alpha\beta^2}\right)\right]\mathrm{poly}(H).
  	\]
  \end{theorem}
  \begin{proof}
  	For any $(s, a) \in \exit\statespace$, \lemmaref{lemma:importance-implies-visitation} implies that $s$ is $\alpha/H$-significant for some task $t \in [T]$.
  	Therefore, $s$ is $\alpha/CH$-significant for any task $t \in [T]$, by \assumpref{assumption:non-limiting-exits}.

  	Now, by an argument similar to that used in the proof of \lemmaref{prop:phase-ii-guarantee}, we have that with probability at least $1 - p/3T$, the choice of $N_\Euler$ implies
  	\[
  	\frac{1}{N_\Euler}\sum_{\pi \in \Psi_t^s}P^{\MDP_t}(s \in \tau_\pi) \geq \frac{\alpha}{2CH}
  	\]
  	for any exit $(s, a) \in \exit\statespace$ and a fixed task $t \in [T]$.
  	Therefore, by a union-bound over the tasks, the same guarantee holds for all tasks simultaneously with probability at least $1 - p/3$.

  	Now, for any fixed $(\alpha/CH)$-significant $(s, a)$ pair, sampling from $\Psi_t^s$ at least $N$ times guarantees that with probability at least $1 - p/3SAT$, $N_t(s, a) \geq N_\thresh$.
  	Therefore, once again performing the necessary union-bound, we obtain the same result uniformly over any $(\alpha/CH)$-significant $(s, a)$ and $t \in [T]$ with probability at least $1 - p/3$.

  	Finally, for a fixed $(s, a)$ and $t$, the estimator for $\transition_t(\cdot \conditionedon s, a)$ satisfies the property that when $N(s, a) > 0$,
  	\begin{align*}
  	\TV{\hat\transition_t(\cdot \conditionedon s, a) - \transition_t(\cdot \conditionedon s, a)} \leq \sqrt{\frac{H^2S}{N_t(s, a)}\log\frac{SAHNT}{p}} + \frac{HS}{N_t(s, a)}\log\frac{SAHNT}{p}
  	\end{align*}
  	with probability at least $1 - p/3SAT$, using an argument similar to that used in \lemmaref{lemma:dynamics-estimation-error-bound}.
  	Again, by a union bound, the same guarantee holds for any $(s, a)$ and $t \in [T]$.
  	In particular, for any $(s, a)$ with $N_t(s, a) \geq N_\thresh$,
  	\[
  	\TV{\hat\transition_t(\cdot \conditionedon s, a) - \transition_t(\cdot \conditionedon s, a)} \leq \frac{\beta}{4}.
  	\]
  	Therefore, by a similar argument to \lemmaref{lemma:beta-level-exit-detection}, the following are true:
  	\begin{enumerate}[label=(\alph*)]
  		\item If $(s, a) \in \exit\statespace$, then there exists $t, t'$ for which
  		\[
  		\TV{\hat\transition_t(\cdot \conditionedon s, a) - \hat\transition_{t'}(\cdot \conditionedon s, a)} > \frac{\beta}{2}.
  		\]

  		\item If $(s, a) \not\in \exit\statespace$, then for any $t \neq t'$ with $N_t(s, a), N_{t'}(s, a) \geq N_\thresh$,
  		\[
  		\TV{\hat\transition_t(\cdot \conditionedon s, a) - \hat\transition_t(\cdot \conditionedon s, a)} \leq \frac{\beta}{2},
  		\]
  	\end{enumerate}
  	Putting everything together, we see that the set returned by \algoref{alg:brute-force-learning} is exactly $\exit\statespace$, with probability at least $1 - p$.
  \end{proof}

  \subsection{Technical Lemmas}

  \begin{lemma}
  	\label{lemma:sampling-to-meet-threshold}
  	Let $X_1, \dots, X_M$ be i.i.d. $\bernoulli{p}$ random variables.
  	Then, if
  	\[
  	M = \Omega\left(\frac{N}{p} + \frac{1}{p^2}\log\frac{1}{\delta}\right),
  	\]
  	then with probability at least $1 - \delta$,
  	\[
  	\sum_{i = 1}^{M}{\ind{X_i = 1}} \geq N.
  	\]
  \end{lemma}
  \begin{proof}
  	By applying Hoeffding's inequality,
  	\begin{align*}
  	\prob{\sum_{i = 1}^{M}\ind{X_i = 1} < N} &= \prob{\frac{1}{M}\sum_{i = 1}^{M}\ind{X_i = 1} - p < \frac{N}{M} - p} \\
  	&= \prob{\frac{1}{M}\sum_{i = 1}^{M}\ind{X_i = 0} - (1 - p) > p - \frac{N}{M}} \\
  	&\leq \exp\left[-2M\left(p - \frac{N}{M}\right)^2\right]
  	\end{align*}
  	Setting the final expression to the failure probability $\delta$ and solving, we obtain the quadratic inequality
  	\[
  	p^2M^2 - \left(2Np + \frac{1}{2}\log\frac{1}{\delta}\right)M + N^2 \geq 0.
  	\]
  	Finally, via solving this inequality for $M$, we find that
  	\[
  	M \geq \frac{2N}{p} + \frac{1}{2p^2}\log\frac{1}{\delta}
  	\]
  	is sufficient to guarantee the desired event with failure probability $\delta$, as desired.
  \end{proof}

  \begin{lemma}[Dynamics estimation error bound]
  	\label{lemma:dynamics-estimation-error-bound}
  	Fix a policy $\pi$, MDP with stationary dynamics $\MDP = (\statespace, \actionspace, \transition, r, H)$, and $N \in \naturals$.
  	Assume that $\pi$ is played $N$ times in $\MDP$, and all transitions are used to form an estimator $\hat\transition(\cdot \conditionedon s, a)$ using empirical averages.
  	For any $(s, a) \in \statespace \times \actionspace$, let $N(s, a)$ be the number of times $(s, a)$ is encountered in this process.
  	Then, with probability at least $1 - p$, any $(s, a)$ with $N(s, a) > 0$ satisfies
  	\[
  	\sup_{f: \statespace \to [0, H]}\abs{\left[\left(\hat\transition - \transition\right)f\right](s, a)} \leq \sqrt{\frac{H^2S}{N(s, a)}\log\frac{SAHN}{p}} + \frac{HS}{N(s, a)}\log\frac{SAHN}{p}.
  	\]
  \end{lemma}
  \begin{proof}
  	Assume that the obtained samples are given by $\set{(s_k, a_k, s'_k) \suchthat k \in [HN]}$, so that $(s_{Hn + r}, a_{Hn + r}, s'_{Hn + r + 1})$ is the $r^{\text{th}}$ time step in the $n^{\text{th}}$ execution of $\pi$ in $\MDP$ for any $0 \leq n \leq N - 1$ and $0 \leq r \leq H - 1$.

  	Fix any $(s, a) \in \statespace \times \actionspace$, and assume that $(s_{(j)}, a_{(j)}, s'_{(j)})$ is the $j^{\text{th}}$ sample from $\transition(\cdot \conditionedon s, a)$.
  	Furthermore, let $m_j(s, a)$ denote the index at which the $j^{\text{th}}$ sample is obtained.
  	We claim that for any $s^\ast \in \statespace$ and $0 < M \leq HT$,
  	\begin{align*}
  	&\abs{\frac{1}{M}\sum_{j = 1}^{M}\ind{m_j(s, a) \leq HT}\left(\ind{s'_{(j)} = s^\ast} - \transition(s^\ast \conditionedon s, a)\right)} \\
  	&\qquad \leq \sqrt{\frac{\transition(s' \conditionedon s, a)}{M}\log\frac{S}{\delta}} + \frac{1}{M}\log\frac{S}{\delta}.
  	\end{align*}
  	Let $\mathcal{F}_i$ be defined as the $\sigma$-algebra induced by the set of random variables
  	\[
  	\set{\left(m_j(a), \ind{s_{(j)}' = s^\ast}\right) \suchthat j \leq i}.
  	\]
  	Clearly, $(\mathcal{F}_i)$ is a filtration such that the $j^{\text{th}}$ term in the sum above is measurable with respect to $\mathcal{F}_j$.
  	Furthermore, observe that
  	\begin{align*}
  	&\expt{\ind{m_j(s, a) \leq HT}\left(\ind{s_{(j)}' = s^\ast} - \transition(s^\ast \conditionedon s, a)\right) \suchthat \mathcal{F}_{j - 1}} \\
  	&\qquad = \expt{\ind{s_{(j)}' = s^\ast} - \transition(s^\ast \conditionedon s, a) \suchthat \mathcal{F}_{j - 1}, m_j(s, a) \leq HT}\prob{m_j(s, a) \leq HT} \\
  	&\qquad = 0.
  	\end{align*}
  	Therefore, the random variables in the sum forms martingale difference sequence.
  	Furthermore, the sequence is bounded in $[-1, 1]$, and satisfies
  	\begin{align*}
  	&\var{\ind{m_j(s, a) \leq HT}\left(\ind{s_{(j)}' = s^\ast} - \transition(s^\ast \conditionedon s, a)\right) \suchthat \mathcal{F}_{j - 1}} \\
  	&\qquad = \expt{\var{\ind{s_{(j)}' = s^\ast} - \transition(s^\ast \conditionedon s, a) \suchthat \mathcal{F}_{j - 1}, m_j(s, a) \leq HT} \suchthat \mathcal{F}_{j - 1}} \\
  	&\qquad \leq \transition(s^\ast \conditionedon s, a).
  	\end{align*}
  	Therefore, by applying Azuma-Bernstein, we have that
  	\begin{align*}
  	&\abs{\frac{1}{M}\sum_{j = 1}^{M}\ind{m_j(s, a) \leq HT}\left(\ind{s'_{(j)} = s^\ast} - \transition(s^\ast \conditionedon s, a)\right)} \\
  	&\qquad \leq \sqrt{\frac{2\transition(s' \conditionedon s, a)}{M}\log\frac{SAHN}{\delta}} + \frac{2}{M}\log\frac{SAHN}{\delta}.
  	\end{align*}
  	with probability at least $1 - p/SAHN$.

  	By applying a union bound on $(s, a, s^\ast)$ and $M$, we thus have that with probability at least $1 - p$,
  	\begin{align*}
  	&\abs{\frac{1}{M}\sum_{j = 1}^{M}\ind{m_j(s, a) \leq HT}\left(\ind{s'_{(j)} = s^\ast} - \transition(s^\ast \conditionedon s, a)\right)} \\
  	&\qquad \leq \sqrt{\frac{2\transition(s' \conditionedon s, a)}{M}\log\frac{SAHN}{\delta}} + \frac{2}{M}\log\frac{SAHN}{\delta}
  	\end{align*}
  	holds for any $(s, a, s^\ast)$ and $M$.
  	Conditioned on this event, we thus have that for any $(s, a)$ with $N(s, a) > 0$,
  	\begin{align*}
  	\TV{\hat\transition_t(\cdot \conditionedon s, a) - \transition_t(\cdot \conditionedon s, a)} &= \frac{1}{2}\sum_{s' \in \statespace}\abs{\hat\transition_t(s' \conditionedon s, a) - \transition_t(s' \conditionedon s, a)} \\
  	&\lesssim \sum_{s' \in \statespace}\sqrt{\frac{\transition(s' \conditionedon s, a)}{N(s, a)}\log\frac{SAHN}{\delta}} + \frac{S}{N(s, a)}\log\frac{SAHN}{\delta} \\
  	&\lesssim \sqrt{\frac{S}{N(s, a)}\log\frac{SAHN}{\delta}} + \frac{S}{N(s, a)}\log\frac{SAHN}{\delta}.
  	\end{align*}
  	The final result follows simply by noting that
  	\[
  	\abs{\left[(\hat\transition_t - \transition_t)f\right](s, a)} \lesssim \TV{\hat\transition_t(\cdot \conditionedon s, a) - \transition_t(\cdot \conditionedon s, a)}\norm[\infty]{f}. \qedhere
  	\]
  \end{proof}

  \begin{lemma}
  	\label{lemma:value-gap-implies-visitation}
  	Fix two MDPs $\MDP = (\statespace, \actionspace, \transition, r, H)$ and $\MDP' = (\statespace, \actionspace, \transition', r, H)$.
  	Let $\Delta$ denote the subset of $\statespace \times \actionspace \times [H]$ for which $\transition_h(\cdot \conditionedon s, a) \neq \transition_h'(\cdot \conditionedon s, a)$.
  	Then, for any policy $\pi$,
  	\[
  	V^{\MDP', \pi}_0(s_0) - V^{\MDP, \ast}_0(s_0) > \rho \implies \prob[\MDP]{\tau_\pi \intersect \Delta \neq \emptyset} = \prob[\MDP']{\tau_\pi \intersect \Delta \neq \emptyset} > \frac{\rho}{H}.
  	\]
  \end{lemma}
  \begin{proof}
  	Write $q = \prob[\MDP']{\tau_\pi \intersect \Delta \neq \emptyset}$.
  	Note that $V_0^{\MDP', \pi}(s_0)$ can be decomposed as
  	\begin{align*}
  	V_0^{\MDP', \pi}(s_0) &= q\expt[\MDP']{\sum_{h = 0}^{H - 1}r_h(s_h, a_h) \suchthat \tau_\pi \intersect \Delta \neq \emptyset} \\
  	&\qquad + (1 - q)\expt[\MDP']{\sum_{h = 0}^{H - 1}r_h(s_h, a_h) \suchthat \tau_\pi \intersect \Delta = \emptyset} \\
  	&\leq qH + (1 - q)\expt[\MDP']{\sum_{h = 0}^{H - 1}r_h(s_h, a_h) \suchthat \tau_\pi \intersect \Delta = \emptyset}.
  	\end{align*}
  	Since $\transition$ and $\transition'$ agree on $(\statespace \times \actionspace \times [H]) \setminus \Delta$, the dynamics of $\MDP$ and $\MDP'$ agree up until $\pi$ performs an action in $\Delta$, and thus
  	\begin{align*}
  	\prob[\MDP]{\tau_\pi \intersect \Delta \neq \emptyset} &= \prob[\MDP']{\tau_\pi \intersect \Delta \neq \emptyset} \\
  	\expt[\MDP]{\sum_{h = 0}^{H - 1}r_h(s_h, a_h) \suchthat \tau_\pi \intersect \Delta = \emptyset} &= \expt[\MDP']{\sum_{h = 0}^{H - 1}r_h(s_h, a_h) \suchthat \tau_\pi \intersect \Delta = \emptyset}
  	\end{align*}
  	Furthermore,
  	\[
  	(1 - q)\expt{\sum_{h = 0}^{h - 1}r_h(s_h, a_h) \suchthat \tau_\pi \intersect \delta = \emptyset} \leq V^{\MDP, \pi}_0(s_0) \leq V^{\MDP, \ast}_0(s_0).
  	\]
  	Putting everything together,
  	\[
  	V_0^{\MDP', \pi}(s_0) \leq qH + V_0^{\MDP, \ast}(s_0) \implies q > \frac{\rho}{H}. \qedhere
  	\]
  \end{proof}

	\clearpage

	\section{Meta-Test Proofs}
		\label{sec:meta-test-proofs}
		We now provide an analysis of the regret incurred by a learner using an approximately learned hierarchy at meta-test time.
We first show that the hierarchy oracle from the source tasks can provide useful temporally extended behavior.
We then show that using these policies results in bounded suboptimality and achieves a better regret bound compared to standard UCB-VI.

Throughout this section, we fix an optimal $\pi^\ast$ satisfying the conditions of \assumpref{assumption:meta-test-compatibility}.
Furthermore, we assume that we have access to a hierarchy oracle that provides $\epsilon$-suboptimal policies as defined in \defref{definition:hierarchy-oracle}.

\subsection{Using the Hierarchy Oracle}

  In this section, we show that the hierarchy oracle can be used to implement two useful behaviors: (1) reaching exits and (2) behaving optimally within a cluster.

  \subsubsection{Near-Optimal Goal Reaching}

    Assume that the agent is currently at a state $z \in \set{s_0} \union \entry\statespace$ at time step $h$, and intends to exit the current cluster $\cluster$ via exit $g = (s^\ast, a^\ast) \in \exit\cluster$.
    We obtain a policy implementing the high-level intent as follows:

    \begin{enumerate}[label=(\arabic*)]
    	\item Define the termination for any $(s, a) \in \exit\statespace$ as:
    	\[
    	f_g(s, a) \defas
    	\begin{cases}
    	\terminate_S & (s, a) = g \\
    	\terminate_F & \text{otherwise}
    	\end{cases}
    	\]

    	\item Define reward as $r_{\terminate_S}(s, a) \defas \ind{s = \terminate_S}$

    	\item Provide $(z, f_g, r_{\terminate_S}, H - h)$ to the hierarchy oracle and obtain a policy $\pi_{z, g, h}$.
    \end{enumerate}

    For simplicity, we will write $T^\hier_{H - h}(z, g)$ for $T^{\pi_{z, g, h}}_{H - h}(z, g)$ throughout our analysis.
    The following proposition quantifies the performance of the obtained policy:

    \begin{proposition}
    	$T^\hier$ satisfies the following inequality:
    	\[
    	\expt{T^\hier_{H - h}(z, g)} \leq T^\ast_{H - h}(z, g) + \epsilon.
    	\]
    \end{proposition}
    \begin{proof}
    	Due to the definition of $\transition_{f_g}$ and $r$, observe that for any $\pi$,
    	\[
    	V^\pi_0(z) = \expt{\sum_{h = 0}^{H - h}r(s_h, a_h) \suchthat s_0 = z} = (H - h) - \expt{T^\pi_{H - h}(z, g)}.
    	\]
    	Therefore,
    	\begin{align*}
    	&(H - h) - T^\ast_{H - h}(z, g) - \epsilon \leq (H - h) - \expt{T^\hier_{H - h}(z, g)} \\
    	&\qquad \implies \expt{T^\hier_{H - h}(z, g)} \leq T^\ast_{H - h}(z, e) + \epsilon. \qedhere
    	\end{align*}
    \end{proof}

  \subsubsection{Near-Optimal Within-Cluster Behavior}

    Assume that the agent is currently at a state $z \in \set{s_0} \union \entry\statespace$ at time $h$, and intends to remain in the current cluster $Z$ while maximizing a given reward function $r$.
    We obtain a policy for this high-level intent as follows:

    \begin{enumerate}[label=(\arabic*)]
    	\item Define transition dynamics for any $(s, a) \in \exit\cluster$ as $\transition(\cdot \conditionedon s, a) = \delta(\terminate_F)$.

    	\item Provide $\transition$, r, and planning horizon $H - h$ to the hierarchy oracle, and obtain a policy $\pi$.
    \end{enumerate}

\subsection{Formal Learning Procedure}
  \label{sec:meta-test-learning-procedure}

  In this section, we describe the procedure for learning a policy using the oracle-provided policies described in the previous section.
  Formally, we construct a surrogate MDP whose dynamics are determined by $\MDP$ and the oracle.
  We can then apply any tabular learning method to this new MDP (in our case, \Euler{}), obtaining a policy in the surrogate MDP that readily translates into a policy in $\MDP$.

  The components defining the surrogate $\MDP_\meta = (\mathcal{Z}, \mathcal{G}, \transition_\meta, R_\meta, H_\effective)$ are as follows:

  \paragraph{Meta-state space $\mathcal{Z}$.} We set
  \[
  \mathcal{Z} \defas \left(\entry\statespace \times \set{0, \dots, \bar{H} + 1}\right) \union \set{\terminate},
  \]
  where $\bar{H}$ is a high-probability bound on the time to move through $\horizon_\effective$ exits (to be determined later).
  We incorporate the time step into the meta-state to ensure that both the dynamics and reward are computable from the state information (ensuring that $\MDP_\meta$ is indeed an MDP).

  \paragraph{Meta-action space $\mathcal{G}$.} Given a current meta-state $(s, h)$ where $s \in \cluster$, the available meta-actions $\mathcal{G}$ can be identified with $\exit\cluster \union \set{\terminate}$.

  \begin{algorithm}[h]
  	\caption{Performing a Meta-Transition}
  	\label{alg:meta-transition}
  	\begin{algorithmic}[1]
  		\Procedure{PerformMetaTransition}{$(z, g) \in \mathcal{Z} \times \mathcal{G}$}
  		\newline\Comment{Executes the desired meta-transition in the original MDP $\MDP$.}
  		\If {$z = \terminate$}
  		\State \Return {$\terminate$}
  		\ElsIf {$z = (s, h)$}
  		\If{$h \leq \bar{H}$}
  		\If {$s \in Z^\ast$ or $g = \terminate$}
  		\State Execute within-cluster policy from oracle until termination.
  		\State \Return $\terminate$
  		\Else
  		\State Execute $\pi_{z, g, h}$ obtained from oracle until $g$ is performed or $h = \bar{H}$.
  		\State $s', h' \gets$ current state and time step
  		\If {$g$ was performed}
  		\State \Return $(s', h')$
  		\Else
  		\State \Return{$(s, \bar{H} + 1)$}
  		\EndIf
  		\EndIf
  		\Else
  		\If {$s \in Z^\ast$ or $g = \terminate$}
  		\State \Return{$\terminate$}
  		\Else
  		\State \Return{$(s, h)$}
  		\EndIf
  		\EndIf
  		\EndIf
  		\EndProcedure
  	\end{algorithmic}
  \end{algorithm}

  \paragraph{Meta-dynamics $\transition_\meta$.} Fix $(z, g) \in \mathcal{Z} \times \mathcal{G}$ for some $z \neq \terminate$, so that $z = (s, h)$.
  We consider the procedure in \algoref{alg:meta-transition} for generating the meta-dynamics.
  Intuitively, we execute a meta-action $g \neq \terminate$ by running the oracle-provided policy until the learner encounters $g$, or has acted for $\bar{H}$ timesteps in the current episode.
  On the other hand, if $g = \terminate$, the agent executes the oracle-provided $\epsilon$-suboptimal policy that remains within the current cluster and acts for $H - h$ timesteps.

  Formally, the next state $z'$ is given by
  \[
  z' =
  \begin{cases}
  \terminate & \text{$s \in \cluster^\ast$ or $g = \terminate$} \\
  (s', h') & h \leq \bar{H} \\
  (s, h) & \text{otherwise} \\
  \end{cases},
  \]
  where $s'$ and $h'$ are generated given $T^\hier_{H - h}(s, g)$ as
  \begin{align*}
  h' \conditionedon T^\hier_{H - h}(s, g) &= \min(h + T^\hier_{H - h}(s, g), \bar{H} + 1) \\
  s' \conditionedon h' &\sim
  \begin{cases}
  \transition(\cdot \conditionedon g) & h' \leq \bar{H} \\
  \delta(s) & \text{otherwise}
  \end{cases}.
  \end{align*}
  Note that the learner can only execute meta-actions while $h \leq \bar{H}$.
  Furthermore, given access to $\MDP$, one can easily simulate the dynamics of $\MDP_\meta$.

  \paragraph{Meta-reward $R_\meta$.} Fix $((s, h), g) \in \mathcal{Z} \times \mathcal{G}$.
  Recall that the reward function of $\MDP$ is supported on $\exit\statespace \union \interior{(Z^\ast)}$.
  Thus, this reward function can be lifted onto $\MDP_\meta$.
  Formally, we define the following reward function:
  \[
  R_\meta(z, g) =
  \begin{cases}
  W_h(s) & \text{$z = (s, h)$, $s \in \cluster^\ast$ and $h \leq \bar{H}$} \\
  r(g) & \text{$z = (s, h)$, $s \not\in \cluster^\ast$ and $h' \leq \bar{H}$} \\
  0 & \text{otherwise}
  \end{cases},
  \]
  where $W_h(s)$ is the random sum of rewards obtained by playing a within-cluster policy starting from $s'$ for the rest of the episode.
  Note that $R_\meta$ depends on $\transition_\meta$ and is thus random.
  Furthermore, this reward function is consistent with how meta-transitions are performed in \algoref{alg:meta-transition}.

  \paragraph{Meta-horizon $\horizon_\effective$.} Recall that there exists an optimal policy that encounters at most $\horizon_\effective$ exits with high probability.
  Accordingly, we limit the learner to being able to choose $\horizon_\effective$ high-level actions, which recall can be choices of exits.

  \paragraph{Solving $\MDP_\meta$.} To obtain the desired policy, we apply \Euler{} to $\MDP_\meta$.
  By the construction in \algoref{alg:meta-transition}, the policy set returned by $\Euler{}$ easily translates into policies on $\MDP$.
  Furthermore, the value of this policy is the same on both MDPs.

\subsection{Proving the Regret Bound}

  Having defined the procedure for learning a policy using the hierarchy, we now proceed with the regret analysis.
  Our analysis proceeds by constructing a policy expressible in $\MDP_\meta$ that achieves near-optimal returns by imitating the high-level decisions made by $\pi^\ast$.
  We then use this policy as a comparator policy when applying \Euler{} regret bounds to $\MDP_\meta$.

  To formally construct the desired comparator policy, we need to first define the notion of a meta-history, which contains the set of high-level decisions made by any policy:

  \begin{definition}
  	Fix a policy $\pi$, which given some horizon $L$, generates a (random) trajectory $(s_0, a_0, \dots, s_L)$.
  	Let $\exit{\pi}$ be the number of exits performed in the trajectory, i.e.
  	\[
  	\exit{\pi} = \sum_{h = 0}^{L - 1}\ind{(s_h, a_h) \in \exit\statespace}.
  	\]
  	The meta-history $\history_\meta(\pi)$ corresponding to this trajectory is the sequence
  	\[
  	(z_0, g_0, z_1, g_1, \dots, z_{\exit\pi}) = (s_{i_0}, (s_{j_0}, a_{j_0}), s_{i_1}, (s_{j_1}, a_{j_1}) \dots, s_{i_{\exit{\pi}}}),
  	\]
  	where
  	\begin{equation*}
  	\begin{aligned}
  	i_n &\defas
  	\begin{cases}
  	0 & n = 0 \\
  	j_{n - 1} + 1 & \text{otherwise} \\
  	\end{cases} \\
  	j_n &\defas \min_{h = i_n, \dots, L - 1}\ind{(s_h, a_h) \in \exit\statespace}.
  	\end{aligned}
  	\end{equation*}
  	Note that $z_i \in \entry\statespace$ and $g_i \in \exit\statespace$ for all $i = 0, \dots, \exit\pi$.
  	We omit $\pi$ in writing $\history_\meta$ when the underlying policy $\pi$ is understood.
  \end{definition}

  Informally, $\history_\meta$ tracks all entrances and exits contained in a trajectory generated by $\pi$.
  We define the length of a meta-history $\history_\meta$, denoted as $\card{\history_\meta}$, as the number of exits contained in $\history_\meta$.

  \subsubsection{Policy Construction}

    We now proceed with constructing the desired policy.
    Intuitively, the comparator imitates the distribution over $\history_\meta(\pi^\ast)$, conditioned on $\card{\history_\meta(\pi^\ast)} \leq \horizon_\effective$.
    To see why this is sufficient for near-optimality, recall that the reward on $\MDP_\target$ is supported on $\exit{\cluster^\ast} \union \interior{(\cluster^\ast)}$.
    Consequently, by imitating the distribution over meta-histories, the policy is expected to obtain roughly the same sum of rewards in expectation from the exits.
    Therefore, all that remains is to ensure that the learner collects roughly the same sum of rewards from $\cluster^\ast$, which is the same as ensuring that this policy does not take too long to reach $\cluster^\ast$.

    \textbf{Construction.} Let $\history$ be the running meta-history, containing $k \leq \horizon_\effective$ actions.
    The optimal policy induces a distribution $q(\cdot \conditionedon \history)$ over $\actionspace_\meta$ representing the next exit it takes\footnote{The distribution $q$ can return $\terminate$ if the learner stays in the cluster until episode termination.}.
    We then define $\pi$ as
    \[
    \pi(\cdot \conditionedon z, \history) =
    \begin{cases}
    q(\cdot \conditionedon \history) & z = (s, h), h < \bar{H} \\
    \terminate & \text{otherwise}.
    \end{cases}
    \]
    Observe that $\pi$ terminates the episode upon reaching $\bar{H}$.
    Furthermore, this policy is dependent on the meta-history.
    However, since $\MDP_\meta$ is an MDP, there exists a stationary policy that achieves at least the same value.

  \subsubsection{Suboptimality Analysis}

    In this section, we prove that $\pi$ achieves bounded suboptimality.
    Rather than analyzing $\pi$ directly in $\MDP_\meta$, we construct a new $\tilde\MDP_\meta$ and $\tilde\pi$ to better track the meta-history.
    In particular, conditioned on the event that $\pi$ requires more than $\bar{H}$ time steps to execute, then the agent would not be able to imitate the full meta-history generated by $\pi^\ast$, even after having performed less than $\horizon_\effective$ exits.

    \paragraph{Constructing a surrogate for analysis.}
    We now formalize the construction of the surrogate MDP $\tilde\MDP_\meta$ and the policy $\tilde\pi$ corresponding to $\pi$ in this MDP.
    To obtain $\tilde\MDP_\meta$, we redefine the dynamics from $\MDP_\meta$ so that $s' \conditionedon h' \sim \transition(\cdot \conditionedon g)$ in $\tilde\MDP_\meta$.
    In effect, we allow the policy to continue performing transitions beyond $\bar{H}$, although without any reward.
    Accordingly, we define $\tilde\pi$ as $\tilde\pi(\cdot \conditionedon z, \history) = q(\cdot \conditionedon \history)$.
    The following lemma formalizes how $\tilde\pi$ and $\tilde\MDP_\meta$ have desirable properties for the analysis:

    \begin{lemma}[Surrogate Policy Characterization]
    	\label{lemma:surrogate-policy-characterization}
    	Let $\mu^\ast$ denote the distribution of $\history_\meta(\pi^\ast) \conditionedon \card{\history_\meta(\pi^\ast)} \leq \horizon_\effective$ in $\MDP_\target$, and $\tilde\mu$ the distribution of $\history(\tilde\pi)$ in $\tilde\MDP_\meta$.
    	Then, $(1 - \zeta)\mu^\ast \leq \tilde\mu$.
    \end{lemma}
    \begin{proof}
    	Let $\nu^\ast$ be the distribution induced by the following procedure:
    	\begin{enumerate}[label=(\arabic*)]
    		\item Sample a meta-history from the distribution $\history_\meta(\pi^\ast) \conditionedon \card{\history_\meta(\pi^\ast)} > \horizon_\effective$.

    		\item Truncate the obtained meta-history to length $\horizon_\effective$.
    	\end{enumerate}
    	It is easy to see from the definition of $\tilde\pi$ that $\tilde\mu = (1 - \zeta)\mu^\ast + \zeta\nu^\ast$.
    	The desired result follows.
    \end{proof}

    Thus, we have indeed shown the desired property that $\tilde\pi$ properly tracks the (truncated) meta-history generated by $\pi^\ast$.
    To justify performing our analysis on $(\tilde\MDP_\meta, \tilde\pi)$, we have the following result, which shows that any result on the value of the pair above applies to the value of $\pi$ in $\MDP_\meta$.

    \begin{lemma}
    	\label{lemma:surrogate-achieves-same-value}
    	As constructed above, $V^{\tilde{\MDP}_\meta, \tilde{\pi}}_0(s_0) = V^{\MDP_\meta, \pi}_0(s_0)$.
    \end{lemma}
    \begin{proof}
    	We write $\MDP \defas \MDP_\meta$ and $\MDP' \defas \tilde\MDP_\meta$.
    	Similarly, we write $\pi' \defas \tilde\pi$.
    	We proceed by proving a chain of equalities.

    	($V^{\pi', \MDP'}(s_0) = V^{\pi, \MDP'}(s_0)$). We omit $\MDP'$ in this part of the argument for clarity.
    	By the performance difference lemma, we have that for any $k \in [\horizon_\effective]$ and $z \in \statespace_\meta$,
    	\[
    	V^{\pi}_{0}(s_0) - V_0^{\pi'}(s_0) = \sum_{j = 0}^{\horizon_\effective - 1}\expt[z \sim d_{j}^{\pi}]{A_j^{\pi'}(z, \pi)}.
    	\]
    	Let $\Delta \defas \set{z \in \statespace_\meta \suchthat z = (s, h), s \not\in Z^\ast, h \geq \bar{H}}$, which is the set on which $\pi$ and $\pi'$ disagree.
    	Observe that for any $\pi$ and $k$, $V_k^\pi(z) = 0$ for any $z \in \Delta$, and thus $A_k^{\pi'}(z, \pi) = 0$ for all such states.
    	For any other $z$, $A_k^{\pi'}(z, \pi)$ is clearly $0$.
    	Thus, we obtain the desired result.

    	($V^{\pi, \MDP'}(s_0) = V^{\pi, \MDP}(s_0)$) We omit $\pi$ in this part of the argument for clarity.
    	Using the simulation lemma,
    	\[
    	V_0^{\MDP}(s_0) - V_0^{\MDP'}(s_0) = \sum_{j = 0}^{\horizon_\effective - 1}\expt[(z, g) \sim d_j^{\MDP'}]{[(\transition_\MDP - \transition_{\MDP'})V_{j + 1}^{\MDP}](z, g)}.
    	\]
    	Observe that the behavior of the two MDPs are identical conditioned on $h' \leq \bar{H}$.
    	On the other hand, conditioned on $h' > \bar{H}$, $\pi$ can no longer receive rewards from either MDP.
    	Therefore, $[(\transition_{\MDP} - \transition_{\MDP'})V_{j}^{\MDP}](z, g) = 0$ for any $j, z, g$ by decomposing the relevant expectations along the two events.
    	We thus obtain the desired result.
    \end{proof}

    \paragraph{Analyzing the surrogate.}
    With the results above, we now proceed to analyze the difference in values
    \[
    V^{\MDP_\target, \ast}_0(s_0) - V^{\tilde\MDP_\meta, \tilde\pi}_0(s_0),
    \]
    which then implies the desired suboptimality result.
    First, we have the following lemma characterizing the time $\tilde{\pi}$ requires to fully execute a given meta-history in the base MDP $\MDP$:

    \begin{lemma}
    	\label{lemma:ref-policy-execution-time}
    	Fix any $\history_\meta = (z_0, g_0, \dots)$ such that $\card{\history_\meta} \leq \horizon_\effective$.
    	Furthermore, define the sequence of reaching times
    	\[
    	T_0 \defas 0 \quad \text{and} \quad T_k \defas T_{k - 1} + T^{\hier}_{H - T_{k - 1}}(z_{k - 1}, u_{k - 1}).
    	\]
    	We define $T^\hier(\history_\meta)$ to be the time required by the hierarchy to execute $\history_\meta$, which is formally given by $T_{\card{\history_\meta}}$ in the sequence above.
    	Then,
    	\begin{enumerate}[label=\textup{(\alph*)}]
    		\item $\expt{T^\hier(\history_\meta)} \leq [1 + (1 + \gamma)W + \epsilon]\horizon_\effective$.

    		\item Let $\sigma^2 \defas \beta^2[(1 + \gamma)W + \epsilon]^2\horizon_\effective$.
    		Then, for any $t > 0$,
    		\[
    		\prob{T^\hier(\history_\meta) \geq [1 + (1 + \gamma)W + \epsilon]\horizon_\effective + t} \leq e^{-t^2/2\sigma^2}.
    		\]
    	\end{enumerate}
    \end{lemma}
    \begin{proof}
    	We prove the two parts separately:
    	\begin{enumerate}[label=(\alph*)]
    		\item We will prove via induction that $\expt{T_k} \leq k\left[1 + (1 + \gamma)W + \epsilon\right]$.
    		For any $k$ and $T_{k - 1}$,
    		\begin{align*}
    		\expt{T_{H - T_{k - 1}}^\hier(z_{k - 1}, g_{k - 1}) \suchthat T_{k - 1}} &\leq \expt{T_{H - T_{k - 1}}^\ast(z_{k - 1}, g_{k - 1}) \suchthat T_{k - 1}} + \epsilon \\
    		&= 1 + \expt{T_{H - T_{k - 1}}^\ast(z_{k - 1}, s(g_{k - 1})) \suchthat T_{k - 1}} + \epsilon \\
    		&\leq 1 + (1 + \gamma)W + \epsilon,
    		\end{align*}
    		where the first inequality uses properties of the hierarchy oracle, while the final inequality follows by combining \assumpref{assumption:env-low-var-and-regular}(b) and \assumpref{assumption:skill-horizon}.
    		Therefore, by linearity and the tower property of expectation,
    		\begin{align*}
    		\expt{T_k} &= \expt{T_{k - 1}} + \expt{T_{H - T_{k - 1}}^\hier(z_{k - 1}, g_{k - 1})} \\
    		&= \expt{T_{k - 1}} + \expt{\expt{T_{H - T_{k - 1}}^\hier(z_{k - 1}, g_{k - 1}) \suchthat T_{k - 1}}} \\
    		&\leq \expt{T_{k - 1}} + 1 + (1 + \gamma)W + \epsilon.
    		\end{align*}
    		The desired result then follows by induction.

    		\item Let $B_k \defas k\left[1 + (1 + \gamma)W + \epsilon\right]$ and $f_k(t) \defas \expt{T_{H - t}^\hier(z_k, g_k)}$.
    		Note that for any $k$ and $t$, $B_{k - 1} + f_{k - 1}(t) \leq B_k$, by following the argument in (a).
    		Therefore, for any $\lambda > 0$,
    		\begin{align*}
    		&\expt{\exp\left\{\lambda\left(T_k - B_k\right)\right\}} \\
    		&\quad = \expt{\expt{\exp\left\{\lambda\left(T_k - B_k\right)\right\} \suchthat T_{k - 1}}} \\
    		&\quad \leq \expt{\expt{\exp\left\{\lambda\left(T_{k - 1} + T_{H - T_{k - 1}}^\hier(z_{k - 1}, g_{k - 1}) - B_{k - 1} - f_{k - 1}(T_{k - 1})\right)\right\} \suchthat T_{k - 1}}},
    		\end{align*}
    		where the last inequality uses the monotonicity of the exponential function.
    		Therefore, by applying the sub-Gaussian condition given in \assumpref{assumption:env-low-var-and-regular},
    		\begin{align*}
    		&\expt{\exp\left\{\lambda\left(T_k - B_k\right)\right\}} \\
    		&\quad \leq \mathbb{E}\left[\exp\left\{\lambda\left(T_{k - 1} - B_{k - 1} \right)\right\}\vphantom{\expt{T_{H - T_{k - 1}}^\hier(z_{k - 1}, g_{k - 1}) \suchthat T_{k - 1}}}\right. \\
    		&\qquad\qquad\qquad \left.\expt{\exp\left\{\lambda\left(T_{H - T_{k - 1}}^\hier(z_{k - 1}, g_{k - 1}) - f_{k - 1}(T_{k - 1})\right)\right\} \suchthat T_{k - 1}}\right] \\
    		&\quad \leq \expt{\exp\left\{\lambda(T_{k - 1} - B_{k - 1})\right\}}\exp\left[\lambda^2C^2/2\right],
    		\end{align*}
    		where we have used the fact that $T_{H - T_{k - 1}}^\hier(z_{k - 1}, s(g_{k - 1}))$ has a sub-Gaussian upper tail with variance proxy
    		\begin{align*}
    		C^2 &= \beta^2\expt{T_{H - T_{k - 1}}^\pi(z_{k - 1}, s(g_{k - 1})) \suchthat T_{k - 1}}^2 \\
    		&\leq \beta^2\left[(1 + \gamma)W + \epsilon\right]^2.
    		\end{align*}
    		Note that we have once again used the properties of the hierarchy oracle, and Assumptions \ref{assumption:env-low-var-and-regular} and \ref{assumption:skill-horizon}.
    		Therefore, by induction, $\expt{\exp\left\{\lambda\left(T_k - B_k\right)\right\}} \leq \expt{\lambda^2(\sqrt{k}C)^2/2}$, from which the desired tail bound follows by making use of Chernoff's inequality. \qedhere
    	\end{enumerate}
    \end{proof}

    As we have shown that $\tilde\pi$ closely tracks the meta-history of $\pi^\ast$ and have analyzed the distribution of time it takes to execute a given meta-history, we can now analyze its suboptimality:

    \begin{lemma}
    	There exists a policy $\pi$ expressible in $\MDP_\meta$ such that
    	\[
    	V^{\MDP_\target, \ast}_0(s_0) - V^{\MDP_\target, \pi}_0(s_0) \lesssim (1 + \horizon_\effective + \beta\sqrt{\horizon_\effective})\epsilon + \left[\gamma\horizon_\effective + \beta(1 + \gamma)\sqrt{\horizon_\effective}\right]W + \zeta H.
    	\]
    \end{lemma}
    \begin{proof}
    	Assume that $\pi^\ast$ generates a (random) meta-history of length $N$ given by $\history_\meta = (z_0, g_0, z_1, g_1, \dots, z_N)$.
    	Furthermore, let $T^\ast$ denote the (random) time $\pi^\ast$ takes to reach $z_N$.
    	Then, given $\history_\meta$ and $T^\ast$, observe that we can write
    	\[
    	V_0^\ast(s_0) = \expt{R_{T^\ast}^\ast(\history_\meta)}, \quad \text{where } R_{T}^\ast(\history_\meta) \defas V_{T}^\ast(z_N)\ind{z_N \in \cluster^\ast} + \sum_{k = 0}^{N - 1}r(g_k),
    	\]
    	using the assumptions on the reward function and condition (a) in \assumpref{assumption:meta-test-compatibility}.
    	Subsequently, letting $E$ be the event $\set{N \leq \horizon_\effective}$, we can bound the right-hand side as
    	\[
    	V_0^\ast(s_0) = \expt{R_{T^\ast}^\ast(\history_\meta)} \leq (1 - \zeta)\expt{R_{T^\ast}^\ast(\history_\meta) \suchthat E} + \zeta H,
    	\]
    	where we have used \assumpref{assumption:meta-test-compatibility} to bound the probability that $N > \horizon_\effective$.

    	Our goal for the rest of this proof is to transform the expectation on the right-hand side into a form that lower bounds $V^{\pi}_0(s_0)$.
    	To this end, we define
    	\[
    	R^\hier_{T}(\history_\meta) \defas V_T^\hier(z_N)\ind{z_N \in Z^\ast} + \sum_{k = 0}^{N - 1}r(g_k),
    	\]
    	and the sequence of times
    	\[
    	T_0 = 0 \quad \text{and} \quad T_k \defas T_{k - 1} + T^\hier_{H - T_{k - 1}}(z_k, g_k).
    	\]
    	Note that $R^\hier$ and $T_N$ are analogous to $R^\ast$ and $T^\ast$, respectively.
    	Then, letting $F$ be the event $\set{T_N \leq \bar{H}}$, note that
    	\begin{align*}
    	V_0^\ast(s_0) &\leq (1 - \zeta)\expt{R^\ast_{T^\ast} \suchthat E} + \zeta H \\
    	&= (1 - \zeta)\expt{R^\ast_{T^\ast} - R^\hier_{T_N} + R^\hier_{T_N} \suchthat E} + \zeta H \\
    	&\leq \underbrace{\expt{\left(R^\ast_{T^\ast} - R^\hier_{T_N}\right)\ind{F} \suchthat E}}_{\mathrm{(I)}} + \underbrace{(1 - \zeta)\expt{R^\hier_{T_N}\ind{F} \suchthat E}}_{\mathrm{(II)}} + \left[\zeta + \prob{F\complement \suchthat E}\right]H.
    	\end{align*}
    	We bound $\mathrm{(I)}$ and $\mathrm{(II)}$ separately.

    	\textit{Bounding} $\mathrm{(I)}$.
    	Let $G$ be the event $E \intersect \set{z_N \in \cluster^\ast}$.
    	Then, if we define
    	\[
    	T_{\min} = \sum_{k = 0}^{N - 1}T_{\min}(z_k, u_k) \leq N(W + 1)
    	\]
    	as the minimum time needed to execute $\history_\meta$, we then have that
    	\begin{align*}
    	\expt{\left(R^\ast_{T^\ast} - R^\hier_{T_N}\right)\ind{F} \suchthat E} &\leq \expt{R^\ast_{T^\ast} - R^\hier_{T_N} \suchthat E} \\
    	&\leq \expt{V_{T^\ast}^\ast(z_N) - V^\hier_{T_N}(z_n) \suchthat G} \\
    	&\leq \expt{V^\ast_{T_{\min}}(z_N) - V_{T_N}^\hier(z_N) \suchthat G} \\
    	&\leq \int_{0}^{H}\prob{V_{T_{\min}}^\ast(z_N) - V_{T_N}^\hier(z_N) > \alpha \suchthat G} \ \diff{\alpha}.
    	\end{align*}
    	Note that the bound on $T_{\min}$ follows from \assumpref{assumption:skill-horizon}.
    	To convert the different in values into a difference of times, observe that if $T_N - T_{\min} \leq \alpha - \epsilon$, then
    	\begin{align*}
    	V_{T_{\min}}^\ast(z_N) - V_{T_N}^\hier(z_N) &= V_{T_{\min}}^\ast(z_N) - V_{T_N}^\ast(z_N) + V_{T_N}^\ast(z_N) - V_{T_N}^\hier(z_N) \\
    	&\leq (\bar{T} - T_{\min}) + \epsilon \\
    	&\leq \alpha.
    	\end{align*}
    	Therefore,
    	\begin{align*}
    	&\int_{0}^{H}\prob{V_{T_{\min}}^\ast(z_N) - V_{T_N}^\hier(z_N) > \alpha \suchthat G} \ \diff{\alpha} \\
    	&\quad \leq \int_{0}^{H}\prob{T_N - T_{\min} > \alpha - \epsilon \suchthat G} \ \diff{\alpha} \\
    	&\quad \leq \expt{\int_{0}^{H}\prob{T_N - T_{\min} > \alpha - \epsilon \suchthat \history_\meta} \ \diff{\alpha} \suchthat G} \\
    	&\quad \leq \expt{\int_{0}^{H}\prob{T_N - [1 + (1 + \gamma)W + \epsilon]\horizon_\effective > \alpha - \epsilon - \horizon_\effective(\gamma W + \epsilon) \suchthat \history_\meta} \ \diff{\alpha} \suchthat G} \\
    	&\quad \leq \epsilon + \horizon_\effective(\gamma W + \epsilon) + \expt{\int_{0}^{\infty}\prob{T_N - [1 + (1 + \gamma)W + \epsilon]\horizon_\effective > \alpha \suchthat \history_\meta} \ \diff{\alpha} \suchthat G} \\
    	&\quad \lesssim \epsilon + \horizon_\effective(\gamma W + \epsilon) + \beta[(1 + \gamma)W + \epsilon]\sqrt{\horizon_\effective},
    	\end{align*}
    	where the final inequality integrates the tail bound provided in \lemmaref{lemma:ref-policy-execution-time}.
    	Overall, we have that by rearranging,
    	\[
    	\mathrm{(I)} \lesssim (1 + \horizon_\effective + \beta\sqrt{\horizon_\effective})\epsilon + \left[\gamma\horizon_\effective + \beta(1 + \gamma)\sqrt{\horizon_\effective}\right]W.
    	\]

    	\textit{Bounding} $\mathrm{(II)}$. By the characterization of $\tilde\pi$ in \lemmaref{lemma:surrogate-policy-characterization},
    	\[
    	(1 - \zeta)\expt{R^\hier_{T_N}(\history_\meta(\pi^\ast))\ind{F} \suchthat E} \leq \expt{R^\hier_{T_N}(\history_\meta(\tilde{\pi}))\ind{F}} \leq V_0^{\tilde\pi}(s_0),
    	\]
    	where the final inequality uses the fact that $\bar{R}_{\bar{T}}(\history_\meta(\tilde\pi))$ is the return of $\tilde\pi$ in $\tilde{\MDP}_\meta$, given $F$.

    	\textit{Concluding.} Putting all of the previous bounds together, we find that
    	\begin{align*}
    	V_0^\ast(s_0) &\leq V_0^{\bar{\pi}}(s_0) + (1 + \horizon_\effective + \beta\sqrt{\horizon_\effective})\epsilon + \left[\gamma\horizon_\effective + \beta(1 + \gamma)\sqrt{\horizon_\effective}\right]W \\
    	&\qquad + \left[\zeta + \prob{F\complement \suchthat E}\right]H.
    	\end{align*}
    	By setting $\bar{H}$ to
    	\[
    	\bar{H} = \horizon_\effective[1 + (1 + \gamma)W + \epsilon] + \beta[(1 + \gamma)W + \epsilon]\sqrt{2\horizon_\effective\log\frac{1}{\zeta}} \ll H,
    	\]
    	sub-Gaussian tail bounds on $T_N$ implies that $\prob{F\complement \suchthat E} \leq \zeta$.
    	Finally, by \lemmaref{lemma:surrogate-achieves-same-value},
    	\[
    	V_0^{\tilde\MDP_\meta, \tilde\pi}(s_0) = V_0^{\MDP_\meta, \pi}(s_0) = V_0^{\MDP_\target, \pi},
    	\]
    	where the last equality follows by the construction of $\MDP_\meta$.
    	We thus obtain the desired suboptimality bound.
    \end{proof}

  \subsubsection{Regret Analysis}

    As earlier suggested, we now make use of $\tilde\pi$ as a comparator policy in order to prove a regret bound on a learner making use of the procedure outlined in \secref{sec:meta-test-learning-procedure}.

    \begin{theorem}
    	Assume that \Euler{} generates policies $\pi_1, \dots, \pi_N$ on $\MDP_\meta$, as constructed in \secref{sec:meta-test-learning-procedure}.
    	Then, we have the following regret bound:
    	\[
    	\sum_{k = 1}^{N}V^{\ast}_0(s_0) - V^{\pi_k}_0(s_0) \lesssim \sqrt{H^2\bar{H}LMN} + N\epsilon_{\mathrm{subopt}},
    	\]
    	where
    	\[
    	\epsilon_{\mathrm{subopt}} \defas (1 + \horizon_\effective + \beta\sqrt{\horizon_\effective})\epsilon + \left[\gamma\horizon_\effective + \beta(1 + \gamma)\sqrt{\horizon_\effective}\right]W + \zeta H.
    	\]
    \end{theorem}
    \begin{proof}
    	Throughout the proof, we consider applying \Euler{} to $\MDP_\meta$ where the rewards are scaled by $1/H$ to ensure that rewards are bounded in $[0, 1]$.
    	As a result, we can bound $\mathcal{G} \leq 1$ in the \Euler{} regret bound in \citet{zanette2019tighter}, since the sum of rewards in $\MDP_\meta$ is also the sum of rewards in $\MDP$, and scaling by $1/H$ gives the desired bound on $\mathcal{G}$.
    	Therefore,
    	\[
    	\sum_{k = 1}^{N}V_0^{\ast, \MDP_\meta}(s_0) - V_0^{\pi_k}(s_0) \lesssim H\sqrt{\frac{1}{\horizon_\effective}\bar{H}LM\horizon_\effective N} = \sqrt{H^2\bar{H}LMN}.
    	\]
    	Furthermore,
    	\[
    	V^\ast_0(s_0) - V_0^{\ast, \MDP_\meta}(s_0) \leq V^\ast_0(s_0) - V^{\pi}_0(s_0) + V_0^{\pi, \MDP_\meta}(s_0) - V_0^{\ast, \MDP_\meta}(s_0) \leq \epsilon_{\mathrm{subopt}}.
    	\]
    	We thus obtain the desired result.
    \end{proof}

\subsection{An Exponential Regret Separation for a Hierarchy-Oblivious Learner}

  In this section, we provide proof of the exponential regret separation between a hierarchical learner and a learner oblivious to the hierarchy.
  The overall idea behind our proof is the reduction of solving the family of minimax instances described in \citet{domingues2021episodic} to a particular family of task distributions.

  \subsubsection{The Hard Task Distribution Family}

    In this section, we describe the family of task distributions that forces any meta-training-oblivious learner to incur exponential regret.
    For any string $s$, we write $\card{s}$ for its length.

    We now define the family of binary tree room MDPs $\mathbb{M}_{W}$ of depth $W$.
    We index a member of this family by a tuple $(\ell^\ast, a^\ast, e^\ast)$, where $\ell^\ast$ is a binary string of length $W - 1$, and $a^\ast, e^\ast \in \set{0, 1}$.
    The MDP $\MDP_{(\ell^\ast, a^\ast, e^\ast)} = (\statespace, \actionspace, \transition_{(\ell^\ast, a^\ast, e^\ast)}, r, H)$ corresponding to this tuple is constructed as follows:

    \paragraph{State Space $\statespace$.}
    We create a root state $s_{\troot}$, $2^{W} - 1$ states indexed by binary strings of length at most $W - 1$ collected into a set $T = \set{s_0, s_1, s_{00}, s_{01}, \dots}$, a gate state $s_{\gate}$, and terminal states $\terminate_\trap, \terminate_S, \terminate_F$.

    \paragraph{Action Space $\actionspace$.}
    The set of available actions at every state is the set $\set{0, 1}$.

    \paragraph{Transition Dynamics $\transition_{(\ell^\ast, a^\ast, e^\ast)}$.}
    We define the dynamics as follows:
    \[
    \transition_{(\ell^\ast, a^\ast, e^\ast)}(\cdot \conditionedon s, a) =
    \begin{cases}
    \delta(s_a) & s = s_\troot \\
    \delta(s_{ta}) & s = s_t \in T, \card{t} < W - 1 \\
    b\delta(s_\gate) + (1 - b)\delta(\terminate_\trap) & s = s_t \in T, \card{t} = W - 1, \\
    {} & s \neq s_{\ell^\ast}, b \sim \bernoulli{1/2} \\
    b\delta(s_\gate) + (1 - b)\delta(\terminate_\trap) & s = s_{\ell^\ast}, b \sim \bernoulli{1/2 + \epsilon\ind{a = a^\ast}} \\
    \delta(\terminate_S) & s = s_\gate, a = e^\ast \\
    \delta(\terminate_F) & s = s_\gate, a \neq e^\ast.
    \end{cases}
    \]

    \paragraph{Reward Function $r$.} The reward function is $r(s, a) = \ind{s = \terminate_S} + \ind{s = s_\gate, a = a^\ast}$.

    Having described all the components of every member of $\mathbb{M}_W$, all that remains is to construct the family of task distributions $\mathbb{T}_W$.
    Each member of this family will be indexed by $(\ell^\ast, a^\ast)$, where $\ell^\ast$ and $a^\ast$ are as described above.
    Then, the task distribution $\mathcal{T}_{(\ell^\ast, a^\ast)} \in \mathbb{T}_W$ chooses uniformly within the set $\set{\MDP_{(\ell^\ast, a^\ast, 0)}, \MDP_{(\ell^\ast, a^\ast, 1)}}$.
    Note that this implicitly defines the latent hierarchy so that the clusters are $\set{s_\troot, s_\gate, \terminate_\trap} \union T$, $\set{\terminate_S}$, and $\set{\terminate_F}$.
    Furthermore, the set of exits for the first cluster is $\set{(s_\gate, 0), (s_\gate, 1)}$.

  \subsubsection{A Family of Hard Instances}

    \begin{algorithm}[b!]
    	\caption{The reduction of learning $\mathbb{N}_W$ to learning $\mathbb{M}_W$ in \secref{sec:hardness-result-formal-proof}.}
    	\label{alg:hard-case-reduction}
    	\begin{algorithmic}[1]
    		\Procedure{$\mathcal{P}_{\mathcal{A}}$}{$\MDP \in \mathbb{N}_W$}
    		\State Initialize $\history_0 = \emptyset$
    		\ForAll{$n \in [N]$}
    		\State Obtain $\pi_n = \mathcal{A}(\history_0, \dots, \history_{n - 1})$.
    		\State Play $\pi_n$ in $\MDP$, get history $\mathcal{G}_n = ((s_0, a_0, r_0, s_1), \dots, (s_{H - 1}, a_{H - 1}, r_{H - 1}, s_H))$.
    		\If{$s_{W + 1} = s_\gate$}
    		\State $s_{W + 1}' \gets s_{W + 1}$
    		\ForAll{$h = W + 1, \dots, H - 1$}
    		\If{$h = W$}
    		\State $s_{h + 1}' \gets \terminate_S$ if $a_h = 1$ else $\terminate_F$.
    		\State $r_h' \gets \ind{a_h = 1}$
    		\Else
    		\State $s_{h + 1}' \gets s_h'$, $r_h' \gets r_h$
    		\EndIf
    		\State Replace $(s_h, a_h, r_h, s_{h + 1})$ with $(s_h', a_h, r_h', s_{h + 1}')$ in $\mathcal{G}_n$
    		\EndFor
    		\EndIf
    		\State $\mathcal{H}_n \gets \mathcal{G}_n$.
    		\EndFor
    		\EndProcedure
    	\end{algorithmic}
    \end{algorithm}

    In this section, we describe the family of hard instances which we reduce to solving the task distribution above.
    Intuitively, if an algorithm incurs low regret throughout $\mathbb{M}_W$, then it must be able to quickly find a policy to reliable reach the gate state $s_\gate$ for any MDP in the family.

    \paragraph{Constructing the hard instances.} Accordingly, we define a new MDP family $\mathbb{N}_W$, which now is only indexed by $(\ell^\ast, a^\ast)$, and is constructed similarly as any member of $\mathbb{M}_W$, but ignoring states outside $\set{s_\troot, s_\gate, \terminate_\trap} \union T$.
    Additionally, we redefine the reward function $r$ for any member to be $r(s, a) \defas \ind{s = s_\gate}$.
    We note that this is exactly the set of hard tasks used to prove a minimax regret bound in \citet{domingues2021episodic}.

    \paragraph{The lower bound.}
    We state the lower bound result from \citet{domingues2021episodic}, in a slightly more restricted form for ease of proof and presentation.
    In particular, we consider the following more restricted definition of an algorithm:

    \begin{definition}
    	Let $\history_n$ be the trajectory data generated by playing a policy $\pi_n$ in an MDP $\MDP$.
    	That is, $\history_n = ((s_0, a_0, r_0, s_1), (s_1, a_1, r_1, s_2), \dots, (s_{H - 1}, a_{H - 1}, r_{H - 1}, s_H))$,
    	where $s_0$ and $a_0$ are fixed, $r_h = r(s_h, a_h)$, and $s_{h + 1} \sim \transition_\MDP(\cdot \conditionedon s_h, a_h)$.
    	Additionally, we set $\history_0 = \emptyset$.

    	Then, a \term{valid algorithm} $\mathcal{A}$ for our purposes is one which, for the $n^{\text{th}}$ episode, outputs a deterministic, non-stationary policy $\pi$ that is solely a function of the current state and action and $\Union_{i = 1}^{n - 1}\history_i$.
    	That is, $\mathcal{A}$ does not output policies that adapt to the current running episode.
    \end{definition}

    We again emphasize that this restriction is not necessary but that many algorithms nevertheless satisfy this condition (including \UCBVI{} and \Euler{}).
    We then have the following hardness result:

    \begin{theorem}[\citet{domingues2021episodic}, Theorem 9, restated]
    	\label{thm:binary-tree-minimax}
    	Assume that $W \geq 2$ and $H \geq 3W$.
    	Then, for every algorithm $\mathcal{A}$, there exists an MDP $\MDP \in \mathbb{N}_W$ such that
    	\[
    	\expt[\MDP, \mathcal{A}]{\sum_{n = 1}^{N}V^\ast_0(s_\troot) - V^{\pi_n}_0(s_\troot)} \gtrsim 2^{W/2}\sqrt{H^2N}.
    	\]
    \end{theorem}

  \subsubsection{Proving the Hardness Result}
    \label{sec:hardness-result-formal-proof}

    We now use the hardness result in the previous section to demonstrate that no algorithm can incur sub-exponential regret in $W$ on all tasks in $\mathbb{M}_W$.
    We do so by proving that an algorithm solving all tasks in $\mathbb{M}_W$ can be used to construct an algorithm for solving all tasks in $\mathbb{N}_W$.

    Formally, let $\mathcal{A}$ be any algorithm for learning any MDP in $\mathbb{M}_W$.
    We then construct an algorithm $\mathcal{P}_{\mathcal{A}}$ for learning any MDP in $\mathbb{N}_W$ as in \algoref{alg:hard-case-reduction}.

    Given this reduction, we aim to prove the following result:

    \begin{proposition}
    	\label{prop:regret-reduction}
    	For any $\MDP_{(\ell^\ast, a^\ast)} \in \mathcal{N}_W$, we have that
    	\[
    	\expt[\MDP_{(\ell^\ast, a^\ast)}, \mathcal{P}_{\mathcal{A}}]{\sum_{n = 1}^{N}V^\ast_0(s_\troot) - V^{\pi_n}_0(s_\troot)} \leq \expt[\MDP_{(\ell^\ast, a^\ast, 1)}, \mathcal{A}]{\sum_{n = 1}^{N}V^\ast_0(s_\troot) - V^{\pi_n}_0(s_\troot)}.
    	\]
    \end{proposition}

    To prove this result, we first prove that $\mathcal{P}_{\mathcal{A}}$ can simulate $\MDP_{(\ell^\ast, a^\ast, 1)}$:

    \begin{lemma}
    	\label{lemma:reduction-can-simulate}
    	For any $n$, the distribution over $(\mathcal{H}_0, \dots, \mathcal{H}_n)$ induced by running \algoref{alg:hard-case-reduction} over $\MDP_{(\ell^\ast, a^\ast)} \in \mathbb{N}_{W}$ is equal to that induced by running $\mathcal{A}$ over $\MDP_{(\ell^\ast, a^\ast, 1)} \in \mathbb{M}_W$.
    \end{lemma}
    \begin{proof}
    	We proceed by induction.
    	The result holds trivially for $n = 0$.

    	Now, assume that the result holds for some $n$.
    	We condition on the histories $(\history_0, \dots, \history_n)$
    	Then, note that both algorithms play the same policy $\pi_{n + 1}$, since $\mathcal{P}_{\mathcal{A}}$ uses $\mathcal{A}$ to obtain the next policy.
    	As a result, by the construction of $\MDP_{(\ell^\ast, a^\ast)}$ and $\MDP_{(\ell^\ast, a^\ast, 1)}$, the distribution over $(s_h, a_h, r_h, s_{h + 1})$ are equal for $h \leq W$.
    	Furthermore, Lines $6-14$ simulates the dynamics of $\MDP_{(\ell^\ast, a^\ast, 1)}$ conditioned on $s_{W + 1} = s_\gate$, while conditioned on $s_{W + 1} = \terminate_\trap$, the dynamics of the two MDPs are the same.
    	Therefore, conditioned on any $(\history_0, \dots, \history_n)$, the distribution over $\history_{n + 1}$ induced by the two algorithms are also the same.
    	Thus, the claim holds by induction.
    \end{proof}

    Finally, we can prove \propref{prop:regret-reduction}.

    \begin{proof}[Proof of \propref{prop:regret-reduction}]
    	Throughout this proof, we omit the starting state $s_\troot$ and the timestep $0$ in the value.
    	We prove the result by induction.
    	Clearly, the result holds for $N = 0$.

    	Assume that the bound holds for some $N$.
    	Then, we have that
    	\begin{align*}
    	&\expt[\MDP_{(\ell^\ast, a^\ast)}, \mathcal{P}_{\mathcal{A}}]{\sum_{n = 1}^{N + 1}V^\ast - V^{\pi_n}} \\
    	&\qquad = \expt[\MDP_{(\ell^\ast, a^\ast)}, \mathcal{P}_{\mathcal{A}}]{\sum_{n = 1}^{N}V^\ast - V^{\pi_n}} + \expt[\MDP_{(\ell^\ast, a^\ast)}, \mathcal{P}_{\mathcal{A}}]{V^\ast - V^{\pi_{N + 1}}} \\
    	&\qquad \leq \expt[\MDP_{(\ell^\ast, a^\ast, 1)}, \mathcal{A}]{\sum_{n = 1}^{N}V^\ast - V^{\pi_n}} + \expt[\MDP_{(\ell^\ast, a^\ast)}, \mathcal{P}_{\mathcal{A}}]{\expt{V^\ast - V^{\pi_{N + 1}} \suchthat (\history_0, \dots, \history_N)}},
    	\end{align*}
    	where the final inequality uses the inductive hypothesis and the tower property of expectation.
    	Now, recall from \lemmaref{lemma:reduction-can-simulate} that
    	\begin{align*}
    	&\expt[\MDP_{(\ell^\ast, a^\ast)}, \mathcal{P}_{\mathcal{A}}]{\expt{V^\ast - V^{\pi_{N + 1}} \suchthat (\history_0, \dots, \history_N)}} \\
    	&\qquad = \expt[\MDP_{(\ell^\ast, a^\ast, 1)}, \mathcal{A}]{\expt{V^\ast - V^{\pi_{N + 1}} \suchthat (\history_0, \dots, \history_N)}}.
    	\end{align*}
    	We emphasize that the value functions are still with respect to $\MDP_{(\ell^\ast, a^\ast)}$.
    	However, for any policy $\pi$ output by $\mathcal{A}$,
    	\begin{align*}
    	V^\ast - V^{\pi} &= \expt[\MDP(\ell^\ast, a^\ast), \pi]{(H - W - 1)\ind{s_{W + 1} \neq s_\gate}} \\
    	&\leq \expt[\MDP(\ell^\ast, a^\ast), \pi]{(H - W - 1)\ind{\text{$s_{W + 1} \neq s_\gate$ or $a_{W + 1} \neq 1$}}} \\
    	&\leq \expt[\MDP(\ell^\ast, a^\ast, 1), \pi]{(H - W - 1)\ind{\text{$s_{W + 1} \neq s_\gate$ or $a_{W + 1} \neq 1$}}}.
    	\end{align*}
    	Note that the right-hand side is the regret in $\MDP(\ell^\ast, a^\ast, 1)$ for playing $\pi$.
    	Therefore, since both algorithms play the same policy $\pi_{N + 1}$, we thus obtain the desired result by induction.
    \end{proof}

    With \propref{prop:regret-reduction}, we can now formally state and prove the separation result:

    \begin{theorem}
    	There exists a task distribution $\mathcal{T}_{(\ell^\ast, a^\ast)} \in \mathcal{T}_W$ such that an algorithm $\mathcal{A}$, without access to the meta-training tasks (and thus without access to the hierarchy), incurs expected regret lower bounded as
    	\[
    	\expt[\MDP \sim \mathcal{T}_{(\ell^\ast, a^\ast)}]{\mathrm{Regret}_N(\MDP, \mathcal{A})} \gtrsim 2^{W/2}\sqrt{H^2N}.
    	\]
    	On the other hand, for any task distribution in the family, the hierarchy-based learner $\mathcal{P}$ in \secref{sec:meta-test-learning-procedure}, with access to a $0$-suboptimal hierarchy oracle, achieves regret bounded by $\sqrt{H^2N}$ with high probability on any sampled task.
    \end{theorem}
    \begin{proof}
    	Fix any algorithm $\mathcal{A}$.
    	Using \thmref{thm:binary-tree-minimax}, there exists $\MDP_{(\ell^\ast, a^\ast)}$ such that
    	\[
    	\expt[\MDP_{(\ell^\ast, a^\ast)}, \mathcal{P}_{\mathcal{A}}]{\sum_{n = 1}^{N}V^\ast_0(s_\troot) - V^{\pi_n}_0(s_\troot)} \gtrsim 2^{W/2}\sqrt{H^2N}
    	\]
    	Thus, by \propref{prop:regret-reduction},
    	\[
    	\expt[\MDP_{(\ell^\ast, a^\ast, 1)}, \mathcal{A}]{\sum_{n = 1}^{N}V^\ast_0(s_\troot) - V^{\pi_n}_0(s_\troot)} \gtrsim 2^{W/2}\sqrt{H^2N}.
    	\]
    	Note that the proof in \propref{prop:regret-reduction} can be extended for $\MDP_{(\ell^\ast, a^\ast, 0)}$ with appropriate modifications to $\mathcal{P}_{\mathcal{A}}$, and thus the same inequality holds.
    	Consequently,
    	\[
    	\expt[\MDP \sim \mathcal{T}_{(\ell^\ast, a^\ast)}]{\mathrm{Regret}_N(\mathcal{A})} \gtrsim 2^{W/2}\sqrt{H^2N}.
    	\]

    	On the other hand, with access to the $0$-suboptimal hierarchy oracle, observe that the learner only has to plan at timesteps $0$ and $W + 1$, allowing us to obtain tighter bounds (as $\statespace_\meta$ is smaller than the construction in \secref{sec:meta-test-learning-procedure}).
    	Furthermore, the suboptimality of planning with the hierarchy oracle is $0$ for any task distribution in the family.
    	We thus obtain the desired bound.
    \end{proof}

\subsection{A Discussion of \assumpref{assumption:env-low-var-and-regular}}
  \label{sec:low-var-discussion}

  In this section, we discuss why the conditions in \assumpref{assumption:env-low-var-and-regular} are sufficient for ensuring low hierarchical suboptimality.
  In particular, we provide examples of MDPs that satisfy \assumpref{assumption:meta-test-compatibility}, and are thus in a sense tasks that are ``compatible with the hierarchy'', but nevertheless force a hierarchy-based learner to incur $O(H)$ suboptimality.

  \subsubsection{$(\alpha, \beta)$-unreliability}

    Consider the MDP in \figref{fig:bad-case-high-variance} with horizon $H + 2$ and two actions $a^\ast$ and $a_1$.
    The optimal policy chooses $a^\ast$ at every step, achieving a value of $H - O(1)$, since
    \begin{align*}
      V^\ast_{H + 1}(s_0) &= \frac{1}{2}H + \frac{1}{2}V^\ast_H(s_1) = \frac{1}{2}H + \frac{1}{4}(H - 1) + \frac{1}{4}V^\ast_{H - 1}(s_2) \\
        &= H\sum_{h=1}^{H}\frac{1}{2^h} - \frac{1}{2}\sum_{h=1}^{H}\frac{h}{2^h} = H - O(1).
    \end{align*}

    Now, assume that the MDP has a latent hierarchy so that the set of exits are given by $(t_i, a)$ for any $i \in [H]$ and $a \in \actionspace$.
    Clearly, the optimal hierarchy-based learner would always choose $(t_0, a^\ast)$ or $(t_0, a_1)$ as its high-level action.
    However, if the agent fails to transition to $t_0$ at the first timestep due to stochasticity, it will go to the end of the chain, back to $s_0$ and try $a^\ast$ once more.
    This is because it already has set a meta-action, and \emph{does not replan until an exit is performed}.
    Thus, the optimal agent on the meta-MDP achieves a value of $H/2$, and is therefore $O(H)$-suboptimal, even with a $0$-suboptimal hierarchy oracle.

    Intuitively, hierarchy-based learners as formulated in \secref{sec:meta-test-learning-procedure} fail on the MDP in \figref{fig:bad-case-high-variance} because such learners commit to a skill until completion.
    Thus, when such skills exhibit high variance in completion times, hierarchy-based learners fare worse than other learners which are able to replan based on the current state (e.g., in this case, choose another exit if $a^\ast$ fails to take the agent to the current subgoal).
    Thus, $(\alpha, \beta)$-reliability serves to eliminate such MDPs, ensuring that the skills corresponding to reaching exits are reliable.

    \begin{figure}[h]
      \centering
      \captionsetup{justification=centering, width=0.8\linewidth}
      \begin{tikzpicture}[scale=0.95, transform shape]
        \node[state] (s0) {$s_0$};
        \node[state,right=of s0] (s1) {$s_1$};
        \node[draw=none,right=of s1] (s1-sh-1) {$\cdots$};
        \node[state,right=of s1-sh-1] (sh-1) {$s_{\horizon - 1}$};
        \node[state,right=of sh-1] (sh) {$s_{\horizon}$};

        \node[state, above=of s0] (t0) {$t_0$};
        \node[state,right=of t0] (t1) {$t_1$};
        \node[draw=none,right=of t1] (t1-th-1) {$\cdots$};
        \node[state,right=of t1-th-1] (th-1) {$t_{\horizon - 1}$};
        \node[state,right=of th-1] (th) {$t_{\horizon}$};

        \node[state, above=of t1-th-1,fill=gray!20!white] (sstar) {$s^\ast$};

        \path (sstar) edge[loop above, color=purple] node {$1, r = 1$} (sstar);
        \draw (sh.south) edge[-latex, bend left, color=red, auto=left] node {$1$} (s0.south);

        \draw[every loop, -latex, bend left, below, color=blue]
          (s0) edge node {$0.5$} (s1)
          (s1) edge node {$0.5$} (s1-sh-1)
          (s1-sh-1) edge node {$0.5$} (sh-1)
          (sh-1) edge node {$0.5$} (sh);

        \draw[every loop, -latex, bend right, below, color=red]
          (s0) edge node {$1$} (s1)
          (s1) edge node {$1$} (s1-sh-1)
          (s1-sh-1) edge node {$1$} (sh-1)
          (sh-1) edge node {$1$} (sh);

        \draw[every loop, -latex, color=blue, auto=left]
          (s0) edge node {$0.5$} (t0)
          (s1) edge node {$0.5$} (t1)
          (sh-1) edge node {$0.5$} (th-1)
          (sh) edge node {$1$} (th);

        \draw[every loop, -latex, color=purple]
          (t0) edge[auto=left] node {$1$} (sstar)
          (t1) edge[auto=right] node {$1$} (sstar)
          (th-1) edge[auto=left] node {$1$} (sstar)
          (th) edge[auto=right] node {$1$} (sstar);
      \end{tikzpicture}
      \caption{An MDP that does not satisfy low $(\alpha, \beta)$-unreliability, where $a^\ast$ is in blue, and $a_1$ is in red (and purple for both actions). State shading represents state clusters, and rewards are $0$ unless indicated otherwise.}
      \label{fig:bad-case-high-variance}
    \end{figure}
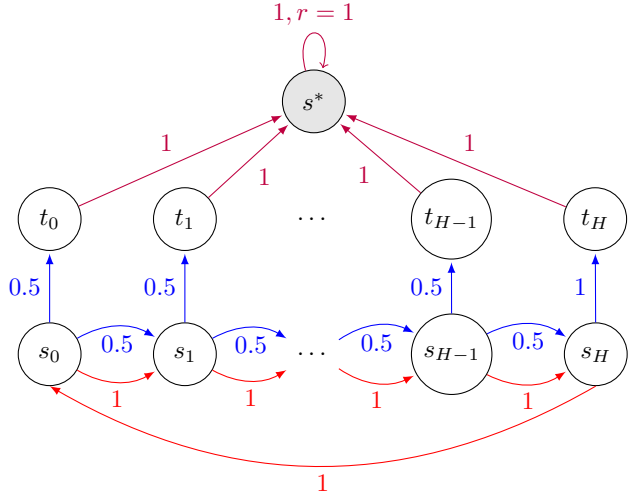

  \subsubsection{$\gamma$-goal-reaching suboptimality}

    In this section, we show that even when a hierarchy-based learner has access to highly reliable skills as in the previous section, the learner may still incur high hierarchical sub-optimality.
    Consider the MDP in \figref{fig:bad-case-high-reliability}, where we focus on a single room for simplicity.
    Furthermore, assume that there are two exits, one from $l_{H/2}$ and one from $r_{H/2}$.
    Note that a $0$-suboptimal hierarchy oracle has highly reliable goal-reaching policies for reaching both of these exit states, requiring exactly $H/2$ timesteps with no stochasticity.

    However, given the values assigned to $l_{H/2}$ and $r_{H/2}$, the optimal policy would opt to take the state $t$, which transitions to either state with probability at least $1/2$ in only two environment steps.
    Therefore, the optimal policy achieves an optimal value of $H - O(1)$.
    However, the optimal policy, in having to commit to exactly one of the exits, will achieve a value of $H/2$, and thus be $O(H)$-suboptimal despite having a perfect hierarchy oracle.

    Hierarchy-based learners fail on the MDP in \figref{fig:bad-case-high-reliability} because an optimal policy for goal-reaching does not necessarily reach a goal as quickly as possible.
    Thus, $\gamma$-goal-reaching suboptimality is a regularity condition that ensures that this is indeed the case.

    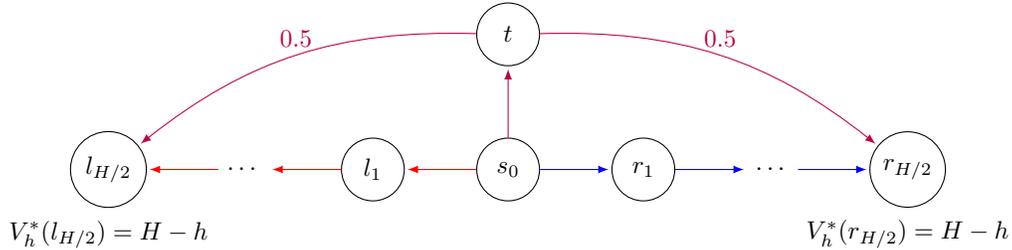
\begin{figure}[h]
      \centering
      \captionsetup{justification=centering, width=0.8\linewidth}
      \begin{tikzpicture}[scale=0.95, transform shape]
        \node[state] (s0) {$s_0$};
        \node[state,right=of s0] (r1) {$r_1$};
        \node[draw=none,right=of r1] (r1-rh) {$\cdots$};
        \node[state,right=of r1-rh] (rh) {$r_{H/2}$};
        \node[state,left=of s0] (l1) {$l_1$};
        \node[draw=none,left=of l1] (l1-lh) {$\cdots$};
        \node[state,left=of l1-lh] (lh) {$l_{H/2}$};
        \node[state,above=of s0] (t) {$t$};
        \node[below=0.5mm of rh] {$V^\ast_h(r_{H/2}) = H - h$};
        \node[below=0.5mm of lh] {$V^\ast_h(l_{H/2}) = H - h$};

        \draw[every loop, -latex, color=blue]
          (s0) edge (r1)
          (r1) edge (r1-rh)
          (r1-rh) edge (rh);

        \draw[every loop, -latex, below, color=red]
          (s0) edge (l1)
          (l1) edge (l1-lh)
          (l1-lh) edge (lh);

        \draw[every loop, -latex, above, color=purple]
          (s0) edge (t)
          (t) edge[bend right=20] node {$0.5$} (lh)
          (t) edge[bend left=20] node {$0.5$} (rh);
      \end{tikzpicture}
      \caption{An MDP that does not satisfy low $\gamma$-goal-reaching suboptimality, with three actions indicated by {\color{red} red}, {\color{blue} blue}, and {\color{purple} purple}, and exits $l_h$ and $r_h$. The MDP satisfies $(\infty, 0)$-unreliability, yet nevertheless exhibits high hierarchical suboptimality.}
      \label{fig:bad-case-high-reliability}
    \end{figure}

\end{document}